\documentclass[11pt,twoside]{article}

\newif\ifhidecomments
\hidecommentsfalse

\usepackage[margin=1in]{geometry}
\usepackage[utf8]{inputenc}
\usepackage[T1]{fontenc}
\usepackage{twemojis}

\usepackage{amsmath,amsthm,amssymb,amsfonts}
\usepackage{mathtools}
\usepackage{mathrsfs}
\usepackage{bm}
\usepackage{bbm}
\usepackage{nicefrac}
\usepackage{nccmath} 
\usepackage{scalerel}
\allowdisplaybreaks

\makeatletter
\def\thm@space@setup{%
  \thm@preskip=0.3\baselineskip
  \thm@postskip=0.3\baselineskip
}
\makeatother

\usepackage[dvipsnames,svgnames,table]{xcolor}
\definecolor{darkblue}{rgb}{0,0,1}

\usepackage{epsf}          
\usepackage{graphicx}
\usepackage{tikz}
\usetikzlibrary{calc,arrows.meta,decorations.pathreplacing}

\usepackage{caption}
\usepackage{subcaption}    
\usepackage[section]{placeins}

\usepackage[skip=1.5mm,indent=0.5cm]{parskip}

\AtBeginEnvironment{align}{\addtolength{\jot}{1mm}}
\AtBeginEnvironment{align*}{\addtolength{\jot}{1mm}}
\makeatletter
\AtBeginEnvironment{cases}{%
  \let\orig@arraycr\@arraycr
  \def\@arraycr{\orig@arraycr[1.5mm]}%
}
\makeatother

\makeatletter
\renewcommand\section{\@startsection{section}{1}{\z@}%
  {-3.5ex plus -1ex minus -.2ex}
  {2ex plus .3ex}
  {\normalfont\Large\bfseries}}

\renewcommand\subsection{\@startsection{subsection}{2}{\z@}%
  {-3.0ex plus -1ex minus -.2ex}
  {1.5ex plus .2ex}
  {\normalfont\large\bfseries}}

\renewcommand\subsubsection{\@startsection{subsubsection}{3}{\z@}%
  {-2.0ex plus -1ex minus -.2ex}
  {1.8ex plus .2ex}
  {\normalfont\normalsize\bfseries}}
\makeatother

\usepackage{comment}
\usepackage{textcomp}
\usepackage{float}
\usepackage{array}
\usepackage{enumitem}
\usepackage{needspace}
\usepackage{etoolbox}
\AtBeginEnvironment{table}{\captionsetup{justification=centering,singlelinecheck=false}}
\usepackage{fancybox}      
\usepackage{tcolorbox}
\usepackage{lipsum}

\setlist[itemize]{%
  itemsep=0.25\baselineskip,
  topsep=0.25\baselineskip,
  parsep=1.5mm,
  partopsep=0pt
}
\setlist[enumerate]{%
  itemsep=0.25\baselineskip,
  topsep=0.25\baselineskip,
  parsep=1.5mm,
  partopsep=0pt
}

\usepackage[max2,noblocks]{authblk}
\usepackage{footmisc}

\usepackage{booktabs}
\usepackage{longtable}
\usepackage{makecell}
\usepackage{multirow}
\usepackage{tabu}
\usepackage{threeparttable}
\usepackage{threeparttablex}
\usepackage{xltabular}
\usepackage{adjustbox}

\newcolumntype{L}[1]{>{\raggedright\arraybackslash}p{#1}}
\newcolumntype{C}[1]{>{\centering\arraybackslash}p{#1}}

\usepackage{algorithm}
\usepackage{algpseudocode}

\makeatletter
\newenvironment{breakablealgorithm}
  {%
   \par\noindent
   \refstepcounter{algorithm}%
   \hrule height.6pt\relax
   \renewcommand{\caption}[2][\relax]{%
     \noindent\textbf{\fname@algorithm~\thealgorithm.} ##2\par
     \ifx\relax##1\relax
       \addcontentsline{loa}{algorithm}{\protect\numberline{\thealgorithm}##2}%
     \else
       \addcontentsline{loa}{algorithm}{\protect\numberline{\thealgorithm}##1}%
     \fi
     \vspace{2pt}\hrule\vspace{2pt}%
   }%
  }{%
   \vspace{3pt}\hrule\par
  }
\makeatother

\makeatletter
\renewcommand{\fnum@algorithm}{\fname@algorithm}
\makeatother

\usepackage[nottoc]{tocbibind}
\usepackage[toc,page,header]{appendix}
\usepackage[round]{natbib}
\usepackage{aliascnt}

\usepackage[
  colorlinks=true,
  citecolor={blue!60!black},
  urlcolor={blue!70!black},
  linkcolor={red!60!black},
  breaklinks=true,
  hypertexnames=false
]{hyperref}
\usepackage{orcidlink}
\usepackage{cleveref}
\usepackage{crossreftools}

\newlength\bshft
\def\fakebold#1{\ThisStyle{\ooalign{$\SavedStyle#1$\cr%
  \kern-\bshft$\SavedStyle#1$\cr%
  \kern\bshft$\SavedStyle#1$}}}

\newenvironment{fminipage}%
  {\begin{Sbox}\begin{minipage}}%
  {\end{minipage}\end{Sbox}\fbox{\TheSbox}}

\newcommand{\mcD}{\mathcal{D}}

\newcommand{\mcO}{\mathcal{O}}

\newcommand{\E}{\mathbb{E}}

\newcommand{\bbP}{\mathbb{P}}
\renewcommand{\P}{\mathbb{P}}

\newcommand{\R}{\mathbb{R}}

\newcommand{\1}{\ensuremath{{\sf (i)}}}
\newcommand{\2}{\ensuremath{{\sf (ii)}}}
\newcommand{\3}{\ensuremath{{\sf (iii)}}}

\DeclareMathOperator{\var}{var}

\DeclareMathOperator*{\argmin}{argmin}

\newcommand{\TV}{d_{\text{TV}}}

\NewDocumentCommand{\norm}{o m}{%
  \ensuremath{%
    \IfNoValueTF{#1}
      {\left\lVert #2 \right\rVert}
      {#1\lVert #2 #1\rVert}%
  }%
}
\newcommand{\ip}[2]{\left\langle #1,#2 \right\rangle}

\newcommand{\indep}{\perp\hspace{-2.3mm}\perp}

\newtheorem{theorem}{Theorem}

\newaliascnt{corollary}{theorem}
\newtheorem{corollary}[corollary]{Corollary}
\aliascntresetthe{corollary}

\newaliascnt{lemma}{theorem}
\newtheorem{lemma}[lemma]{Lemma}
\aliascntresetthe{lemma}

\newaliascnt{proposition}{theorem}
\newtheorem{proposition}[proposition]{Proposition}
\aliascntresetthe{proposition}

\newtheorem{claim}{Claim}

\newtheorem{remark}{Remark}

\crefname{theorem}{theorem}{theorems}
\Crefname{theorem}{Theorem}{Theorems}

\crefname{lemma}{lemma}{lemmas}
\Crefname{lemma}{Lemma}{Lemmas}

\crefname{proposition}{proposition}{propositions}
\Crefname{proposition}{Proposition}{Propositions}

\crefname{corollary}{corollary}{corollaries}
\Crefname{corollary}{Corollary}{Corollaries}

\crefname{claim}{claim}{claims}
\Crefname{claim}{Claim}{Claims}

\crefname{assumption}{assumption}{assumptions}
\Crefname{assumption}{Assumption}{Assumptions}

\Crefname{equation}{Eq.}{Eqs.}

\makeatletter
\newcommand{\ostar}{\mathbin{\mathpalette\make@circled\star}}
\newcommand{\make@circled}[2]{%
  \ooalign{$\m@th#1\smallbigcirc{#1}$\cr\hidewidth$\m@th#1#2$\hidewidth\cr}%
}
\newcommand{\smallbigcirc}[1]{%
  \vcenter{\hbox{\scalebox{0.77778}{$\m@th#1\bigcirc$}}}%
}
\makeatother

\NewDocumentCommand{\tens}{e{_^}}{%
  \mathbin{\mathop{\otimes}\displaylimits
    \IfValueT{#1}{_{#1}}
    \IfValueT{#2}{^{#2}}%
  }%
}

\newlength{\widebarargwidth}
\newlength{\widebarargheight}
\newlength{\widebarargdepth}

\newcommand{\BER}{\ensuremath{\mathrm{Ber}}}

\newcommand{\thetastar}{\ensuremath{\theta^{\star}}}

\long\def\ignore#1{}

\definecolor{sapphire}{rgb}{0.03, 0.15, 0.4}            
\definecolor{skobeloff2}{rgb}{0.0, 0.45, 0.55}          
\definecolor{mediumorchid}{rgb}{0.73, 0.33, 0.83}       
\definecolor{candyapplered}{rgb}{1.0, 0.03, 0.0}         
\definecolor{cadmiumred}{rgb}{0.89, 0.0, 0.13}           
\definecolor{darkpastelred}{rgb}{0.76, 0.23, 0.13}       
\definecolor{darkcoral}{rgb}{0.8, 0.26, 0.17}            
\definecolor{limegreen}{rgb}{0.2, 0.64, 0.1}           
\definecolor{darkspringgreen}{rgb}{0.09, 0.45, 0.27}   
\definecolor{darkbrown}{rgb}{0.4, 0.26, 0.13}          
\definecolor{mahogany}{rgb}{0.75, 0.25, 0.0}           
\definecolor{pumpkin}{rgb}{1.0, 0.46, 0.09}            
\definecolor{tangerine}{rgb}{0.95, 0.52, 0.0}        
\definecolor{ochreyellow}{rgb}{0.8, 0.34, 0.0}     

\newcommand{\blue}[1]{\textcolor{blue}{#1}}

\newcommand{\red}[1]{\textcolor{red}{#1}}
\newcommand{\dsgreen}[1]{\textcolor{darkspringgreen}{#1}}
\newcommand{\brown}[1]{\textcolor{brown}{#1}}
\newcommand{\magenta}[1]{\textcolor{magenta}{#1}}

\ifhidecomments
    \newcommand{\apcomment}[1]{}
    \newcommand{\kav}[1]{}
    \newcommand{\jrc}[1]{}
    \newcommand{\jrcupdate}[1]{}
    \newcommand{\rs}[1]{}
\else
    \newcommand{\apcomment}[1]{{\noindent\bfseries\blue{AP --- #1}}}
    \newcommand{\kav}[1]{{\noindent\bfseries\brown{KAV --- #1}}}
    \newcommand{\jrc}[1]{\noindent{\bfseries\red{JRC --- #1}}}
    \newcommand{\jrcupdate}[1]{\noindent{\bfseries\dsgreen{JRC --- #1}}}
    \newcommand{\rs}[1]{\noindent{\bfseries\magenta{RS --- #1}}}
\fi

\usepackage{xcolor}
\usepackage{graphicx}
\usepackage{pgfplots}
\usepackage{pgfplotstable}
\usepgfplotslibrary{fillbetween}
\usetikzlibrary{plotmarks}

\pgfplotsset{compat=1.18}

\providecommand{\norm}[1]{\lVert#1\rVert_2}
\providecommand{\thetastar}{\theta^\star}

\definecolor{CompleteDataColour}{HTML}{9E9E9E}
\definecolor{CompleteCaseColour}{HTML}{E69F00}
\definecolor{ZeroColour}{HTML}{D55E00}
\definecolor{GaussianMeanColour}{HTML}{0072B7}
\definecolor{GaussianMLEColour}{HTML}{56B4E9}
\definecolor{ProposedColour}{HTML}{009F73}

\pgfplotsset{
    discard unless method/.style={
        x filter/.code={
            \edef\MethodInCSV{\thisrowno{1}}
            \edef\RequestedMethod{"#1"}
            \ifx\MethodInCSV\RequestedMethod
            \else
                
            \fi
        }
    }
}

\newcommand{\PlotBand}[4]{%
    \addplot[
        name path=#2upper,
        draw=none,
        forget plot
    ]
    table[
        col sep=comma,
        x index=0,
        y index=4,
        discard unless method={#3}
    ]{#1};

    \addplot[
        name path=#2lower,
        draw=none,
        forget plot
    ]
    table[
        col sep=comma,
        x index=0,
        y index=3,
        discard unless method={#3}
    ]{#1};

    \addplot[
        fill=#4,
        fill opacity=0.15,
        draw=none,
        forget plot
    ]
    fill between[of=#2upper and #2lower];
}

\newcommand{\PlotCurve}[7]{%
    \addplot[
        color=#3,
        #4,
        line width=#6,
        mark=#5,
        mark size=#7,
        mark options={
            solid,
            draw=#3,
            fill=#3,
            line width=0.5pt
        }
    ]
    table[
        col sep=comma,
        x index=0,
        y index=2,
        discard unless method={#2}
    ]{#1};
}

\newcommand{\MaybeLegend}[2]{%
    \ifnum#1=1\relax
        \addlegendentry{#2}%
    \fi
}

\newcommand{\PlotAllMethods}[3]{%
    \PlotBand{#1}{#2cd}{Complete-data MLE}{CompleteDataColour}
    \PlotBand{#1}{#2cc}{Complete-case MLE}{CompleteCaseColour}
    \PlotBand{#1}{#2zi}{Zero imputation}{ZeroColour}
    \PlotBand{#1}{#2gm}{Gaussian mean imputation}{GaussianMeanColour}
    \PlotBand{#1}{#2gl}{Gaussian MLE}{GaussianMLEColour}
    \PlotBand{#1}{#2sa}{Proposed SA}{ProposedColour}

    \PlotCurve{#1}{Complete-data MLE}
        {CompleteDataColour}{solid}{diamond*}{1.2pt}{1.9pt}
    \MaybeLegend{#3}{Oracle logistic}

    \PlotCurve{#1}{Complete-case MLE}
        {CompleteCaseColour}
        {dash pattern=on 7pt off 3pt}
        {x}{0.72pt}{1.8pt}
    \MaybeLegend{#3}{Complete-case MLE}

    \PlotCurve{#1}{Zero imputation}
        {ZeroColour}
        {dash pattern=on 2pt off 2pt}
        {triangle*}{0.72pt}{1.8pt}
    \MaybeLegend{#3}{Zero imputation}

    \PlotCurve{#1}{Gaussian mean imputation}
        {GaussianMeanColour}
        {dash pattern=on 1pt off 3pt on 4pt off 3pt}
        {asterisk}{0.72pt}{2.1pt}
    \MaybeLegend{#3}{Gaussian mean imputation}

    \PlotCurve{#1}{Gaussian MLE}
        {GaussianMLEColour}
        {densely dotted}
        {square*}{0.72pt}{1.7pt}
    \MaybeLegend{#3}{Gaussian MLE}

    \PlotCurve{#1}{Proposed SA}
        {ProposedColour}{solid}{*}{1.2pt}{1.8pt}
    \MaybeLegend{#3}{Proposed algorithm}
}

\newif\ifshowtcb
\showtcbfalse

\ifshowtcb
\else
  \renewenvironment{tcolorbox}[1][]{}{}
\fi

\newcommand{\PopOp}{\Psi}
\newcommand{\StochOp}{\widehat{\PopOp}}

\begin{document}

\begin{center}

\vspace*{2sp}
{\LARGE\bfseries
\makebox[\linewidth][c]{Assumption-lean logistic regression with missing covariates}
}
\vspace*{1.5mm}

{\large{
\begin{tabular}{ccc}
Jyotishka Ray Choudhury$^\star$, Kabir Aladin Verchand$^\circ$, \\[0.5mm]
Richard J. Samworth$^\ddagger$, Ashwin Pananjady$^{\star, \dagger}$
\end{tabular}
}}
\vspace*{4mm}

\begin{tabular}{c}
$^\star$School of Industrial and Systems Engineering, Georgia Institute of Technology \\[0.5mm]
$^\circ$Department of Data Sciences and Operations, University of Southern California \\[0.5mm]
$^\ddagger$Statistical Laboratory, University of Cambridge \\[0.5mm]
$^\dagger$School of Electrical and Computer Engineering, Georgia Institute of Technology
\end{tabular}

\vspace*{4mm}

\today

\vspace*{5.5mm}

\begin{abstract}
Missing covariates are frequently encountered in supervised learning problems, and classical methods for estimation using such data use carefully chosen imputation schemes for missing data, or likelihood approximations that lead to nonconvex $M$-estimation problems. These methods and their relatives are suitable for scenarios in which the covariate distribution is known (or can be reliably estimated), and more broadly, have enjoyed tremendous success in linear models. But even in basic nonlinear problems such as logistic regression in moderate dimensions, such methods can experience drastic failure modes when the covariate distribution is unknown.

Motivated by the need for reliable alternatives, we consider the problem of parameter estimation in logistic regression with missing covariates.
Crucially, we operate in the assumption-lean setting where the covariate distribution is unknown (but bounded). We design a stochastic approximation method that is based on $Z$-estimation with a novel monotone operator, and establish that our algorithm is computationally efficient and achieves provable signal recovery at parametric rates under the hypothesis that covariates are missing completely at random. Our theory sharply characterizes the $\ell_2^2$ risk of the estimator in terms of the missingness profile, accommodating heterogeneous observation probabilities. Importantly, it shows that our method always outperforms the de facto ``complete-case'' estimator that ignores observations with any missing data. Even in the setting with homogeneous missingness (in which each covariate is observed independently with probability $q$), our bounds exhibit intricate and nonstandard dependence on $q$ that can yield significant improvements over using only complete cases. We complement our upper bounds with new information-theoretic lower bounds that show that this intricate dependence on $q$ is fundamental in a minimax sense.
\end{abstract}
\end{center}

\tableofcontents


\section{Introduction}
\label{sec:intro}

Regression methods are often developed under the assumption that all covariates are fully observed. In practice, however, missing covariates arise routinely in almost every domain of application; for example, 
through non-response in surveys and observational studies~\citep{Dardanoni2011impute}, 
measurement failures in sensor networks~\citep{ZhangEtAl2013imputevsignoresensor,DecorteEtAl2024wirelesssensor}, 
integration of data collected from multiple sources~\citep{YuLiShenLiu2020multimodal,XueQu02021multisource},
and in electronic health records~\citep{WellsEtAl2013missingEHR,Groenwold2020EHR}, among others. 
Classical approaches for handling missing covariates include 
complete-case analysis~\citep{GreenlandFinkle1995epidemiologic,KoulEtAl-2012-transfercomplete,RossEtAl-2020-completecase}, 
imputation methods such as mean or regression imputation~\citep{Wilks1932missing,Buck-1960-regressionimpute, AfifiElashoff1966missing, LeMorvanEtAl2021imputationpredict,AymeEtAl2023imputation,BertsimasEtAl2024simpleimputation}, 
and model-based procedures---both Bayesian~\citep{IbrahimChenLipsitz2002bayesianGLM,HuangChenIbrahim2005bayesianGLM,KalayliogluOzturk2013bayesianGLM} and likelihood-based~\citep{Anderson-1957-MLEmissing,Ibrahim1990incompleteGLM,Little1992missingX}---that specify a distribution for the missing covariates.
In particular, a common strategy is to first construct a completed version of the covariates or a functional of the covariates, either through direct imputation or using an estimated model for the missing values, and then apply some procedure to the completed data. 
While variants of such strategies have been shown to provably succeed in linear models~\citep[e.g.,][]{loh2012high,datta2017cocolasso,KabirEtAl2020imputationHDLR,AymeEtAl2023imputation}, they can behave poorly in nonlinear regression problems~\citep[see, e.g.,][]{VachSchumacher1993logistic,LoboEtAl2024primerclassification,ma2026deep}, especially in high dimensions. 

To see the core issue with, e.g., imputation, consider a well-specified logistic regression model with response $Y \in \{-1,+1\}$, covariate $X \in \mathbb{R}^d$ with $\mathbb{E}[X] = 0$, 
and regression parameter $\thetastar \in \Theta \subseteq \mathbb{R}^d$:
\begin{equation}
\label{eq:logistic-def-basic}
    \mathbb P\bigl(Y=1 \bigm| X\bigr)
    =
    \sigma\bigl(X^\top\thetastar\bigr),
    \qquad
    \sigma(t):=\frac{1}{1+e^{-t}},
\end{equation}
and suppose that the missing coordinates are replaced by their expected value zero. If the covariates are missing completely at random (MCAR), then there is a masking vector $M$ independent of $X,Y$, and we observe~$M$ and $Z := X \odot M$ instead of $X$. Logistic regression applied to $Z$ targets the population minimizer
\begin{equation}
\label{eq:zero-impute-obj}
    \theta_{\scriptscriptstyle\mathrm{ZI}}
    \in
    \argmin_{\theta \in \Theta} \,
    \underbrace{
    \mathbb{E}\left[
        \log\bigl\{1+\exp(-YZ^\top\theta)\bigr\}
    \right]}_{\mathscr{L}_{\scalebox{.9}{$\scriptscriptstyle\mathrm{ZI}$}}(\theta)}.
\end{equation}

Generally speaking, the estimate $\theta_{\scriptscriptstyle\mathrm{ZI}}$ will not coincide with the true parameter $\thetastar$. Indeed, at the true parameter, the gradient of $\mathscr{L}_{\scriptscriptstyle\mathrm{ZI}}$ satisfies
\begin{align*}
    \nabla \mathscr{L}_{\scriptscriptstyle\mathrm{ZI}}(\thetastar)
    &=
    \mathbb{E}\biggl[
        -YZ \cdot \frac{\exp(-YZ^\top\thetastar)}{1+\exp(-YZ^\top\thetastar)}
    \biggr] \\ 
    &=
    -\,\mathbb{E}\left[
        YZ \cdot \sigma(-YZ^\top\thetastar)
    \right] \\ 
    &=
    \mathbb{E}\Big[
        Z\cdot\mathbb{E}\bigl[
            -Y\sigma(-YZ^\top\thetastar) \bigm| X,M
        \bigr]
    \Big]
    =
    \mathbb{E}\Big[
        Z\Big\{
            \sigma(Z^\top\thetastar) - \sigma(X^\top\thetastar)
        \Big\}
    \Big],
\end{align*}
which is non-zero in general. Thus, zero-imputing the covariates and minimizing the function $\mathscr{L}_{\scriptscriptstyle\mathrm{ZI}}$ may no longer target the original complete-data parameter $\theta^\star$. Notably, this discrepancy occurs at the population level, i.e., even if we had access to an infinite amount of data.

The source of this population-level discrepancy is the delicate interaction between missingness and the nonlinearity of the logistic link function $\sigma(\cdot)$. Concretely, for a fixed missingness pattern, let $S \subseteq [d]$ denote the set of observed coordinates. By the tower property of conditional expectation, 
\begin{align}
    \mathbb P(Y=1\mid X_S)
    &=
    \mathbb{E}\bigl[
        \mathbb P(Y=1\mid X) \bigm| X_S
    \bigr]
    =
    \mathbb{E}\Bigl[
        \sigma\big(X_S^\top\thetastar_S + X_{S^c}^\top\thetastar_{S^c}\big) \Bigm| X_S
    \Bigr],
\end{align}
which is generally not itself a logistic function of $X_S$ since 
$
\mathbb{E}\big[\sigma\bigl(h(X)\bigr) \,|\, X_S\big] \neq \sigma\big(\mathbb{E}\,\bigl[h(X) \,|\, X_S\bigr]\big).
$ 
Any imputation-based procedure uses a surrogate (e.g.~$Z$ in zero-imputation) with the hope that the response can be modeled via logistic regression using this surrogate and the same population parameter. But such a surrogate may not be obtainable via imputation alone.

As an alternative approach, if the law $P_X$ of the covariates is known, then one may attempt to integrate over the missing coordinates, either explicitly or through likelihood-based procedures such as the expectation--maximization (EM) algorithm \citep{DempsterLairdRubin1977EMalgo,Wu1983EMconvergence} or its variants (see, e.g., \citet{Roche2012EMvariants}) to approximate $\P(Y=1 \,|\, X_S)$. If $P_X$ is unknown, however, then estimating the conditional distribution of the missing coordinates given the observed covariates corresponds to a density estimation problem, which can be prohibitive in high dimensions~\citep{HasminskiiIbragimov1990densityestm}. While the curse of dimensionality can be avoided by assuming some parametric form for $P_X$ (such as Gaussianity), procedures designed under this assumption may fail when the assumption is violated. Indeed, most real data sets are non-Gaussian and there is a need for procedures that do not make strong parametric assumptions on the covariate distribution.

\subsection{A motivating numerical experiment}

As an illustration of the difficulties in this problem, \Cref{fig:intro-simulation} considers a well-specified logistic model with bounded, non-Gaussian covariates in dimension $d$ and independent coordinate-wise homogeneous MCAR missingness in which each coordinate is observed i.i.d. (and independently of the data) with probability $q$. We consider two settings, viz. $(d,q) = (10,0.40)$ and $(100,0.85)$; full details of this experiment are provided in \Cref{sec:non-Gaussian-simulation-details}. 

With dotted lines, we plot the estimation error of four existing procedures:
(a) \emph{Complete-case logistic regression}, which performs vanilla logistic regression on only those observations where there is no missingness; (b) \emph{Zero-imputation} followed by logistic regression, as described in \Cref{eq:zero-impute-obj}; (c)~\emph{Gaussian conditional-mean imputation}, which imputes the missing coordinates by their conditional means assuming the data is Gaussian with the same (known) first two moments, so $X_{\mathrm{mis}}$ is imputed with $\mathbb{E}_G[X_{\mathrm{mis}} \mid X_{\mathrm{obs}}]$, followed by logistic regression; and (d) \emph{Gaussian MLE}, which uses the same working model, but instead numerically integrates the logistic likelihood over the conditional distribution of the missing covariates, with contribution $\mathbb{E}_G[\sigma(YX^\top\theta)\mid X_{\mathrm{obs}}]$ to the observed likelihood. 
In addition, the \emph{complete-data MLE}, which assumes oracle access to the covariates (without missingness), is included as an oracle benchmark and its performance is plotted with the gray solid line.


\begin{figure}[H]
    \centering

    \def\LeftCSV{%
        sim-results/%
        non_gaussian_estimation_summary_d_10_q_0.4_seed_2026.csv}
    \def\RightCSV{%
        sim-results/%
        non_gaussian_estimation_summary_d_100_q_0.85_seed_2026.csv}

    \def\PanelGap{0.015\textwidth}

    \pgfplotslegendfromname{nonGaussianLegend}

    \vspace{-2mm}

    \noindent
    \begin{minipage}[t]{0.016\textwidth}
        \vspace{0pt}
        \parbox[c][1.90in][c]{\linewidth}{%
            \centering
            \rotatebox{90}{%
                {\footnotesize $\|\widehat\theta-\thetastar\|_{2}$}%
            }%
        }
    \end{minipage}%
    \begin{minipage}[t]{0.482\textwidth}
        \vspace{0pt}
        \raggedleft

        \begin{tikzpicture}[trim axis left,trim axis right]
            \begin{axis}[
                width=0.84\linewidth,
                height=1.90in,
                scale only axis,
                xmode=log,
                ymode=log,
                log basis x=10,
                log basis y=10,
                xmin=180,
                xmax=1.1e6,
                ymin=1e-2,
                ymax=3,
                axis lines=box,
                xlabel={Sample size $(n)$},
                xtick pos=bottom,
                ytick pos=left,
                yticklabel pos=left,
                xtick={
                    200,500,1000,2000,5000,10000,
                    20000,50000,100000,250000,500000,1000000
                },
                xticklabels={
                    {},{},{$10^3$},{},{},{$10^4$},
                    {},{},{$10^5$},{},{},{$10^6$}
                },
                grid=major,
                major grid style={
                    draw=gray!18,
                    line width=0.3pt
                },
                axis line style={
                    draw=black!55,
                    line width=0.5pt
                },
                tick style={black!55},
                tick align=outside,
                tick label style={font=\scriptsize},
                label style={font=\footnotesize},
                unbounded coords=discard,
                filter discard warning=false,
                legend to name=nonGaussianLegend,
                legend columns=3,
                legend cell align=left,
                legend style={
                    draw=none,
                    font=\footnotesize,
                    /tikz/every even column/.append style={
                        column sep=0.9em
                    },
                    row sep=0.15em
                }
            ]
                \PlotAllMethods{\LeftCSV}{left}{1}
            \end{axis}
        \end{tikzpicture}

        \par\smallskip
        \makebox[\linewidth][r]{%
            \makebox[0.84\linewidth][c]{%
                \footnotesize
                \textbf{(a)}\; $d=10,\,q=0.40$%
            }%
        }
    \end{minipage}%
    \hspace{\PanelGap}%
    \begin{minipage}[t]{0.482\textwidth}
        \vspace{0pt}
        \raggedright

        \begin{tikzpicture}[trim axis left,trim axis right]
            \begin{axis}[
                width=0.84\linewidth,
                height=1.90in,
                scale only axis,
                xmode=log,
                ymode=log,
                log basis x=10,
                log basis y=10,
                xmin=180,
                xmax=1.1e6,
                ymin=1e-2,
                ymax=3,
                axis lines=box,
                xlabel={Sample size $(n)$},
                xtick pos=bottom,
                ytick pos=right,
                yticklabel pos=right,
                xtick={
                    200,500,1000,2000,5000,10000,
                    20000,50000,100000,250000,500000,1000000
                },
                xticklabels={
                    {},{},{$10^3$},{},{},{$10^4$},
                    {},{},{$10^5$},{},{},{$10^6$}
                },
                grid=major,
                major grid style={
                    draw=gray!18,
                    line width=0.3pt
                },
                axis line style={
                    draw=black!55,
                    line width=0.5pt
                },
                tick style={black!55},
                tick align=outside,
                tick label style={font=\scriptsize},
                label style={font=\footnotesize},
                unbounded coords=discard,
                filter discard warning=false
            ]
                \PlotAllMethods{\RightCSV}{right}{0}
            \end{axis}
        \end{tikzpicture}

        \par\smallskip
        \makebox[\linewidth][l]{%
            \makebox[0.84\linewidth][c]{%
                \footnotesize
                \textbf{(b)}\; $d=100,\,q=0.85$%
            }%
        }
    \end{minipage}

\caption[]{
\emph{Estimation error comparison under MCAR covariate missingness with coordinate-wise observation probability $q$.} The oracle logistic MLE has access to the complete, unmasked covariates, and thus serves as the gold standard. The remaining curves correspond to the complete-case MLE, logistic MLE after zero or Gaussian conditional-mean imputation, the observed-data MLE under a moment-matched Gaussian model, and our proposed SA-based estimator.
The curves plot the RMSE $\|\widehat\theta-\thetastar\|_2$ over $100$ replications against the sample size $n$. Shaded regions give pointwise $95\%$ confidence intervals across replications. 
}
\label{fig:intro-simulation}
\end{figure}
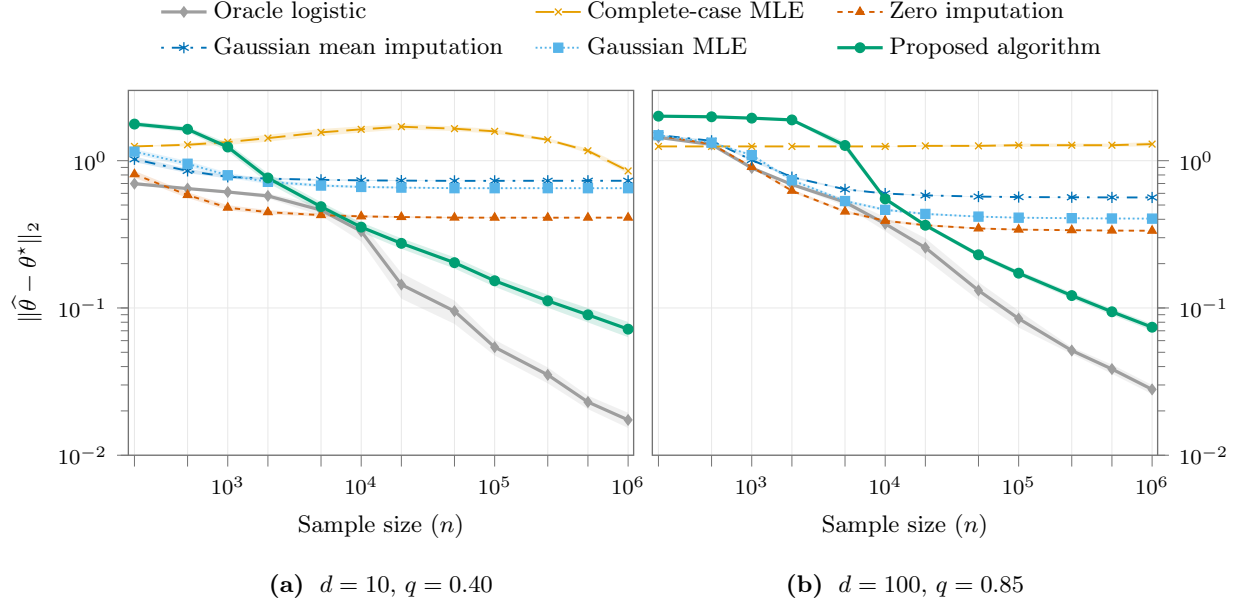

As shown by the figure, the oracle estimator is consistent in both experiments, so it is possible to perform consistent estimation at least when there is no missingness. However, all the imputation-based and Gaussian likelihood-based procedures suffer a substantial estimation error even for large sample sizes and are inconsistent in these experiments. As described above, this is because these procedures need not target $\thetastar$ for a general covariate distribution. The complete-case MLE avoids this form of misspecification, but discards every incomplete observation. However, under coordinate-wise MCAR, a covariate vector is fully observed only with probability $q^d$, which can be extremely small when $q$ is small or $d$ is large. Consequently, using only complete cases leads to poor performance in both cases in Figure~\ref{fig:intro-simulation}.

Given that existing methods fail even in this simple example, the following basic question forms the primary motivation of this paper:
\begin{quote}
\centering
\emph{
Can the parameter $\thetastar$ of the original  logistic regression model be estimated from incomplete covariates without modeling or estimating the unknown covariate distribution, and what is the intrinsic statistical cost of missingness?
}
\end{quote}

We study this question under a missing-completely-at-random (MCAR) mechanism with each coordinate observed independently, given that it appears to be unresolved even in this basic setting. We allow for heterogeneous missingness probabilities across coordinates. Our aim is to understand whether the original logistic parameter $\thetastar$ can be recovered without any knowledge of the covariate distribution, and how the estimation error depends on the severity of missingness. Ideally, we would like both information-theoretic and methodological answers to these questions. Indeed, one of our main contributions---described in detail in \Cref{sec:contributions} below---is an algorithm (plotted with a green solid line in \Cref{fig:intro-simulation}) that achieves consistent estimation when existing methods fail. This estimator only has access to missing data, and yet performs comparably to the oracle MLE that performs logistic regression using the covariates without missingness.

\subsection{Contributions and organization} \label{sec:contributions}

Our main contributions are summarized below.
\begin{enumerate}[label=\arabic*.]

\item We propose an algorithm for estimating the logistic regression parameter $\thetastar$ from incomplete covariates with MCAR missingness. The algorithm is assumption-lean, in the sense that it does not rely on any information about the covariate distribution and so does not impute missing entries. Instead, we formulate a $Z$-estimation problem by constructing a monotone operator at the population level, ensuring that the unique zero of this operator is given by the desired signal $\thetastar$. To obtain a data-driven estimation procedure, we construct unbiased estimates of this operator from each incomplete observation and utilize these estimates within a projected stochastic approximation algorithm (\Cref{alg:Proj-SA}).

\item We provide a non-asymptotic analysis of the resulting estimator under the coordinate-wise MCAR missingness mechanism. Under the assumptions that $\norm{X}_2\le1$, $\E[XX^\top]\succeq\mu I$, and $\norm{\thetastar}_2\le R$, we obtain an error bound that depends on the interaction between dimension, the constants $(R, \mu)$, and the observation probabilities. Notably, our estimator is computationally efficient and always improves on the complete-case estimator, in all regimes of the problem. We also showcase interpretable improvements over complete-case estimation when there is homogeneous missingness in high-dimensional problems. These results are presented in \Cref{subsec:main-upper}.

\item We establish a minimax lower bound over the above class of covariate distributions, which matches our upper bound up to universal constants when the missingness mechanism is homogeneous. 
This lower bound is the result of one of our main technical innovations:
we carefully construct hypotheses conforming to well-specified logistic regression models such that these are indistinguishable when too few covariates are observed.
Notably, our construction utilizes a Boolean function to conceal carefully designed parameter perturbations in the appropriate likelihood ratios under missingness. 
The main lower bound result is presented in \Cref{subsec:main-lower}.

\end{enumerate}

The rest of this paper is organized as follows. In \Cref{sec:main}, we formally introduce the setup, develop our distribution-free stochastic approximation procedure, state an upper bound on the risk of the algorithm, and establish a matching information-theoretic minimax lower bound.
The proof of the main upper bound result is given in \Cref{sec:proof-upper}, while \Cref{sec:proof-lower} proves the minimax lower bound. We conclude with a discussion of possible extensions in Section~\ref{sec:discussion}. Proofs of supporting technical lemmas are deferred to the appendices.

\subsection{Related work}\label{sec:related}

The missing data literature is vast, and we cannot hope to survey it here. We restrict our discussion to related approaches, both to the logistic regression problem with missing data and simpler variants. Our discussion is organized under three headings. 

\paragraph{Likelihood-based methods and their relatives.}  When data are missing completely at random (or more broadly, missing at random~\citep{seaman2013meant}), the missingness mechanism is ignorable and likelihood-based methods can be utilized based on the observed data.  We direct the reader to the textbook by~\citet[Part II]{little2019statistical} for an exhaustive overview of likelihood-based methods in the presence of missing data. Since the negative log-likelihood is typically nonconvex in most of these problems, the most prevalent technique for optimizing it is the EM algorithm developed by~\citet{DempsterLairdRubin1977EMalgo}; see also~\citet{Wu1983EMconvergence,Ibrahim1990incompleteGLM,horton1999maximum}.  While valid, likelihood-based methods suffer from two distinct issues: (i) they require knowledge of the distribution of the covariates and (ii) since the optimization problem is nonconvex, it may be difficult to guarantee that we reach statistically useful solutions in high dimensions.  Indeed, even if the covariate distribution is known and algorithms such as EM~\citep{DempsterLairdRubin1977EMalgo,Wu1983EMconvergence} can be employed, their convergence properties are delicate; in high dimensions, provably optimal solutions are known to be attainable only in special cases such as the linear model with Gaussian covariates~\citep{balakrishnan2017statistical}.  Aside from EM being analyzable, the linear model is also special for other reasons. For example, the negative log-likelihood under Gaussian noise is quadratic in the covariates, and good optimization algorithms can be designed provided the covariance matrix can be estimated well~\citep{lounici2014high}. In particular, it is now well known that computationally efficient optimization for the linear model with missing data need not rely on estimating the full joint distribution of the covariates~\citep{loh2012high,risebrow2026semi}. Another example of such a phenomenon is in learning Ising models with missing data~\citep{goel2019learning}.

\paragraph{Imputation-based approaches.} Most of the existing work on logistic regression (or more broadly nonlinear regression) with missing covariates operates in the low-dimensional and/or asymptotic regime.  In this regime, a complete-case analysis is consistent and the focus is on valid inference and efficiency.  As such, variants of multiple imputation~\citep{rubin1978multiple,van2011mice} form the gold standard.  In frequentist approaches, this requires the computation of an initial consistent estimate of the coefficients~\cite[\S 14.2]{tsiatis2006semiparametric} which is in turn used to form several complete data sets~\citep[see][for asymptotic properties]{wang1998large}.  Importantly, these approaches are not distribution-free and typically require knowledge of the joint distribution of the covariates in order to perform imputation.  More recently, approaches inspired by prediction-powered inference (PPI)~\citep{angelopoulos2023prediction} have been proposed to leverage black-box methods for imputation.  For instance,~\citet{zhao2025imputation}, \citet{chen2025unified}, \citet{kluger2025prediction}, \citet{duan2025imputation} and \citet{xu2025blockwise} propose variants of this theme when covariates can be missing.  We note that these works consider inference of the population coefficients and are focused on using black-box imputation to improve upon a complete-case analysis.  By contrast, and as we establish shortly,
our estimator is never worse than the complete-case estimator, and when the dimension is large, can bypass the curse of dimensionality incurred by a complete-case analysis. Thus, we offer a significant, imputation-free improvement over complete-case analysis under MCAR covariates.  

\paragraph{Semiparametric approaches.}
The PPI-inspired methods~\citep[e.g.,][]{zhao2025imputation,chen2025unified} are members of a much broader class of semiparametric methods, which form an attractive alternative to likelihood-based methods.
In particular,~\citet{robins1994estimation} propose a general class of augmented inverse propensity weighting (AIPW) methods. It is known that among AIPW methods, there always exists a semiparametrically efficient estimator; however, in general this estimator may not admit a closed form and may be difficult to compute.  In order to deal with this, it is common to restrict these estimators in some way~\citep[see, e.g.,][\S 4.1]{sun2018inverse}.  Interestingly, the estimator that we propose---based on a particular estimating equation---can be understood as a member of the AIPW family. It is simple to write down, computationally efficient to solve, and is optimal in a minimax sense.  We refer the reader to~\citet{tsiatis2006semiparametric} for a textbook treatment of these approaches in the more general setting of coarsened data and~\citet{stefanski1985covariate,bickel1987efficient} for similar semiparametric methods in the related error-in-variables model for logistic regression.

\subsection{Notation}

For a positive integer $k$, we write $[k]:=\{1,\ldots,k\}$. For vectors $u,v\in\R^d$, we use $\ip{u}{v}:=u^\top v$ to denote their Euclidean inner product. We use $\norm{u}$ and $\norm{u}_2$ interchangeably to denote the Euclidean $(\ell_2)$ norm of vector $u$; no other norms are used in the paper. For $S\subseteq[d]$, we write $u_S:=(u_j)_{j\in S}$ and $S^c:=[d]\setminus S$. We let $e_j$ denote the $j$-th standard basis vector, $\mathbf{1}_k$ the all-ones vector in $\R^k$, and $I_k$ the $k\times k$ identity matrix, where we write $I$ when the dimensions are clear from context. For symmetric matrices $A,B$, we write $A\succeq B$ to mean that $A-B$ is positive semi-definite. We write $\lambda_{\min}(A)$ for the smallest eigenvalue of $A$.
For a random variable $W$, $\mathsf{Law}(W)$ denotes its distribution. 
We write $\mathcal{P}(\mathcal{X})$ for the set of probability measures on a measurable space $\mathcal{X}$; $\BER(q)$ for the Bernoulli distribution with success probability $q$; $U\indep V$ for independence; and $\mathbbm{1}\{\mathcal{E}\}$ for the indicator of an event $\mathcal{E}$. For vectors $a,b$ having the same dimension, $a \odot b$ represents their Hadamard product (i.e. coordinate-wise multiplication).
For probability measures $P$ and $Q$, we write $\TV(P,Q)$ for their total variation distance; and when $P \ll Q$, we write $\chi^2(P,Q)$ for their $\chi^2$--divergence; $P^{\otimes n}$ denotes the $n$--fold product measure corresponding to the joint law of $n$ independent observations from $P$.
Throughout, we denote universal constants by $c, C, c_0, C_0, c_1, C_1, c_2, C_2$, etc., whose values may change from line to line depending on the context, unless otherwise specified. For non-negative quantities $A,B$, we write $A \lesssim B$ if $A \le cB$ for a universal constant $c > 0$, and define $A \gtrsim B$ analogously; we write $A \asymp B$ when both $A \lesssim B$ and $A \gtrsim B$ hold simultaneously.


\section{Methodology and main results}
\label{sec:main}

We now formally set up the model, beginning at the population level. Let $X$ denote a covariate vector in dimension $d:=\operatorname{dim}(X)$ and let $Y\in\{-1,1\}$ denote the corresponding response. 
We do not assume that the covariate distribution $P_X$ is known; instead we assume that it comes from a flexible class of distributions. For $\mu>0$, let $\mathcal{P}_{\mu}$ denote the collection of distributions of finite-dimensional covariate vectors satisfying
\begin{equation}
\label{eq:def-P-mu}
    \mathcal{P}_{\mu}
    :=
    \Bigl\{
        \mathsf{Law}(X):
        \norm{X}_2\le1\ \text{a.s.},
        \ \lambda_{\min}\bigl(\E[XX^\top]\bigr)\ge\mu
    \Bigr\}.
\end{equation}
Throughout, we assume that $P_X\in\mathcal{P}_{\mu}$. Given $X$, we assume that the response $Y$ is drawn from a well-specified logistic regression model with true parameter $\thetastar \in \R^{d}$. Concretely, we have
\[
    \bbP(Y=y\mid X)
    =
    \frac{\exp\left(\frac{1}{2}y\ip{X}{\thetastar}\right)}
    {\exp\left(\frac{1}{2}\ip{X}{\thetastar}\right)+\exp\left(-\frac{1}{2}\ip{X}{\thetastar}\right)}.
\]
We also assume that the true parameter $\thetastar$ is bounded. Concretely, for any $R>0$, define
\[
    \Theta_R:=\bigl\{\theta\in\R^{d} : \norm{\theta}_2\le R\bigr\},
\]
 and assume that $\thetastar\in\Theta_R$. 

By definition, if $P_X \in \mathcal{P}_{\mu}$, then 
\begin{equation}
\label{ineq:dimension-mu-relation}
    \operatorname{dim}(X) \cdot \mu = d \mu
    \le
    \operatorname{tr}\bigl(\E[XX^\top]\bigr)
    =
    \E\norm{X}_2^2
    \le1.
\end{equation}
In particular, the dimension and curvature constant $\mu$ are implicitly related even though the definition of $\mathcal{P}_{\mu}$ is dimension-free. Accordingly, while our main theorem is stated with explicit dimension dependence, it will be convenient later in the paper to parameterize problem instances using only the curvature constant~$\mu$.

We assume that the response $Y$ is always observed, while  coordinates of $X$ may be missing independently of one another (and independently of the data). 
In particular, for each coordinate $j\in[d]$, we assume a coordinate-wise missing-completely-at-random (MCAR) mechanism, under which the $j$-th observation indicator is denoted by\footnote{In the missing data literature, it is common to denote missing values with explicit indicators such as $\star$ or $\mathtt{NA}$. However, for notational convenience, we work in this paper with a mask vector $M$ and a zero-imputed covariate vector $Z$, noting that this is an equivalent representation of an incomplete covariate vector.}
\[
    M_j:=\mathbbm{1}\{X_j ~\text{is observed}\} \sim\BER(q_j),
\]
where the probability $q_j \in (0,1]$ with which the $j$-th entry is observed is assumed to be known. The masks $M_1,\ldots,M_d$ are independent across coordinates, and $M=(M_1,\ldots,M_d)$ is independent of the data $(X,Y)$. 

We represent the observed covariates via a pair $(M,Z)$, where $Z:=M\odot X$; i.e.,
\[
    Z_j=M_jX_j,\quad\text{for all}\,~ j\in[d].
\]
Thus, for each $j \in [d]$, $Z_j=X_j$ whenever the $j$-th covariate is observed (equivalently, $M_j = 1$), and $Z_j=0$ if it is missing (equivalently, $M_j = 0$). Since $M$ is observed alongside $Z$, the pair $(M,Z)$ contains the same information as the corresponding partially observed covariate vector. Accordingly, we represent each observed (possibly incomplete) data point by $(Y,M,Z)$.

Suppose  $\bigl(X^{(i)},Y^{(i)},M^{(i)},Z^{(i)}\bigr)_{i= 1}^n$ are independent copies of $(X,Y,M,Z)$, and that we observe $\bigl(Y^{(i)},M^{(i)},Z^{(i)}\bigr)_{i= 1}^n$. Our goal is to estimate $\thetastar$ in $\ell_2$ norm. In the next section, we describe an algorithm for performing this task.

\subsection{Distribution-free estimation via stochastic approximation}
\label{subsec:projSA}

Our starting point is to construct a population estimating equation for $\thetastar$ that does not require knowledge of the covariate distribution, i.e. a distribution-free $Z$-estimator~\citep[see, e.g.,][]{robins1994estimation,ReillyPepe-1995-meanscore,TsiatisMa-2004-semiparametric}. In particular, we seek an operator $\PopOp:\Theta_R\to\mathbb R^d$ with the following two properties:
\begin{enumerate}[label=(\roman*)]
    \item $\PopOp(\thetastar)=0$, with $\thetastar$ being the unique root of $\PopOp$ over $\Theta_R$;
    \item $\PopOp(\theta)$ admits an unbiased stochastic estimate based only on a realized sample $(Y,M,Z)$, without modeling or estimating the distribution of $X$.
\end{enumerate}

We choose the following operator:
\begin{equation}
\label{eq:A-theta}
    \PopOp(\theta)
    :=
    \E\Bigl[
        -YX\exp\Bigl\{-\mfrac{1}{2} Y\ip{X}{\theta}\Bigr\}
    \Bigr],
    \qquad\text{for}\,~\theta\in\Theta_R.
\end{equation}

As we show in Lemma~\ref{lem:A-properties} in the appendix, $\PopOp$ is strongly monotone on $\Theta_R$ and has a zero at $\thetastar$. It thus follows that $\widetilde{\theta} = \thetastar$ is the unique solution to the variational inequality
\begin{align} 
\label{eq:population-VI}
\bigl \langle \PopOp(\widetilde{\theta}), \theta - \widetilde{\theta} \bigr \rangle \geq 0, \qquad \text{for all}\,~\theta \in \Theta_R. 
\end{align}

If we can evaluate $\PopOp$ at any point, then there are many (first-order) methods that can solve the variational inequality~\eqref{eq:population-VI}~\citep[see, e.g.][]{Lan2020firstorder}.
However, $\PopOp(\theta)$ cannot be evaluated directly from an incomplete observation (since it involves the full covariate vector $X$) and so our goal is to compute an unbiased estimate. The key observation that enables such an estimate is that the exponential term in $\PopOp(\theta)$ factorizes multiplicatively across coordinates. In particular, its $j$-th entry is given by
\begin{align} 
\label{eq:op-expand}
\bigl[\PopOp(\theta)\bigr]_j = \mathbb{E} \Bigl[-Y X_j \exp\Bigl( -\mfrac{1}{2} Y\theta_j X_j \Bigr) \prod_{k \neq j} \exp\Bigl( -\mfrac{1}{2} Y\theta_k X_k \Bigr)\Bigr].
\end{align}
Now for any Borel measurable function $g_j : \R \to \R$, we have\footnote{Note that there are several choices of an unbiased estimator, including the rescaled zero-imputation $\frac{M_j}{q_j} \cdot g_j(Z_j)$. Our choice is inspired by the fact that we use products of the individual estimators $\widehat{f}_j$~\eqref{eq:f-hat} to obtain our final estimate, so we would like our estimator for the exponential function to evaluate to $1$ when data is unobserved.}
\begin{align}
\label{eq:gjXj-generic-exp}
    \E\Biggl[\;
        \underbrace{g_j(0)+\frac{M_j}{q_j}\bigl\{g_j(Z_j)-g_j(0)\bigr\}}_{=:\, \widehat{g}_j(M_j,Z_j;q_j)} \,\bigg|\, X_j
    \Biggr]
    =
    g_j(0)+\frac{1}{q_j}\cdot q_j \bigl\{g_j(X_j)-g_j(0)\bigr\} + 0
    =
    g_j(X_j).
\end{align}

Therefore, $\widehat{g}_j(M_j,Z_j;q_j)$ is an unbiased estimator of $g_j(X_j)$, conditional on $X_j$. Since the masks $M_1,\ldots,M_d$ are independent, products of these estimators remain unbiased for the corresponding products of the population quantities. Applying this observation to the expansion in \Cref{eq:op-expand} with 
$g_k(X_k) = \exp\bigl\{-\tfrac{1}{2} Y\theta_k X_k\bigr\}$
and
$g_j(X_j) = -Y X_j \exp\bigl\{-\tfrac{1}{2} Y\theta_j X_j\bigr\}$
allows us to construct an unbiased estimate using only $(Y,M,Z)$, as we now describe. For $\theta = (\theta_1, \ldots, \theta_d) \in \Theta_R$ and each $j\in[d]$, define the coordinate-wise, scalar estimates
\begin{equation}
\label{eq:f-hat}
    \widehat{f}_j(\theta)
    :=
    1 + \frac{M_j}{q_j}
    \Bigl[\exp\Bigl\{-\mfrac{1}{2}Y\theta_j Z_j\Bigr\}-1\Bigr].
\end{equation}
\indent
Next, for each $j\in[d]$, define 
\begin{equation}
\label{eq:Ak-hat}
    \bigl[\StochOp(\theta; Y, M, Z)\bigr]_j
    :=
    - Y\frac{Z_j}{q_j}
    \exp\Bigl\{-\mfrac{1}{2} Y\theta_j Z_j\Bigr\}\,
    \prod_{k\ne j}\widehat{f}_k(\theta),
\end{equation}
and let
\begin{equation}
\label{eq:A-hat}
    \StochOp(\theta; Y, M, Z)
    :=
    \Bigl(
        \bigl[\StochOp(\theta; Y, M, Z)\bigr]_1,\ldots,\bigl[\StochOp(\theta; Y, M, Z)\bigr]_d
    \Bigr).
\end{equation}
Then $\StochOp(\theta; Y, M, Z)$ is a distribution-free, conditionally unbiased estimate of the population operator $\PopOp(\theta)$; see \Cref{lem:A-hat-properties}(a). It depends only on the observed data and the  observation probabilities~$q_1, \ldots, q_d$, and does not require any model for the covariate distribution $P_X$. We employ $\StochOp$ to devise a simple, online, projected stochastic approximation algorithm for estimating $\thetastar$. 

For $t=1,2,\ldots,n$, let $\bigl(X^{(t)},Y^{(t)},M^{(t)},Z^{(t)}\bigr)$ be independent copies of $(X,Y,M,Z)$, where at iteration $t$, only $\bigl(Y^{(t)},M^{(t)},Z^{(t)}\bigr)$ is observed. 
Starting from an arbitrary $\widehat\theta^{(0)} \in \Theta_R$, for each $t = 0,1,\ldots,n-1$, we consider the update
\begin{equation}
\label{eq:Proj-SA}
    \widehat\theta^{(t+1)}
    =
    \Pi_{\Theta_R}
    \left\{
        \widehat\theta^{(t)} - \eta_t \,\StochOp\bigl(\widehat\theta^{(t)}; Y^{(t+1)}, M^{(t+1)}, Z^{(t+1)}\bigr)
    \right\},
\end{equation}
where $\Pi_{\Theta_R}$ denotes the $\ell_2$--projection operator onto $\Theta_R$, given by $\Pi_{\Theta_R}(g) = \argmin_{h \in \Theta_R} \norm{g-h}_2$, and $(\eta_t)$ is a sequence of step sizes. For each $t = 0,1,\ldots, n-1$, we prescribe 
\begin{equation}
\label{eq:def-eta-alpha}
\eta_t := \frac{1}{\alpha_R (t+2)}, \qquad\text{where}\,~ \alpha_R:=\frac{\mu}{2\cosh(R/2)}.
\end{equation}

Our procedure, summarized in \Cref{alg:Proj-SA} below, uses a single observation at each iteration. 
After processing $n$ observations, the corresponding output of the algorithm is $\widehat\theta^{(n)}$.

\vspace{-2.5mm}
\begin{breakablealgorithm}
\small
\caption{Projected SA for distribution-free logistic regression with MCAR missing covariates}
\label{alg:Proj-SA}
\begin{algorithmic}[1]
\Require Observations $(Y^{(i)},M^{(i)},Z^{(i)})_{i=1}^n$, $R > 0$, observation probability vector $\boldsymbol{q} = (q_1, \ldots, q_d)$, and curvature constant $\mu > 0$.
\vspace{1mm}
\State Initialize $\widehat\theta^{(0)}\in\Theta_R$ and set $\alpha_R:=\mu/\{2\cosh(R/2)\}$.
\vspace{1mm}
\For{$t=0,1,\ldots,n-1$}
\vspace{1mm}
    \State Set $\eta_t\gets\{\alpha_R(t+2)\}^{-1}$
    \vspace{1mm}
    \State For each $j\in[d]$, set
    \[
        \widehat f_{t+1,j}\bigl(\widehat\theta^{(t)}\bigr)
        :=
        1+\frac{M^{(t+1)}_j}{q_j}
        \Bigl[
            \exp\Bigl\{-\mfrac{1}{2} Y^{(t+1)}\widehat\theta^{(t)}_{j}Z^{(t+1)}_j\Bigr\}-1
        \Bigr]
    \]
    \State For each $j\in[d]$, set
    \[
        \bigl[\StochOp_{t+1}\bigl(\widehat\theta^{(t)}\bigr)\bigr]_j
        :=
        -Y^{(t+1)}\frac{Z^{(t+1)}_j}{q_j} \,
        \exp\Bigl\{-\mfrac{1}{2} Y^{(t+1)} \widehat\theta^{(t)}_j Z^{(t+1)}_j\Bigr\} \,
        \prod_{k\ne j}\widehat f_{t+1,k}\bigl(\widehat\theta^{(t)}\bigr)
    \]
    \State Set $\StochOp_{t+1}\bigl(\widehat\theta^{(t)}\bigr) := \Bigl(\bigl[\StochOp_{t+1}\bigl(\widehat\theta^{(t)}\bigr)\bigr]_1,\ldots,\bigl[\StochOp_{t+1}\bigl(\widehat\theta^{(t)}\bigr)\bigr]_d\Bigr)$
    \vspace{1mm}
    \State Update~ $\widehat\theta^{(t+1)}\gets\Pi_{\Theta_R}\bigl\{\widehat\theta^{(t)} - \eta_t\,\StochOp_{t+1}\bigl(\widehat\theta^{(t)}\bigr)\bigr\}$
    \vspace{2mm}
\EndFor
\vspace{1mm}
\State \Return $\widehat\theta^{(n)}$
\vspace{1mm}
\end{algorithmic}
\end{breakablealgorithm}
\vspace{1mm}

\begin{remark}[An underlying population objective]
\label{rem:Phi-objective}
Although our methodology is formulated directly in terms of the estimating equation $\PopOp(\thetastar)=0$, the operator $\PopOp$ is also the gradient of a natural convex population objective. In particular, if we define
\begin{equation}
\label{eq:Phi-theta}
    \Phi(\theta)
    :=
    2\,\E\Bigl[
        \exp\Bigl\{-\mfrac{1}{2}Y\ip{X}{\theta}\Bigr\}
    \Bigr],
    \qquad\text{for}\,~\theta\in\Theta_R,
\end{equation}
then $\nabla\Phi(\theta)=\PopOp(\theta)$. The function $\Phi$ is the population exponential loss, a classical convex surrogate in binary classification; see, e.g., \citet{FriedmanHastieTibshirani2000logreg,Zhang2004classification,BartlettJordanMcAuliffe2006}. Thus, our stochastic approximation procedure can equivalently be viewed as projected stochastic gradient descent on $\Phi$, where we use the stochastic gradient $\StochOp(\theta; Y, M, Z)$ as an unbiased estimator of $\nabla \Phi(\theta)$.
\end{remark}

\begin{remark}
The radius $R$ should be viewed as an a priori upper bound on the size of the regression parameter, rather than as a quantity that must be known exactly. In particular, if $\norm{\thetastar}_2\le R'$ and only an upper bound $R\ge R'$ is available, then \Cref{alg:Proj-SA} may simply be run over $\Theta_R$. The choice of $R$ enters only through the projection set and the choice of the step size; the construction of the stochastic operator $\StochOp$ itself does not depend on $R$.
\end{remark}

\subsection{Upper bounds on the risk of the algorithm}
\label{subsec:main-upper}

We next establish a finite-sample guarantee for the projected SA procedure in our setting where the data is missing completely at random but with heterogeneous missingness across coordinates. 
Denote the vector of observation probabilities by $\boldsymbol{q} := (q_1, \ldots, q_d)$. 

It is useful to set up some notation before we state our theorem. For $q \in (0,1]$, define the function
$\phi_q: [0, 1] \to [0,\log(1/q)]$ as
\begin{align} \label{eq:phi-function}
\phi_q(s) := \log \biggl(1 + \frac{1 - q}{q} \cdot s^2 \biggr).
\end{align}
Clearly, $\phi_q$ is non-decreasing on the interval $[0, 1]$, with $\phi_q(0) = 0$ and $\phi_q(1) = \log(1/q)$. With this notation in hand, define the non-negative scalars $b_1(\boldsymbol{q},R), \ldots, b_d(\boldsymbol{q}, R)$ via the following optimization problems, where the vector $\boldsymbol{s} = (s_1, \ldots, s_d)$ collects the decision variables:
\begin{align} 
\label{eq:bk}
b_j(\boldsymbol{q}, R)
\,:= 
\sup_{\substack{\boldsymbol{s} \in [0,1]^d: \\ \sum_{k = 1}^d s_k \leq R/2}} \;  \frac{1}{q_j} \exp\biggl\{\,\sum_{k \neq j} \phi_{q_k}(s_k) \biggr\}, \qquad j \in [d].
\end{align}
Since $\sup_{s \in [0,1]} \phi_{q_j}(s) = \log(1/q_j)$, we have $b_j(\boldsymbol{q}, R) \leq (\prod_{k = 1}^d q_k)^{-1}$ for all $j \in [d]$, and the upper bound is achieved if we are able to set $\boldsymbol{s} = \mathbf{1}$. The value of $\boldsymbol{s}$ that saturates the upper bound is only admitted by the constraint $\sum_{k = 1}^d s_j \leq R/2$ if  $d \leq R/2$, i.e. in sufficiently low-dimensional problems. 
On the other hand, in higher dimensional problems (and depending on the value of $\boldsymbol{q}$) each scalar $b_j(\boldsymbol{q}, R)$ may be much smaller than the upper bound $(\prod_{k = 1}^d q_k)^{-1}$. In general, each $b_j(\boldsymbol{q}, R)$ has an intricate dependence on $\boldsymbol{q}$, obtained by solving the optimization problem~\eqref{eq:bk}.  Finally, it is convenient to define
\begin{equation}
\label{eq:V-qR}
V_{\boldsymbol{q},R} 
:= 
\sum_{j = 1}^d \mathbb{E}[X_j^2] \cdot b_j(\boldsymbol{q}, R).
\end{equation}

With this notation in hand, we are now ready to state our main theorem:
\begin{tcolorbox}
\begin{theorem}
\label{thm:general-sgd-rate}
For any initialization $\widehat{\theta}^{(0)} \in\Theta_R$, the output of \Cref{alg:Proj-SA} satisfies
\[
    \E\,\big\|\widehat\theta^{(n)} - \thetastar \big\|^2
    \le 
    \frac{(e^R+1)^2}{n\mu^2} \cdot V_{\boldsymbol{q},R}.
\]
\end{theorem}
\end{tcolorbox}
Since $\StochOp(\theta)$ is an unbiased estimate of the strongly monotone population operator $\PopOp(\theta)$, Theorem~\ref{thm:general-sgd-rate} is obtained by applying results from stochastic approximation~\citep[e.g.][]{NemirovskiJuditskyLanShapiro2009convexSA,Lan2020firstorder} in conjunction with a uniform bound on the variance of our stochastic estimate of the operator. Carefully controlling the variance (in \Cref{lem:A-hat-properties}) yields the delicate dependence on the terms $b_1(\boldsymbol{q}, R), \ldots, b_d(\boldsymbol{q}, R)$ in~\Cref{eq:bk}.

Two key consequences of this fine-grained result are worth emphasizing. First, the rate of estimation is always parametric and scales as $n^{-1}$, since the remaining multiplicative factors do not depend on the sample size $n$. Second, the behavior of the upper bound varies substantially across different regimes of the missingness parameters, since the scalars $b_1(\boldsymbol{q}, R), \ldots, b_d(\boldsymbol{q}, R)$ depend intricately on the observation probability vector $\boldsymbol{q}$. Nonetheless, one can extract a meaningful consequence from the aforementioned upper bound $\max_{j \in [d]} b_j(\boldsymbol{q}, R) \leq (\prod_{j = 1}^d q_j)^{-1}$, which is tight in low-dimensional problems. Applying this upper bound in conjunction with Theorem~\ref{thm:general-sgd-rate} and the assumed covariate bound $\| X \|_2 \leq 1$, we obtain
\begin{align} \label{eq:complete-case-compare}
    \E\,\big\|\widehat\theta^{(n)} - \thetastar \big\|^2
    \le
    \frac{(e^R+1)^2}{n\mu^2}\, \frac{1}{\prod_{j = 1}^d q_j}.
\end{align}
The dependence in the upper bound~\eqref{eq:complete-case-compare} on the observation probability vector $\boldsymbol{q}$ is comparable to that of complete-case logistic regression~\citep{vach2012logistic}.
Indeed, complete-case estimation retains only those observations for which every coordinate is observed, and so the number of samples retained is of the order $n \prod_{j = 1}^d q_j$. Its corresponding missingness penalty is therefore of the order $(\prod_{j = 1}^d q_j)^{-1}$. By  Eq.~\eqref{eq:complete-case-compare},  the missingness penalty of our estimator is never higher than that of the complete-case estimator, and the rates attained by both estimators are comparable in low-dimensional problems in which $b_j(\boldsymbol{q}, R) \asymp (\prod_{k = 1}^d q_k)^{-1}$ for all $j \in [d]$. However,
as we demonstrate shortly, in many cases the missingness penalty of our estimator can be significantly lower than that of its complete-case counterpart. 

To showcase these improvements over the complete-case estimator, it is helpful to specialize Theorem~\ref{thm:general-sgd-rate} to the special case of homogeneous missingness, where we have $q_1 = \cdots = q_d = q$. In this case, $b_1(\boldsymbol{q}, R) = \cdots = b_d(\boldsymbol{q}, R)$ due to symmetry, and the result admits further simplification.
Let $x_\star$ denote the unique positive solution to
\begin{equation}
\label{eq:def-h-func}
    h(x) := \frac{2x}{1+x} - \log(1+x) = 0,
\end{equation}
and define $q_\star := (1+x_\star)^{-1}$ and
$\kappa_\star := \sqrt{x_\star} (1+x_\star)^{-1} = q_\star\sqrt{x_\star}$.
Numerically, $x_\star \approx 3.92155$, $q_\star \approx 0.20319$, and $\kappa_\star \approx 0.40237$.
Let $\overline{\phi}_q: [0,\infty) \to [0, \log(1/q)]$ denote the least concave majorant of the function $\phi_q$ in~\Cref{eq:phi-function}, with the convention that $\overline{\phi}_q$ takes the constant value $\log(1/q)$ when its argument is greater than $1$. 
We establish in \Cref{claim:phiq-LCM} in \Cref{appsec:proof-upper} that 
\begin{equation} 
\label{eq:phiq-LCM-definition}
\overline\phi_q(s)
\,=\,
\begin{cases} 
  \, s\log\mfrac{1}{q}, &\text{if}\,~ q \in [q_\star,1] ~\text{and}~ s \in [0,1], \\ 
  \, 2\kappa_\star\sqrt{\mfrac{1-q}{q}}\,s, &\text{if}\,~ q \in (0,q_\star) ~\text{and}~ s\in[0,\tau_q],\\ 
  \, \log\Bigl(1+\mfrac{1-q}{q}s^2\Bigr), &\text{if}\,~ q \in (0,q_\star) ~\text{and}~ s\in[\tau_q,1] \\
  \, \log\mfrac{1}{q}, &\text{if}\,~ s > 1.
\end{cases}
\end{equation}
where $\tau_q := \sqrt{x_\star q (1-q)^{-1}}$ takes values in the range $(0,1)$ for all $q \in (0, q_\star)$.
The following corollary specializes Theorem~\ref{thm:general-sgd-rate} to the case of homogeneous missingness, and is proved in \Cref{proof:cor:sgd-rate}. 
\begin{tcolorbox}
\begin{corollary}
\label{cor:sgd-rate}
Suppose we have homogeneous missingness, with $q_1 = \cdots = q_d = q$. Then, for any initialization $\widehat{\theta}^{(0)} \in\Theta_R$, the output of \Cref{alg:Proj-SA} satisfies
\[
    \E\,\big\|\widehat\theta^{(n)} - \thetastar \big\|^2
    \le 
    \frac{(e^R+1)^2}{nq\mu^2} \, \exp\Biggl\{ (d - 1) \cdot \overline{\phi}_q \biggl( \frac{R}{2(d-1)} \biggr) \Biggr\}.
\]
\end{corollary}
\end{tcolorbox}
In the low-dimensional regime where $R > 2(d-1)$, Corollary~\ref{cor:sgd-rate} should be interpreted as recovering the complete-case rate alluded to above. In particular, since this regime corresponds to the final case in the RHS of~\Cref{eq:phiq-LCM-definition}, the upper bound simplifies to
\[
\E\,\big\|\widehat\theta^{(n)} - \thetastar \big\|^2
    \le
    \frac{(e^R+1)^2}{nq^d\mu^2},
\]
which is comparable to the complete-case rate.
However, in all other regimes, we obtain a strictly better rate than complete-case estimation.
Take for instance the regime $R \leq 2 \tau_q (d - 1)$ and $q \leq q_\star$, in which the problem is high-dimensional and we have significant missingness. This corresponds to the second case in the RHS of~\Cref{eq:phiq-LCM-definition}, and our bound simplifies to
\[
\E\,\big\|\widehat\theta^{(n)} - \thetastar \big\|^2
    \le
    \frac{(e^R+1)^2}{nq\mu^2} \cdot \exp\biggl\{C\frac{R}{\sqrt{q}}\biggr\}
\]
for some universal constant $C>0$. Since $R/\sqrt{q}\lesssim d-1$ in this regime, the resulting missingness penalty is much smaller than the complete-case penalty $q^{-d}$. Similar improvements over the complete-case rate are evident in other regimes of $(R, d, q)$.

To summarize, in Corollary~\ref{cor:sgd-rate} we obtain a result for homogeneous missingness with explicit constants. However, one disadvantage is that this result has several regimes depending on the joint values taken by $(R, d, q)$. Given that $\mu$ and $d$ are intrinsically related (see~\Cref{ineq:dimension-mu-relation}), it is useful to have a simplified characterization that depends only on the relative magnitudes of $R\mu$ and $q$.
\Cref{cor:sgd-rate-simplified} below groups the behavior of the upper bound into three such regimes. The proof is provided in \Cref{proof:cor:sgd-rate-simplified}.

\begin{tcolorbox}
\begin{corollary}
\label{cor:sgd-rate-simplified}
Under the conditions of \Cref{cor:sgd-rate}, there exists a universal constant $C>0$ such that
\begin{align}
\label{eq:sgd-rate-simplified}
    \E\,\bigr\|\widehat{\theta}^{(n)}-\thetastar\bigr\|^2
    \le
    \frac{4\,e^{2R}}{nq\mu^2} \cdot
    \begin{cases}
        \,\displaystyle \exp\left\{C\frac{R}{\sqrt{q} }\log(2-q)\right\}, & \text{if}\,~ R\mu\le\sqrt{q},\\ 
        \,\displaystyle \exp\left\{\frac{C}{\mu}\log\left(1+\frac{1-q}{q} \, R^2\mu^2\right)\right\}, & \text{if}\,~ \sqrt{q}<R\mu\le1,\\ 
        \,\displaystyle \exp\left\{\frac{C}{\mu}\log\frac{1}{q}\right\}, & \text{if}\,~ R\mu>1.
    \end{cases}
\end{align}
\end{corollary}
\end{tcolorbox}

When $q=1$, all three regimes collapse and \Cref{eq:sgd-rate-simplified} yields an upper bound on the order $e^{2R}/(n\mu^2)$. This is the well-known rate  for vanilla logistic regression under our assumptions\footnote{To obtain the claimed result for logistic regression, take $d= 1/\mu$ in~\citet[Proposition~10]{Bach2014SGDlogistic}. Also,~\citet{Bach2014SGDlogistic} bounds the squared parameter error in terms of the inverse square of the smallest eigenvalue of the population Hessian. In our setting, this smallest eigenvalue of the Jacobian of $\PopOp$ satisfies (see \Cref{eq:PopOp-hessian} in the proof of \Cref{lem:A-properties})
\[
\lambda_{\mathrm{min}} \bigl(\mcD\PopOp(\theta)\bigr) 
\ge
\frac{\mu}{2 \cosh(R/2)}
\asymp
\mu \, e^{-R/2},
\]
which leads to an exponential deterioration with $R$.}~\citep[e.g.][Proposition~10]{Bach2014SGDlogistic}. However, our estimator does not correspond to logistic regression even when observations are complete, and to our knowledge, this is the first such result for projected stochastic gradient descent on the exponential loss even without missingness.
Moreover, when $q = 1$, the rate presented in Corollary~\ref{cor:sgd-rate-simplified} is known to be minimax optimal under our assumptions by the Bayesian Cram\'er--Rao bounds of \citet{LeeCourtade2020linearGLM}.

The optimality of our estimator when $q=1$ is encouraging, but an interesting and important question is whether the three regimes in the upper bound~\eqref{eq:sgd-rate-simplified} are sharp in any fundamental sense when $q<1$. 
All three regimes in~\Cref{eq:sgd-rate-simplified}, for instance, exhibit a different dependence on $q$ than in linear regression problems, where it is known that the missingness penalty is on the order $q^{-2}$~\citep{wang2019rate,risebrow2026semi}.
The question of whether this nonstandard dependence on $q$ is fundamental in the nonlinear logistic regression problem forms the focus of our next section.

\subsection{A minimax lower bound}
\label{subsec:main-lower}

We now show a minimax lower bound that matches the rate in \Cref{cor:sgd-rate-simplified}, up to universal constants in the exponent. Throughout this subsection, we restrict attention to homogeneous missingness, and the curvature regime $\mu\in(0,\nicefrac{1}{3}]$. This curvature restriction excludes only the two lowest-dimensional cases, namely $d = 1,2$, since by \Cref{ineq:dimension-mu-relation}, $\mu>1/3$ only allows $d < 3$. Note that if $d=1,2$, then complete-case estimation is already effective~\citep{vach2012logistic}.

For any $P_X \in \mathcal{P}_{\mu}$, $\thetastar \in \mathbb{R}^d$, and $q\in(0,1]$, let $P^{(\thetastar)} := \mathsf{Law}(Y,M,Z) \in \mathcal{P}(\{\pm 1\} \times \{0,1\}^d \times \mathbb{R}^d)$ denote the corresponding observed-data distribution under the well-specified logistic model with~$\thetastar$ as the true regression parameter and the MCAR mechanism with homogeneous observation probability $q$. In particular, 
\begin{align}
\label{eq:Pthetastar}
X \sim P_X, \quad Z = X \odot M, \text{ where } M \sim \BER(q)^{\otimes d} \indep \{X, Y\}, \quad \text{and} \quad \mathbb{P}(Y = 1 \mid X) = \sigma(X^{\top} \thetastar).
\end{align}
Define the minimax risk
\begin{equation}
\label{def:minimax-risk}
\mathfrak{R}_n(R, \mu, q) 
:= 
\inf_{\widehat{\theta}} \,
\sup_{(P_X,\thetastar) \in (\mathcal{P}_{\mu},\Theta_R)} \; 
\mathbb{E}_{P^{(\thetastar)}}\Bigl[\bigl\| \widehat{\theta} - \thetastar\bigr \|_2^2\Bigr].
\end{equation}

Thus, $\mathfrak{R}_n(R,\mu,q)$ measures the worst-case minimax estimation risk over all covariate distributions and dimensions jointly permitted by the model class $\mathcal{P}_{\mu}$ and parameter class $\Theta_R$, in problems with homogeneous missingness and observation probability $q$.

\begin{tcolorbox}
\begin{theorem}
\label{thm:minimax-lower-bound}
Suppose we have homogeneous missingness, with $q_1 = \cdots = q_d = q$. There exist universal constants $c_0,c_1,C_0,C_1>0$ such that the following holds. Assume the sample-size condition
\begin{equation}
\label{eq:sample-size-cond-main-thm}
    n \ge
    \frac{C_0}{\mu^2} \cdot \max\Biggl\{
        \frac{\exp(2c_1 R)}{qR^2}\;\!,\;\!
        \exp\biggl(
            C_1 \min\biggl\{
                \frac{R}{\sqrt{q}},\frac{1}{\mu}
            \biggr\}
            \log\biggl(\frac{e}{q}\biggr)
        \biggr)
    \Biggr\}.
\end{equation}
Then, the minimax risk $\mathfrak{R}_n(R,\mu,q)$ over the class $(\mathcal{P}_{\mu},\Theta_R)$, as defined in \Cref{def:minimax-risk}, satisfies 
\[
    \mathfrak{R}_n(R,\mu,q)
    \ge
    \frac{c_0 \, e^{c_1 R}}{nq\mu^2} \cdot
    \begin{cases}
        \displaystyle
        \exp\left\{ c_1\frac{R}{\sqrt{q}} \log(2-q)\right\},
        & \text{if } R\mu\le\sqrt{q},
        \\ 
        \displaystyle
        \exp\left\{
            \frac{c_1}{\mu} \log\left(1+\frac{1-q}{q}R^2\mu^2\right)
        \right\},
        & \text{if } \sqrt{q}<R\mu\le1,
        \\ 
        \displaystyle
        \exp\left\{\frac{c_1}{\mu} \log\left(\frac{1}{q}\right)\right\},
        & \text{if } R\mu>1.
    \end{cases}
\]
\end{theorem}
\end{tcolorbox}

Comparing \Cref{thm:minimax-lower-bound} with \Cref{cor:sgd-rate-simplified}, we see that the upper and lower bounds match, up to universal constants, in their dependence on $n$, $q$, $R$, and $\mu$. In particular, both bounds exhibit the same three regimes according to the relative magnitudes of $R\mu$, $\sqrt{q}$ and $1$, with the same dependence on the problem parameters within each regime. Thus, up to universal constants in the exponents (and the  multiplicative factor $c_0$), \Cref{thm:minimax-lower-bound} establishes the minimax optimality of the rate achieved by our stochastic approximation procedure.

Given our goal of establishing that Corollary~\ref{cor:sgd-rate-simplified} is near-optimal, the sample-size requirement in the minimax lower bound \Cref{eq:sample-size-cond-main-thm} should not be viewed as a restrictive condition. Indeed, let $n_0(R,\mu,q)$ denote the right-hand side of \Cref{eq:sample-size-cond-main-thm}. 
If $n \lesssim n_0(R,\mu,q)$, then the upper bound in \Cref{cor:sgd-rate-simplified} is vacuous, and we can only infer that the error is bounded above on the order $R^2$. This upper bound is trivial given that this is the error attained by the all-zero estimator that ignores the data. Consequently, \Cref{thm:minimax-lower-bound} establishes the sharpness of Corollary~\ref{cor:sgd-rate-simplified} precisely in the regime in which the upper bound provides a non-trivial result.

It is worth commenting on the proof of \Cref{thm:minimax-lower-bound}, which is one of our main technical contributions. It combines two complementary Assouad constructions. 
The first captures the intrinsic difficulty already present in the fully observed problem: by placing a large common component in the regression parameter and then perturbing orthogonal directions, we construct neighboring logistic laws whose observed-data likelihoods become exponentially difficult to distinguish as $R$ grows. This portion of our proof can be viewed as parallel to existing Bayesian Cram\'er--Rao techniques~\citep{LeeCourtade2020linearGLM}.
On the other hand, the second construction is specific to missing covariates, and forms the main technical ingredient of the proof. We first present a two-point construction where we introduce a collection of coupled Rademacher covariates, and carefully design the coupling such that the observed-data likelihood ratio between the pair is exactly $1$ when too few of these covariates are observed. The two hypotheses are thus indistinguishable on \emph{all} such missingness patterns. Our coupling is designed using the coefficients of a Boolean function, whose conditional expectation remains equal to $1$ whenever fewer than half of the covariates are observed.
We then embed our two-point construction into a multidimensional Assouad hypercube, and optimizing the number and strength of the auxiliary coordinates yields the three-regime lower bound presented in \Cref{thm:minimax-lower-bound}.


\section{Proof of the upper bound}
\label{sec:proof-upper}

The proof of Theorem~\ref{thm:general-sgd-rate} proceeds by establishing several properties---detailed below---of the operator~$\PopOp$ and then invoking standard stochastic approximation convergence machinery.  

\subsection{Preliminaries}
The following lemma demonstrates that the population operator $\PopOp$ in \Cref{eq:A-theta} is well-behaved. Recall the definition of $\alpha_R$ from~\Cref{eq:def-eta-alpha}.

\begin{lemma}
    \label{lem:A-properties}
    Let $P_X \in \mathcal{P}_{\mu}$.
    Let $\thetastar \in \mathbb{R}^{\mathrm{dim}(X)}$ and suppose that the response variable $Y$ satisfies $\mathbb{P}(Y = 1 \mid X) = \sigma(X^{\top}\thetastar) = 1 - \mathbb{P}(Y = -1 \mid X)$.  Then,
    \begin{itemize}
        \item[(a)] $\PopOp(\thetastar) = 0$.
        \item[(b)] $\PopOp(\cdot)$ is $\alpha_R$--strongly monotone over $\Theta_R$, i.e.
        \[
    \ip{\PopOp(\theta)-\PopOp(\theta')}{\theta-\theta'}
    \ge
    \alpha_R\|\theta-\theta'\|^2
    \qquad\text{for all}\,~
    \theta, \theta' \in \Theta_R.
\]
    \end{itemize}
\end{lemma}
Lemma~\ref{lem:A-properties} is proved in Section~\ref{sec:lem-A-proof}.
Part $(a)$ of the above lemma reveals that $\thetastar$ is a zero of the operator $\PopOp$, and taken together, both parts of Lemma~\ref{lem:A-properties} imply that $\thetastar$ is the unique solution to the population variational inequality~\eqref{eq:population-VI}.

Since we only have access to missing data, we are unable to compute the population operator $\PopOp$ and therefore use the empirical surrogate $\StochOp$ as in \Cref{eq:A-hat}. The following lemma establishes that the empirical operator is a good approximation of its population-based counterpart. Recall the scalars $(b_j(\boldsymbol{q}, R))_{j \in [d]}$ from~\Cref{eq:bk} and $V_{\boldsymbol{q},R}$ from \Cref{eq:V-qR}.

\begin{lemma}
\label{lem:A-hat-properties}
    The empirical operator $\StochOp$ in \Cref{eq:A-hat} satisfies the following:
    \begin{itemize}
        \item[(a)] For every $\theta \in \Theta_R$, $\mathbb{E}\bigl[\StochOp(\theta)\bigr] = \PopOp(\theta)$.
        \item[(b)] For every $\theta \in \Theta_R$, 
        \[
        \mathbb{E}\bigl[\bigl\| \StochOp(\theta)\bigr\|^2\bigr] \leq \exp(R) \cdot V_{\boldsymbol{q},R}.
        \]
    \end{itemize}
\end{lemma}
Lemma~\ref{lem:A-hat-properties} is proved in Section~\ref{sec:lem-Ahat-proof}. Part $(a)$ of the above lemma establishes that the empirical operator $\StochOp$ is an unbiased estimator of the population operator $\PopOp$ and part $(b)$ shows that its second moment can be controlled at the level $V_{\boldsymbol{q}, R}$.  In the following section, we will use these two lemmas to prove Theorem~\ref{thm:general-sgd-rate}.

\subsection{Proof of Theorem~\ref{thm:general-sgd-rate}}
Recall the definition of $\widehat\theta^{(t+1)}$ from \Cref{eq:Proj-SA}. We have
\begin{align*}
     \bigl \| \widehat\theta^{(t+1)} - \thetastar \bigr \|^2
     &= \bigl \| \Pi_{\Theta_R}\bigl\{ \widehat\theta^{(t)} - \eta_t \StochOp(\widehat\theta^{(t)}; Y^{(t+1)}, M^{(t+1)}, Z^{(t+1)})\bigr\} - \Pi_{\Theta_R}(\thetastar) \bigr \|^2 \\
     &\leq \bigl \| \widehat\theta^{(t)} - \eta_t \StochOp(\widehat\theta^{(t)}; Y^{(t+1)}, M^{(t+1)}, Z^{(t+1)}) - \thetastar \bigr\|^2 \\
     &= \bigl \| \widehat\theta^{(t)} - \thetastar \bigr \|^2 - 2\eta_t \bigl\langle \widehat\theta^{(t)} - \thetastar, \StochOp(\widehat\theta^{(t)}; Y^{(t+1)}, M^{(t+1)}, Z^{(t+1)})\bigr\rangle \\ 
     &\hspace{6.5cm} + \eta_{t}^2 \bigl \| \StochOp(\widehat\theta^{(t)}; Y^{(t+1)}, M^{(t+1)}, Z^{(t+1)})\bigr\|^2,
\end{align*}
Next, let $\mathcal{F}_t := \sigma\bigl( \widehat\theta^{(0)}, \{Y^{(s)}, M^{(s)}, Z^{(s)}\}_{s=1}^t \bigr)$ denote the $\sigma$-algebra generated by the observations before the $t$th update, where $\mathcal{F}_0 := \sigma\bigl(\widehat\theta^{(0)}\bigr)$. 
Then
\begin{align}
\label{eq:error-expand}
    \E\Bigl[\bigl \| \widehat\theta^{(t+1)} - \thetastar \bigr \|^2 \Big| \mathcal{F}_t\Bigr] &\leq 
    \bigl \| \widehat\theta^{(t)} - \thetastar \bigr \|^2 - 2\eta_t \bigl\langle \widehat\theta^{(t)} - \thetastar, \mathbb{E}\bigl[\StochOp(\widehat\theta^{(t)}; Y^{(t+1)}, M^{(t+1)}, Z^{(t+1)}) \mid \mathcal{F}_t\bigr]\bigr\rangle \nonumber\\ 
    &\hspace{5cm} + \eta_{t}^2 \mathbb{E}\bigl[\bigl \| \StochOp(\widehat\theta^{(t)}; Y^{(t+1)}, M^{(t+1)}, Z^{(t+1)})\bigr\|^2 \mid \mathcal{F}_t\bigr] \nonumber\\
    &\overset{\1}{\leq} \bigl \| \widehat\theta^{(t)} - \thetastar \bigr \|^2 - 2\eta_t \bigl\langle \widehat\theta^{(t)} - \thetastar, \PopOp(\widehat\theta^{(t)})\bigr\rangle + \eta_{t}^2 \cdot e^R \cdot V_{\boldsymbol{q}, R} \nonumber \\
    &= \bigl \| \widehat\theta^{(t)} - \thetastar \bigr \|^2 - 2\eta_t \bigl\langle \widehat\theta^{(t)} - \thetastar, \PopOp(\widehat\theta^{(t)}) - \PopOp(\widehat\thetastar)\bigr\rangle + \eta_{t}^2 \cdot e^R \cdot V_{\boldsymbol{q}, R} \nonumber \\
    &\overset{\2}{\leq} (1 - 2 \alpha_R \eta_t) \, \bigl\|\widehat\theta^{(t)}-\thetastar\bigr\|^2 + \eta_t^2 \cdot e^R \cdot V_{\boldsymbol{q}, R},
\end{align}
where in $\1$ we have applied Lemma~\ref{lem:A-hat-properties}\emph{(a)} to simplify the inner product, and Lemma~\ref{lem:A-hat-properties}\emph{(b)} to upper bound the second moment term, while in $\2$ we have invoked the strong monotonicity of the operator~$\PopOp$. 
Marginalizing over $\mathcal{F}_t$ then yields the inequality
\[
\E\bigl \| \widehat\theta^{(t+1)} - \thetastar \bigr \|^2  
\le (1 - 2 \alpha_R \eta_t) \, \mathbb{E}\bigl\|\widehat\theta^{(t)}-\thetastar\bigr\|^2 + \eta_t^2 \cdot e^R \cdot V_{\boldsymbol{q}, R}.
\]
Thus, writing $\Delta_t := \mathbb{E}\|\widehat\theta^{(t)} - \thetastar\|^2$, substituting the step size choice $\eta_t = \{\alpha_R (t+2)\}^{-1}$ from \Cref{eq:def-eta-alpha} and simplifying, we obtain the recursive relation
\begin{align} 
\label{ineq:Deltat-ineq1}
\Delta_{t+1} \leq \frac{t}{t+2} \Delta_t + \frac{e^R}{\alpha_R^2 (t+2)^2} V_{\boldsymbol{q}, R}.
\end{align}
We claim that this recursive inequality yields 
\begin{align}
    \label{ineq:recursive-ineq-Deltat}
    \Delta_t\leq \frac{e^R}{t+1} \dfrac{V_{\boldsymbol{q}, R}}{\alpha_R^2},
\end{align}
and prove this by induction.

\noindent \underline{Base case: $t = 0$}.  Since 
$\widehat\theta^{(0)},\thetastar\in\Theta_R$, we have 
\begin{align*}
    \Delta_0 = \E\bigl\|\widehat\theta^{(0)}-\thetastar\bigr\|^2 \le 4 R^2 &\overset{\1}{\leq} \bigl(e^R + 1\bigr)^2 \cdot \frac{1}{\mu} \sum_{j = 1}^d b_j(\boldsymbol{q}, R) \\
    &\overset{\2}{\leq} \frac{(e^R+1)^2}{\mu^2} \,
    \Big\{ \sum_{j=1}^d \E[X_j^2] \cdot  b_j(\boldsymbol{q}, R) \Big\} \overset{\3}{=} \frac{e^R \cdot V_{\boldsymbol{q}, R}}{\alpha_R^2},  
\end{align*}
where $\1$ follows from the fact that by definition~\eqref{eq:bk}, we have $b_j \geq q_j^{-1} \geq 1$ and $\mu \leq 1/d \leq 1$, $\2$ follows from the fact $\mathbb{E}[X_j^2] \geq \lambda_{\min}\bigl(\E[XX^\top]\bigr) \geq \mu$, and $\3$ follows from the definition of $\alpha_R$ in \Cref{eq:def-eta-alpha} which leads to the identity
\begin{equation}
\label{eq:mu-alpha-identity}
\frac{(e^R + 1)^2}{\mu^2} 
= \frac{e^R (e^{R/2} + e^{-R/2})^2}{\mu^2}
= \frac{4e^R\cosh^2(R/2)}{\mu^2}
= \frac{e^R}{\alpha_R^2}.
\end{equation}

\noindent \underline{Inductive step}. Assume that the inequality~\eqref{ineq:recursive-ineq-Deltat} holds at iteration $t$.  Applying the recursive relation~\eqref{ineq:Deltat-ineq1}, we obtain
\begin{align*}
\Delta_{t+1} 
&\leq \frac{t}{t+2} \Delta_t + \frac{e^R}{\alpha_R^2 (t+2)^2} V_{\boldsymbol{q}, R} \leq \frac{e^R \cdot V_{\boldsymbol{q}, R}}{\alpha_R^2} \, \biggl[\frac{t}{(t+1)(t+2)} + \frac{1}{(t+2)^2}\biggr]
\,\leq\, \frac{1}{t+2} \cdot \frac{e^R \cdot V_{\boldsymbol{q}, R}}{\alpha_R^2}.
\end{align*}
This establishes Ineq.~\eqref{ineq:recursive-ineq-Deltat} by induction, and we conclude that
\[
    \E\bigl\|\widehat\theta^{(n)}-\thetastar\bigr\|^2
    \,=\,
    \Delta_n
    \le
    \frac{e^R}{n+1}
    \dfrac{V_{\boldsymbol{q}, R}}{\alpha_R^2}
    \overset{\1}{\le}
    \frac{(e^R+1)^2}{n\mu^2} \cdot V_{\boldsymbol{q},R},
\]
where $\1$ uses the same identity as in \Cref{eq:mu-alpha-identity}. 
\hfill\qed


\section{Proof of the lower bound}
\label{sec:proof-lower}

The proof of \Cref{thm:minimax-lower-bound} is based on two complementary lower bounds. The first captures the intrinsic difficulty associated with large parameter radius $R$, and yields an exponential dependence on $R$. The second is tailored to the missing covariates setting, and captures the additional difficulty induced by the MCAR mechanism, leading to the regime-dependent penalty in $q$, $R\mu$, and $\mu$. We establish these two bounds separately and then combine to obtain the claimed minimax lower bound.

The first construction is captured by the following proposition, which is proved in Section~\ref{subsec:baseline-assouad} below. 
\begin{tcolorbox}
\begin{proposition}
\label{prop:baseline-assouad}
There exists a universal constant $c>0$ such that, for every observation probability $q\in(0,1]$, radius $R>0$, and curvature parameter $0<\mu\le\nicefrac{1}{3}$, the minimax risk satisfies
\[
    \mathfrak{R}_n(R,\mu,q)
    \ge
    c\cdot\min\left\{
        R^2,\frac{e^{R/4}}{nq\mu^2}
    \right\}.
\]
\end{proposition}
\end{tcolorbox}

The second result leads to the three regimes that appear in the statement of \Cref{thm:minimax-lower-bound}; the proposition is proved in Section~\ref{subsec:large-R}.

\begin{tcolorbox}
\begin{proposition}
\label{prop:regime-dependent-lower}
There exist universal constants $c,C>0$ such that the following holds.
Let $q\in(0,1]$, $R>0$, and $0<\mu\le\nicefrac{1}{3}$. Suppose that
\[
    n \ge
    \frac{C}{\mu^2} \cdot \max\Biggl\{
        \frac{1}{qR^2}\;\!,\;\!
        \exp\biggl(
            C \min\biggl\{
                \frac{R}{\sqrt{q}},\frac{1}{\mu}
            \biggr\}
            \log\biggl(\frac{e}{q}\biggr)
        \biggr)
    \Biggr\}.
\]
Then, the minimax risk satisfies
\[
    \mathfrak{R}_n(R,\mu,q)
    \ge
    \frac{c}{nq\mu^2} \cdot
    \begin{cases}
        \displaystyle
        \exp\left\{ c\frac{R}{\sqrt{q}}\right\},
        & \text{if } R\mu\le\sqrt{q},\\ 
        \displaystyle
        \exp\left\{
            \frac{c}{\mu} \log\left(1+\frac{1-q}{q}R^2\mu^2\right)
        \right\},
        & \text{if } \sqrt{q}<R\mu\le1,\\ 
        \displaystyle
        \exp\left\{\frac{c}{\mu}\log\left(\frac{1}{q}\right)\right\},
        & \text{if } R\mu>1.
    \end{cases}
\]
\end{proposition}
\end{tcolorbox}

We now show that \Cref{prop:baseline-assouad,prop:regime-dependent-lower} together imply \Cref{thm:minimax-lower-bound}.

\paragraph{Proof of Theorem~\ref{thm:minimax-lower-bound}.}
First, choose the universal constants $c_1, C_0, C_1$ in the sample-size condition Eq.~\eqref{eq:sample-size-cond-main-thm} to satisfy 
\begin{equation}
\label{eq:sample-size-cond-master}
    n \ge
    \frac{C}{\mu^2} \cdot \max\Biggl[
        \frac{e^{R/8}}{qR^2}\;\!,\;\!
        \exp\biggl(
            C \min\biggl\{
                \frac{R}{\sqrt{q}},\frac{1}{\mu}
            \biggr\}
            \log\biggl(\frac{e}{q}\biggr)
        \biggr)
    \Biggr].
\end{equation}
Since $e^{R/8} \geq 1$, this implies that the lower bound on the sample size $n$ in Proposition~\ref{prop:regime-dependent-lower} is in force.  

Applying Proposition~\ref{prop:baseline-assouad}, we obtain the lower bound
\begin{equation}
\label{eq:baseline-lower-final}
    \mathfrak R_n(R,\mu,q) \ge c\cdot\min\left\{
        R^2,\frac{e^{R/4}}{nq\mu^2} 
    \right\} 
    \ge
    c \frac{e^{R/8}}{nq\mu^2},
\end{equation}
where the final inequality follows as the sample-size condition
in Eq.~\eqref{eq:sample-size-cond-master} implies $e^{R/8}/(nq\mu^2) \leq R^2$.

Continuing, we define the quantity
\[
    \zeta(q,R,\mu)
    :=
    \begin{cases}
        \displaystyle \frac{R}{\sqrt q}\log(2-q),
        & \text{if } R\mu\le\sqrt q,\\ 
        \displaystyle \frac{1}{\mu}\log\left(1+\frac{1-q}{q}R^2\mu^2\right),
        & \text{if } \sqrt q<R\mu\le1,\\ 
        \displaystyle \frac{1}{\mu}\log(1/q),
        & \text{if } R\mu>1,
    \end{cases}
\]
and apply Proposition~\ref{prop:regime-dependent-lower} in conjunction with the numeric inequality $\log(2 - q) \leq 1$ to obtain the lower bound 
\begin{equation}
\label{eq:regime-lower-final}
    \mathfrak R_n(R,\mu,q)
    \ge
    \frac{c}{nq\mu^2}\exp\{c\,\zeta(q,R,\mu)\}.
\end{equation}
Finally, combining \Cref{eq:baseline-lower-final,eq:regime-lower-final} and using $\max\{a,b\}\ge\sqrt{ab}$, we obtain
\[
    \mathfrak R_n(R,\mu,q)
    \ge
    \frac{c\,e^{R/16}}{nq\mu^2} \exp\left\{\frac{c}{2}\zeta(q,R,\mu)\right\}.
\]
This yields the desired result. \hfill \qed

\medskip\noindent 
It remains to prove \Cref{prop:baseline-assouad,prop:regime-dependent-lower}, which we do next. The proofs of some of the lemmas that we use in the proofs of these two propositions are deferred to \Cref{appsec:proof-lower}.

\subsection{Proof of \Cref{prop:baseline-assouad}}
\label{subsec:baseline-assouad}
The proof is based on a careful application of Assouad's lemma.  In particular, in this proof, we will let $\mathrm{dim}(X) = p+1$, where $p$ is defined as 
\begin{equation}
\label{eq:baseline-p}
    p:=\max\left\{1,\left\lfloor\frac{1}{4\mu}\right\rfloor\right\}.
\end{equation}
By construction,
\begin{equation}
\label{eq:baseline-p-bounds}
    \frac{1}{8\mu}\le p\le\frac{1}{\mu},
\end{equation}
so that $p \asymp 1/\mu$.  We then define a family of complete data distributions $P_{\omega} = \mathsf{Law}(X,Y) \in \mathcal{P}(\mathbb{R}^{p+1} \times \{-1,1\})$ and missing data distributions $Q_{\omega}  = \mathsf{Law}(Y, M, Z)$.

We define the complete data distributions $P_{\omega}$ as follows.  First, we define a common covariate distribution.  To this end, let $K \sim \operatorname{Unif}(1,\ldots,p)$, and let $B$ denote a Rademacher chosen independently of $K$. We then define the covariates
\begin{equation}
\label{eq:baseline-X}
    X:=\frac{1}{\sqrt{2}}e_1+\frac{B}{\sqrt{2}}e_{K+1},
\end{equation}
and take $P_X$ to denote its distribution.  Then, for each $\omega \in \{-1, 1\}^p$ we define the parameter
\begin{equation}
\label{eq:baseline-theta-omega}
    \theta_\omega:=\frac{R}{\sqrt{2}}e_1+2\sqrt{2}\,\alpha\sum_{k=1}^p\omega_k e_{k+1},
\end{equation}
where $\alpha>0$ denotes a scalar to be chosen.  Conditional on $X$, the response $Y$ follows the logistic model with underlying coefficient $\theta_{\omega}$, i.e. $\mathbb{P}(Y = 1 \mid X) = \sigma(X^{\top}\theta_{\omega}) = 1 - \mathbb{P}(Y = -1 \mid X)$, and taken together we have $P_{\omega} = P_X \times P_{Y \mid X}$.   Some straightforward calculations yield that for every $\omega\in\{-1,1\}^p$, we have
\begin{enumerate}[label=(\roman*)]
\item $\norm{X}_2=1 \quad \text{almost surely}$,

\item
$
    \E_{P_X}\bigl[XX^\top\bigr]
    =
    \operatorname{diag}\Bigl(\mfrac{1}{2},\mfrac{1}{2p}I_p\Bigr) 
    \succeq 
    \mu I_{p+1},
$

\item 
$\norm{\theta_\omega}_2^2 = \mfrac{R^2}{2} + 8p\alpha^2.$
\end{enumerate}

\smallskip\noindent
In particular, if $\alpha^2\le\nicefrac{R^2}{16p}$, then $\theta_\omega\in\Theta_R$.
Thus, for every $\omega\in\{-1,1\}^p$, we have $P_X \in \mathcal{P}_{\mu}$ (see \Cref{eq:def-P-mu}) and $\theta_{\omega} \in \Theta_R$ and so each instance lies within our class.

We next construct $Q_{\omega}$.  To this end, let $M=(M_1,\ldots,M_{p+1})$ have independent $\BER(q)$ coordinates and suppose that $M$ is independent of the fully observed data $(X,Y)$.  Let $Z := X \odot M$ and let $Q_\omega = \mathsf{Law}(Y, M, Z)$.  The following lemma controls the $\chi^2$ distance between $Q_{\omega}$ and $Q_{\omega'}$ when $\omega$ and $\omega'$ are neighbors.  For $k\in[p]$ and $\omega\in\{-1,1\}^p$, let $\omega^{(k)}$ denote the vertex obtained by flipping the sign of $\omega_k$, i.e., by replacing $\omega_k$ by $-\omega_k$.
\begin{tcolorbox}
\begin{lemma}
\label{lem:baseline-assouad-divergence}
Suppose that $0 < \alpha\le \nicefrac{R}{24}$. Then, for every $k\in[p]$ and every $\omega\in\{-1,1\}^p$,
\begin{equation}
\label{eq:baseline-edge-chi2-bound}
    \chi^2\bigl(Q_{\omega^{(k)}},Q_\omega\bigr)
    \le
    \frac{16q\alpha^2}{p}\exp\biggl(-\frac{R}{4}\biggr).
\end{equation}
Consequently,
\begin{equation}
\label{eq:baseline-edge-chi2-n}
    \chi^2\Bigl(Q_{\omega^{(k)}}^{\otimes n},Q_\omega^{\otimes n}\Bigr)
    \le
    \exp\biggl\{
        \frac{16nq\alpha^2}{p}\exp\biggl(-\frac{R}{4}\biggr)
    \biggr\}-1.
\end{equation}
\end{lemma}
\end{tcolorbox}
We provide the proof of Lemma~\ref{lem:baseline-assouad-divergence} in Section~\ref{proof:lem:baseline-assouad-divergence}.  The proof techniques analogous to those for the fully observed setting, but we need to condition on the probability of observing the $k$-th coordinate, from which the multiplicative factor of $q$ arises.  

Equipped with the $\chi^2$ bound from Lemma~\ref{lem:baseline-assouad-divergence}, we complete the proof of Proposition~\ref{prop:baseline-assouad}. Choose
\begin{equation}
\label{eq:baseline-alpha-choice}
    \alpha^2
    :=
    c'
    \min\left\{ \frac{R^2}{p},\,\frac{p\,\exp(R/4)}{nq} \right\},
\end{equation}
where $c' > 0$ is a universal constant chosen small enough such that $\alpha^2\le\nicefrac{R^2}{16p}$ and $\alpha\le\nicefrac{R}{24}$. 
In particular, for all $\omega \in \{-1,+1\}^p$, every $P_\omega$ belongs to the required model class, and invoking \Cref{lem:baseline-assouad-divergence} yields
\[
    \chi^2\left(Q_{\omega^{(k)}}^{\otimes n},Q_\omega^{\otimes n}\right)
    \le
    \exp(16c')-1.
\]
Thus, choosing $c'$ smaller than $\frac{1}{16} \log(2)$, we obtain that
\begin{equation}
\label{eq:baseline-edge-TV}
    \TV\left(Q_{\omega^{(k)}}^{\otimes n},Q_\omega^{\otimes n}\right)
    \le\frac{1}{2}
\end{equation}
for every $k$.  Since $\| \theta_{\omega^{(k)}} - \theta_{\omega} \|^2 \geq 8\alpha^2$, it follows from Assouad's lemma~\cite[e.g.][Lemma 8.7]{SamworthShah2026modern} that 
\[
    \mathfrak{R}_n(R,\mu,q)
    \ge
    2c' \cdot
    \min\biggl\{
        R^2,\,\frac{p^2\exp(R/4)}{nq}
    \biggr\} \overset{\1}{\geq} c \cdot 
    \min\biggl\{
        R^2,\,\frac{\exp(R/4)}{nq\mu^2}
    \biggr\}.
\]
Step $\1$ follows since by \Cref{eq:baseline-p-bounds}, we have $p\ge1/(8\mu)$.  \hfill $\qed$

\subsection{Proof of \Cref{prop:regime-dependent-lower}}
\label{subsec:large-R}

We now turn to the proof of \Cref{prop:regime-dependent-lower} by developing a second lower-bound construction that captures the additional statistical difficulty induced by missing covariates. The main idea is to construct well-specified logistic models whose regression parameters are separated, while the difference between their observed-data distributions is largely concealed by the missingness mechanism. The construction relies on a collection of auxiliary Rademacher covariates together with a carefully chosen Boolean function that ensures that hypotheses are indistinguishable whenever sufficiently few of these auxiliary covariates are observed.

Given that the proof ideas are delicate, we present the argument in three steps.
\begin{enumerate}[label={(\arabic*)}]
    \item 
    In \Cref{subsubsec:large-R-two-point}, we construct a pair of well-specified logistic models $P_0$ and $P_1$, indexed by an odd integer $m$ and parameters $\lambda,\alpha>0$. The parameter $\alpha$ controls the separation between the two regression coefficients, while $m$ controls the number of auxiliary covariates, and $\lambda$ determines how effectively the perturbation is hidden by missingness. We establish both validity of the construction and an upper bound on the $\chi^2$-divergence between the corresponding observed-data laws. This section conveys the key Boolean hiding construction underlying the lower bound.

    \item 
    In \Cref{subsubsec:large-R-assouad}, we carefully embed this two-point experiment into a $p$-dimensional hypercube $\{P_\omega:\omega\in\{0,1\}^p\}$. Along each edge of the hypercube, the parameter separation is inherited from the two-point construction, whereas the divergence between neighboring observed-data laws is reduced by an additional factor $1/p$. These two properties provide the ingredients for an Assouad argument, and yield a stronger lower bound in terms of $m,\lambda,\alpha$, and $p$.

    \item Finally, in \Cref{subsubsec:large-R-regimes}, we choose the free parameters  in our construction depending on the problem parameters $(R, \mu, q)$. Different choices of $m$ and $\lambda$ lead to the three regimes in the statement of the lower bound.
\end{enumerate}

\addtocontents{toc}{\protect\setcounter{tocdepth}{3}}
\subsubsection{A two-point construction}
\label{subsubsec:large-R-two-point}
\addtocontents{toc}{\protect\setcounter{tocdepth}{2}}

We begin with a two-point experiment that forms the basic building block of the lower bound. 
Unless otherwise stated, throughout this subsubsection let \(m:=2r+1\ge5\) be an odd integer that will be related to the covariate dimension. The definitions below also apply verbatim when \(m=1\).
Let $\lambda>0$, and define
\begin{equation}
\label{eq:def-rho}
    \rho:=\tanh(\lambda).
\end{equation}
The parameter $\lambda$ (and through it, $\rho$) will determine the strength of the baseline logistic signal, while a second parameter $\alpha>0$, introduced later, will determine the separation between the two hypotheses.

Eventually, our goal is to carefully construct two distribution-parameter pairs $(P_0,\theta_0)$ and $(P_1,\theta_1)$ that give rise to valid logistic models, but missingness in the corresponding observed laws $Q_0$ and $Q_1$ makes it hard to discern between $\theta_0$ and $\theta_1$. In particular, our two-point test will be given by the hypotheses
\begin{align*}
    H_0: &\quad (X,Y)\sim P_0,
    \qquad
    \thetastar=\theta_0,
    \\ 
    H_1: &\quad (X,Y)\sim P_1,
    \qquad
    \thetastar=\theta_1.
\end{align*}

\paragraph{Null distribution $(P_0)$.}
We define the null law $P_0$ through the following construction. Let $Y$ and~$B$ be independent uniform Rademacher random variables. Independently of $(Y,B)$, let $T_1,\ldots,T_m$ be independent biased (i.e. asymmetric) Rademacher random variables satisfying
\begin{equation}
\label{eq:def-T}
    \mathbb{P}(T_j=t)
    =
    \frac{\exp(\lambda t)}{2\cosh(\lambda)},
    \qquad\text{for}\,~ t\in\{-1,1\}.
\end{equation}
Thus, $\mathbb{E}_0[T_j]=\tanh(\lambda)=\rho$. Define
$A_j:=YT_j$ for $j\in[m]$, and define
\[
    X
    :=
    \left( \frac{A_1}{\sqrt{2m}},\ldots,\frac{A_m}{\sqrt{2m}},\frac{B}{\sqrt{2}} \right)^\top
    \in\mathbb{R}^{m+1}.
\]
Thus, $P_0$ is the joint distribution of $(X,Y)$, or equivalently of $(A,B,Y)$, induced by the preceding construction. We denote the corresponding marginal distribution of $X$ by $P_{0,X}$.

As we show in Lemma~\ref{lem:two-point-validity} to follow, choosing the data distribution to be $P_0$ is consistent with the following choice of null parameter:
\begin{equation}
\label{eq:two-point-theta0}
    \theta_0
    :=
    2\sqrt{2m}\,\lambda\,\big(\underbrace{1,\ldots,1}_{m ~\text{times}},0\big)^\top
    :=
    2\sqrt{2m}\,\lambda\, \big(\boldsymbol{1}_m^\top,0\big)^\top.
\end{equation}

The following notational convention will be useful going forward. For $a=(a_1,\ldots,a_m)\in\{-1,1\}^m$, let us denote the scaled sum by
\begin{equation}
\label{eq:s-a}
    s(a) := \lambda\sum_{j=1}^m a_j = \lambda \, a^\top\boldsymbol{1}_m.
\end{equation}
Then, for
$
    x(a,b)
    :=
    \Bigl(
        \mfrac{a_1}{\sqrt{2m}},\ldots,\mfrac{a_m}{\sqrt{2m}},\mfrac{b}{\sqrt{2}}
    \Bigr)^\top
$
and $\theta_0$ chosen as above,
we have
\begin{equation}
\label{eq:s-a-cond}
    \frac{1}{2} \bigl\langle x(a,b),\theta_0\bigr\rangle = s(a).
\end{equation}
For each $a \in \{-1, 1\}^m$ and $b,y \in \{-1, 1\}$, some calculations (presented in the proof of Lemma~\ref{lem:two-point-validity}) yield
\begin{align} \label{eq:label-distribution-null}
\mathbb{P}_0(Y = y \;|\; X = x(a, b)) = \frac{\exp\bigl\{y(s(a))\bigr\}}
    {2\cosh\bigl\{s(a)\bigr\}}
\end{align}
and
\begin{align} \label{eq:covariate-dist-null}
\mathbb{P}_0(X = x(a, b)) = \frac{\cosh\{s(a) \}}{2(2\cosh(\lambda))^m}.
\end{align}

\paragraph{The alternative distribution $(P_1)$.}
For $\alpha>0$, define the alternative parameter
\begin{equation}
\label{eq:two-point-theta1}
    \theta_1
    :=
    \theta_0+2\sqrt{2}\,\alpha e_{m+1}.
\end{equation}
Thus, following \Cref{eq:s-a-cond},
\[
    \frac{1}{2}\bigl\langle x(a,b),\theta_1\bigr\rangle
    =
    s(a)+\alpha b.
\]
We directly prescribe the conditional distribution of $Y$ under the alternative by
\begin{equation}
\label{eq:two-point-P1-conditional}
    \mathbb{P}_1(Y=y\mid X=x(a,b))
    :=
    \frac{\exp\bigl\{y(s(a)+\alpha b)\bigr\}}
    {2\cosh\bigl\{s(a)+\alpha b\bigr\}},
    \qquad
    y\in\{-1,1\}.
\end{equation}
To specify the marginal distribution of $X$, let $P_{0,X}$ denote the covariate marginal under $P_0$, and define $P_{1,X}$ through the following Radon--Nikodym derivative:
\begin{equation}
\label{eq:generic-two-point-P1X}
    \frac{\mathrm{d}P_{1,X}}{\mathrm{d}P_{0,X}}\bigl(x(a,b)\bigr)
    :=
    F_{\alpha}(a,b) \cdot
    \frac{\cosh\{s(a)+\alpha b\}}{\cosh\{s(a)\}},
\end{equation}
where $F_\alpha:\{-1,1\}^{m+1}\to\mathbb{R}$ is a function to be chosen appropriately below. Whenever the right-hand side is non-negative and has expectation 1 (with respect to $P_{0,X}$), \Cref{eq:generic-two-point-P1X} defines a valid covariate distribution. 

There are many possible choices of $F_\alpha$. We choose $F_{\alpha}$ so that, after the MCAR mechanism is applied, the resulting observed likelihood ratio is exactly equal to $1$ whenever at most $r := \frac{m-1}{2}$ of the first $m$ covariates are observed. The explicit choice of $F_{\alpha}$ that we make is\footnote{Another natural choice (see Remark~\ref{rem:proof-intuition} for intuition) is $F_{\alpha}(a, b) = \exp \big\{ J(a) \cdot b \big\}$ for a function $J$ satisfying~\Cref{eq:required-properties-J}. But we pursue the choice in~\Cref{eq:def-F-alpha} since it simplifies the algebra.}
\begin{equation}
\label{eq:def-F-alpha}
    F_\alpha(a,b) := \cosh(\alpha) - b\sinh(\alpha) \cdot J(a),
\end{equation}
where $J:\{-1,1\}^m\to\mathbb{R}$ is a function satisfying 
\begin{equation}
\label{eq:required-properties-J}
    J(-t)=-J(t),
    \quad\;
    \mathbb{E}_0\bigl[J(T)\bigr]=1,
    \quad\;
    \mathbb{E}_0\bigl[J(T) \,\big|\, T_S\bigr]=1
    \quad
    \text{for every }S\subseteq[m]\text{ with }|S|\leq r.
\end{equation}
It is not immediately clear that there exists a function $J$ satisfying all the properties mentioned in Eq.~\eqref{eq:required-properties-J}, but we construct such a function $J$ in the next paragraph.

\begin{remark} 
\label{rem:proof-intuition}
At this juncture, it is useful to pause to understand the intuition behind our construction through a heuristic calculation. 
For a complete observation $(Y,X)=(y,x(a,b))$, combining \Cref{eq:generic-two-point-P1X} with the conditional laws of $Y$ under $P_0$ and $P_1$ yields
\[
\frac{\mathrm{d}P_1}{\mathrm{d}P_0}(y,x(a,b))
=
e^{\alpha yb}F_\alpha(a,b)
=
e^{\alpha yb}\bigl\{\cosh(\alpha)-yb\sinh(\alpha) \cdot J(T)\bigr\}.
\]
If, heuristically, we could take $J(T)=1$, then
\[
\frac{\mathrm{d}P_1}{\mathrm{d}P_0}(y,x(a,b))
=
e^{\alpha yb}\bigl\{\cosh(\alpha)-yb\sinh(\alpha)\bigr\}
=
e^{\alpha yb}e^{-\alpha yb}
=
1.
\]
Thus, the factor $F_\alpha$ is designed to cancel exactly the change in the conditional law of $Y$ induced by perturbing $\theta_0$ to $\theta_1$. Of course, an odd function $J$ cannot satisfy $J(T) \equiv c$ uniformly for any constant $c \neq 0$. Instead, we construct $J$ so that $\E_0[J(T)\mid T_S]=1$ whenever $|S|\le r$, since this suffices to control the observed-data likelihood ratio $\mathrm{d}Q_1/\mathrm{d}Q_0$; see \Cref{lem:observed-likelihood-ratio-exact}.
\end{remark}

\paragraph{Construction of the function $J$.}
For all $U\subseteq[m]$, define
\begin{equation}
\label{eq:def-chiU-prod}
    \chi_U(t):=\prod_{j\in U}t_j,
    \qquad\text{for}\,~ t = (t_1,\ldots,t_m)\in\{-1,1\}^m,
\end{equation}
with $\chi_\varnothing\equiv1$. Let us consider the majority function defined by
\begin{equation*}
    f(t):=\operatorname{Maj}_m(t)
    :=\operatorname{sign}\Biggl\{\,\sum_{j=1}^m t_j\Biggr\}.
\end{equation*}
Since $f : \{-1,1\}^m \to \{-1,1\}$ is a Boolean function, it admits a unique Walsh--Fourier expansion \citep[Theorem 1.1]{odonnell2021analysisbooleanfunctions} under the uniform measure on $\{-1,1\}^m$, given by
\begin{equation}
\label{eq:def-gU}
    f(t) = \sum_{U\subseteq[m]} g_U \, \chi_U(t),
    \quad\text{where}\quad\!
    g_U := \frac{1}{2^m} \sum_{t\in\{-1,1\}^m} f(t) \, \chi_U(t).
\end{equation}
Finally, using the above coefficients, define
\begin{equation}
\label{eq:def-J}
    J(t)
    :=
    \sum_{\substack{U\subseteq[m]\\ |U|\ \mathrm{odd}}} g_U\rho^{-|U|}\chi_U(t),
    \qquad\text{for}\,~ t\in\{-1,1\}^m.
\end{equation}
Therefore, the alternative covariate distribution $P_{1,X}$ is defined through
\begin{equation}
\label{eq:two-point-P1X}
    \frac{\mathrm{d}P_{1,X}}{\mathrm{d}P_{0,X}}\bigl(x(a,b)\bigr)
    :=
    \bigl(\cosh(\alpha) - b\sinh(\alpha) \cdot J(a) \bigr) \,
    \frac{\cosh\{s(a)+\alpha b\}}{\cosh\{s(a)\}},
\end{equation}
where $J(\cdot)$ is defined in \Cref{eq:def-J}. 
The particular choice of the majority function ensures that $J$ has the properties~\eqref{eq:required-properties-J}. We show this in the following lemma, which also establishes that our constructions in both the null and alternative models lie in the model class. A proof can be found in \Cref{proof:lem:two-point-validity}.

\begin{tcolorbox}
\begin{lemma}
\label{lem:two-point-validity}
Let $m=2r+1\ge5$ and $\lambda>0$. Then the following hold.

\begin{enumerate}[label=(\alph*)]
    \item The function $J(\cdot)$ defined in \Cref{eq:def-J} satisfies all the properties in \Cref{eq:required-properties-J}, viz.,
    \begin{equation*}
        J(-t)=-J(t),
        \quad~
        \mathbb{E}_0\bigl[J(T)\bigr]=1,
        \quad\text{and}\quad
        \mathbb{E}_0\bigl[J(T) \,\big|\, T_S\bigr]=1
        \quad
        \forall\,S\subseteq[m]\text{ with }|S|\leq r.
    \end{equation*} 
    Moreover, $J$ satisfies the uniform bound:
    \begin{equation*}
        \|J\|_\infty
        \le
        \left( 2+\frac{1}{\sinh^2(\lambda)} \right)^{m/2}.
    \end{equation*}

    \item Suppose that $\alpha\|J\|_\infty\le\nicefrac{1}{2}$.
    Then, \Cref{eq:two-point-P1X} defines a valid covariate distribution $P_{1,X}$. Also, $P_0$ and $P_1$ are well-specified logistic regression models with parameters $\theta_0$ and $\theta_1$, respectively. Furthermore, the complete-data likelihood ratio satisfies
    \begin{equation}
    \label{eq:complete-data-likelihood-ratio}
        \frac{\mathrm{d}P_1}{\mathrm{d}P_0}(Y,A,B)
        =
        \exp(\alpha YB)
        \bigl\{\cosh(\alpha) - B\sinh(\alpha) \cdot J(A)\bigr\}.
    \end{equation}
    
    \item Suppose that $\alpha\|J\|_\infty\le\nicefrac{1}{2}$. Under both $P_0$ and $P_1$, $\|X\|_2=1$ almost surely, and
    \begin{equation}
        \mathbb{E}_0[XX^\top]
        =
        \mathbb{E}_1[XX^\top]
        =
        \frac{1}{2}
        \begin{bmatrix}
            \mfrac{1}{m}
            \left\{(1-\rho^2) I_m + \rho^2 \boldsymbol{1}_m \boldsymbol{1}_m^\top \right\} & 0\\ 
            0 & 1
        \end{bmatrix}
        \succeq
        \frac{\operatorname{sech}^2(\lambda)}{2m} I_{m+1}.
    \end{equation}
\end{enumerate}
\end{lemma}
\end{tcolorbox}

Also, the regression parameters $\theta_0$ and $\theta_1$ satisfy $\|\theta_0\|_2^2 = 8m^2\lambda^2$, $\|\theta_1\|_2^2 = 8m^2\lambda^2+8\alpha^2$, and more importantly,
    \begin{equation}
    \label{eq:two-point-separation}
        \|\theta_1-\theta_0\|_2^2 = 8\alpha^2.
    \end{equation}
We next quantify how much of this parameter separation remains visible under missingness. 

\paragraph{Missingness mechanism and observed-data likelihood ratio.}
Under both hypotheses, let the vector of missingness indicators $M=(M_1,\ldots,M_{m+1})$ have independent $\BER(q)$ coordinates, independently of $(X,Y)$. Write
\[
    Z := M\odot X.
\]

For $i\in\{0,1\}$, denote by $Q_i$ the observed-data law of $(Y,M,Z)$ when the complete-data distribution of $(Y,X)$ is $P_i$. Define
\begin{equation}
\label{eq:two-point-observed-masks}
    S:=\bigl\{j\in[m]:M_j=1\bigr\},
    \qquad
    D:=M_{m+1}.
\end{equation}
Thus, $S$ is the set of observed coordinates among the first $m$ coordinates, whereas $D$ indicates whether the last coordinate is observed or not. For notational convenience, we also define
\begin{equation}
\label{eq:def-V}
    V:=YB.
\end{equation}

In the following lemma, we derive the exact, observed-data likelihood ratio based on a single sample $(Y,M,Z)$ as a function of $S$ and $D$.

\begin{tcolorbox}
\begin{lemma}
\label{lem:observed-likelihood-ratio-exact}
Recall the definitions of the Boolean hiding function $J(\cdot)$ from \Cref{eq:def-J}, and the vector $T = (T_1,\ldots,T_m)$ of asymmetric Rademacher random variables from \Cref{eq:def-T}. 
For every deterministic $U \subseteq [m]$, denote
\[
    J_U(T_U) := \mathbb{E}_0\bigl[J(T) \bigm| T_U\bigr].
\]
The observed-data likelihood ratio based on a single observation $(Y,M,Z)$ satisfies
\[
    \widetilde{L}_\alpha := \frac{\mathrm{d}Q_1}{\mathrm{d}Q_0}(Y,M,Z)
    =
    \begin{cases}
        \displaystyle
        \exp(\alpha V)
        \left\{
            \cosh(\alpha) - V \sinh(\alpha) \cdot J_S(T_S)
        \right\},
        & \text{if}~\, D=1,
        \\
        \displaystyle
        \cosh^2(\alpha) - \sinh^2(\alpha) \cdot J_S(T_S),
        & \text{if}~\, D=0.
    \end{cases}
\]
In particular, $\widetilde{L}_\alpha=1$ on the event $\{|S|\leq r\}$.
\end{lemma}
\end{tcolorbox}

The key takeaway from \Cref{lem:observed-likelihood-ratio-exact} is the fact that the observed-data likelihood ratio is identically equal to $1$ whenever $|S|\le r$ (see \Cref{rem:proof-intuition}).
Thus, conditional on \emph{any} observed missingness pattern that reveals at most $r=\frac{m-1}{2}$ of the first $m$ auxiliary coordinates, the two hypotheses are completely indistinguishable. We crucially use the Boolean cancellation property of $J(\cdot)$ as in \Cref{eq:required-properties-J} in the proof of \Cref{lem:observed-likelihood-ratio-exact}, which can be found in \Cref{proof:lem:observed-likelihood-ratio-exact}.

Next, for $t\ge0$, define the function $\beta : [0,\infty) \to [0,\infty)$ such that
\begin{equation}
\label{eq:beta-function}
    \beta(t)
    :=
    \frac{t+\sqrt{t^2+4t}}{2},
\end{equation}
and let
\begin{equation}
\label{eq:def-Bmlq}
    \mathcal{B}_{m,\lambda,q}
    :=
    \biggl[
        \beta\biggl(\frac{q}{\sinh^2(\lambda)}\biggr)
    \biggr]^m.
\end{equation}
The next lemma gives a bound on the $\chi^2$--divergence between the observed-data distributions $Q_0$ and $Q_1$, provided $\alpha$ is chosen to be small enough.

\begin{tcolorbox}
\begin{lemma}
\label{lem:two-point-divergence}
Let \(m=2r+1\ge1\) be an odd integer. Suppose that
\[
    \alpha\|J\|_\infty\le\frac{1}{2},
    \qquad\text{and}\qquad
    0<\alpha^2\le q.
\]
Then, there exists a universal constant $C>0$ such that
\begin{equation}
\label{eq:two-point-observed-chi2}
    \chi^2(Q_1,Q_0)
    \le
    Cq\alpha^2\,\mathcal{B}_{m,\lambda,q}.
\end{equation}
Consequently, based on $n$ independent observations,
\begin{equation}
\label{eq:two-point-observed-chi2-n}
    \chi^2\left(Q_1^{\otimes n},Q_0^{\otimes n}\right)
    \le
    \exp\left\{
        Cnq\alpha^2\,\mathcal{B}_{m,\lambda,q}
    \right\}-1.
\end{equation}
\end{lemma}
\end{tcolorbox}
The proofs of \Cref{lem:two-point-validity,lem:two-point-divergence} are deferred to \Cref{proof:lem:two-point-validity,proof:lem:two-point-divergence}, respectively. The two properties that will be carried into the multidimensional construction are the separation bound in \Cref{eq:two-point-separation} and the observed-data $\chi^2$--divergence bound in \Cref{eq:two-point-observed-chi2}.

\addtocontents{toc}{\protect\setcounter{tocdepth}{3}}
\subsubsection{A multidimensional Assouad extension}
\label{subsubsec:large-R-assouad}
\addtocontents{toc}{\protect\setcounter{tocdepth}{2}}

The two-point construction in \Cref{subsubsec:large-R-two-point} produces a pair of well-specified logistic models whose parameters are separated in $\ell_2^2$ by $8\alpha^2$, while the corresponding observed-data distributions remain close under missingness (see \Cref{eq:two-point-observed-chi2}). However, given that the construction is on two points, the perturbation acts in only a single parameter direction. To obtain the desired dependence on the curvature parameter $\mu$, we now distribute the same two-point experiment across several directions and apply Assouad's method. At a high level, this proof follows the structure of the proof of Proposition~\ref{prop:baseline-assouad} along with the properties of the two-point construction established above.

Let $p\ge1$ be an integer that will be chosen later, and let
\[
    K\sim\operatorname{Unif}(1,\ldots,p).
\]
For each
$
    \omega=(\omega_1,\ldots,\omega_p)\in\{0,1\}^p,
$
conditional on $K=k$, generate $(A,B,Y)$ according to the two-point model $P_{\omega_k}$ constructed in \Cref{subsubsec:large-R-two-point}. Thus, conditional on $K=k$, the observation is generated from $P_0$ when $\omega_k=0$, and from $P_1$ when $\omega_k=1$.

We embed the special coordinate $B$ into the $k$-th of $p$ possible directions. More precisely, define the $(m+p)$-dimensional covariate
\begin{equation}
\label{eq:large-R-assouad-X}
    X
    :=
    \frac{1}{\sqrt{2m}}
    \sum_{j=1}^m A_j e_j
    +
    \frac{B}{\sqrt{2}}e_{m+K},
\end{equation}
and, for each $\omega\in\{0,1\}^p$, define
\begin{equation}
\label{eq:large-R-assouad-theta}
    \theta_\omega
    :=
    2\sqrt{2m}\,\lambda
    \sum_{j=1}^m e_j
    +
    2\sqrt{2}\,\alpha
    \sum_{k=1}^p\omega_k e_{m+k}.
\end{equation}

Let $P_\omega$ denote the resulting joint distribution of $(X,Y)$.
Conditional on the event $K=k$, only the coordinate $e_{m+k}$ is active among the last $p$ coordinates. Therefore,
\begin{equation}
\label{eq:large-R-assouad-logit}
    \frac{1}{2}X^\top\theta_\omega
    =
    \lambda\sum_{j=1}^m A_j+\alpha\omega_k B
    =
    s(A)+\alpha\omega_k B.
\end{equation}
Consequently, conditional on $K=k$, the model $P_\omega$ coincides exactly with the two-point model $P_{\omega_k}$. In this sense, the $k$-th coordinate of $\omega$ determines which member of the two-point experiment is placed along the $k$-th Assouad direction.

The following lemma verifies that the resulting hypercube consists of valid logistic regression models and computes the separation between neighboring vertices.

\begin{tcolorbox}
\begin{lemma}
\label{lem:large-R-assouad-validity}
Let $m=2r+1\ge5$, $\lambda>0$, and $p\ge1$. Suppose that
\begin{equation}
\label{eq:large-R-assouad-validity-conditions}
    m\lambda
    \le
    \frac{R}{4\sqrt{2}},
    \qquad
    \alpha\|J\|_\infty
    \le
    \frac{1}{2},
    \qquad
    p\alpha^2
    \le
    \frac{3R^2}{32}.
\end{equation}
Then, the following hold.
\begin{enumerate}[label=(\alph*)]
    \item For every $\omega\in\{0,1\}^p$, $P_\omega$ is a well-specified logistic model with parameter $\theta_\omega\in\Theta_R$.

    \medskip
    \item Under all $P_{\omega}$, $\|X\|_2=1$ almost surely, for every $\omega\in\{0,1\}^p$. Moreover,
    \begin{equation}
    \label{eq:large-R-assouad-second-moment-lower}
        \mathbb{E}_\omega[XX^\top]
        =
        \frac{1}{2}
        \begin{bmatrix}
            \mfrac{1}{m}
            \left\{(1-\rho^2)I_m + \rho^2 \boldsymbol{1}_m \boldsymbol{1}_m^\top\right\} & 0 \\ 
            0 & \mfrac{1}{p}I_p
        \end{bmatrix}
        \succeq
        \frac{1}{2}
        \min\left\{
            \frac{\operatorname{sech}^2(\lambda)}{m},
            \frac{1}{p}
        \right\}
        I_{m+p}.
    \end{equation}

    \item Fix any $k\in[p]$ such that $\omega\in\{0,1\}^p$ satisfies $\omega_k=0$. Let $\omega^{(k)}$ be the neighboring vertex of $\omega$ obtained by replacing $\omega_k$ by $1$. Then,
    \begin{equation}
    \label{eq:large-R-assouad-edge-separation}
        \|\theta_{\omega^{(k)}}-\theta_\omega\|_2^2
        =
        8\alpha^2.
    \end{equation}
\end{enumerate}
\end{lemma}
\end{tcolorbox}

In particular, \Cref{eq:large-R-assouad-second-moment-lower} shows that the curvature condition
$
    \mathbb{E}_\omega\bigl[XX^\top\bigr]\succeq\mu I_{m+p},
$
holds whenever
\begin{equation}
\label{eq:large-R-assouad-curvature-condition}
    \mu
    \le
    \frac{1}{2}
    \min\left\{
        \frac{\operatorname{sech}^2(\lambda)}{m}, \frac{1}{p}
    \right\}.
\end{equation}

We next turn to controlling the distance between observed-data distributions. Under every $P_\omega$, let $M\in\{0,1\}^{m+p}$ have independent $\BER(q)$ coordinates, independently of $(X,Y)$, and let $Z:=M\odot X$. Denote the corresponding observed-data distribution of $(Y,M,Z)$ by $Q_\omega$.
Consider two neighboring vertices $\omega$ and $\omega^{(k)}$. Conditional on $K=\ell\neq k$, the two models coincide, since $\omega_\ell=\omega^{(k)}_\ell$. Conditional on $K=k$, they reduce exactly to the two-point pair $(P_0,P_1)$ from \Cref{subsubsec:large-R-two-point}. Since $\mathbb{P}(K=k)=1/p$, the distinguishability via $\chi^2$--divergence between the laws of neighboring vertices is therefore reduced by a factor $1/p$.

\begin{tcolorbox}
\begin{lemma}
\label{lem:large-R-assouad-hamming}
Suppose that
\[
    \alpha\|J\|_\infty\le\frac{1}{2},
    \qquad\text{and}\qquad
    0<\alpha^2\le q.
\]
Recall the definition of $\mathcal{B}_{m,\lambda,q}$ in \Cref{eq:def-Bmlq}. Fix any $k\in[p]$ such that $\omega\in\{0,1\}^p$ satisfies $\omega_k=0$. Let $\omega^{(k)}$ be obtained by replacing $\omega_k$ by $1$. Then, there exists a universal constant $C>0$ such that
\begin{equation}
\label{eq:large-R-assouad-edge-chi2}
    \chi^2\left(
        Q_{\omega^{(k)}},
        Q_\omega
    \right)
    \le
    \frac{Cq\alpha^2}{p}
    \mathcal{B}_{m,\lambda,q}.
\end{equation}
Consequently, based on $n$ independent observations,
\begin{equation}
\label{eq:large-R-assouad-edge-chi2-n}
    \chi^2\left(
        Q_{\omega^{(k)}}^{\otimes n},
        Q_\omega^{\otimes n}
    \right)
    \le
    \exp\left\{
        \frac{Cnq\alpha^2}{p}
        \mathcal{B}_{m,\lambda,q}
    \right\}-1.
\end{equation}
\end{lemma}
\end{tcolorbox}

The proofs of \Cref{lem:large-R-assouad-validity,lem:large-R-assouad-hamming} are deferred to the appendix. To summarize their takeaways, neighboring parameters are separated by $8\alpha^2$ according to \Cref{eq:large-R-assouad-edge-separation}, while the corresponding observed-data divergence is bounded by \Cref{eq:large-R-assouad-edge-chi2}. 

It will also be convenient to define the shorthand
\begin{equation}
\label{eq:def-G-lambda}
    G_\lambda
    :=
    2+\frac{1}{\sinh^2(\lambda)}.
\end{equation}
By \Cref{lem:two-point-validity}, the Boolean hiding function satisfies $\|J\|_\infty^2\le (G_\lambda)^m$.
Thus, while $\mathcal{B}_{m,\lambda,q}$ (\Cref{eq:def-Bmlq}) determines how effectively missingness conceals the perturbation, $(G_\lambda)^m$ controls how small the perturbation size $\alpha$ must be for the construction to remain valid.

The following lemma summarizes the resulting Assouad lower bound. Given that the lower bound on the risk will be used repeatedly thereafter, we prove the lemma immediately after stating it.

\begin{tcolorbox}
\begin{lemma}
\label{lem:generic-assouad-reduction}
Let $m\ge5$ be odd, let $\lambda>0$, and let $p\ge1$ be an integer satisfying
\begin{equation}
\label{eq:generic-assouad-deterministic-conditions}
    m\lambda\le\frac{R}{4\sqrt{2}},
    \quad\,\text{and}\quad\,
    \mu \le
    \frac{1}{2} \min\left\{\frac{\operatorname{sech}^2(\lambda)}{m},\frac{1}{p}\right\}.
\end{equation}
Suppose further that, for a sufficiently large universal constant $C>0$,
\begin{equation}
\label{eq:generic-assouad-sample-size}
    n
    \ge
    \frac{C}{\mathcal{B}_{m,\lambda,q}}
    \max\left\{
        \frac{p}{q^2},\frac{p}{q}\, \bigl(G_\lambda\bigr)^m, \frac{p^2}{qR^2}
    \right\}.
\end{equation}
Then, there exists a universal constant $c>0$ such that
\begin{equation}
\label{eq:generic-assouad-lower-bound}
    \mathfrak{R}_n(R,\mu,q)
    \ge
    c\,
    \frac{p^2}
    {nq\,\mathcal{B}_{m,\lambda,q}}.
\end{equation}
\end{lemma}
\end{tcolorbox}

\begin{proof}[Proof of \Cref{lem:generic-assouad-reduction}]
Choose
\begin{equation}
\label{eq:generic-assouad-alpha-choice}
    \alpha^2
    :=
    \widetilde{c} \frac{p}{nq\,\mathcal{B}_{m,\lambda,q}},
\end{equation}
where $\widetilde{c}>0$ is a universal constant defined as 
\[
\widetilde{c} := \min\biggl\{\frac{3C}{32},\frac{1}{C} \log(2) \biggr\},
\]
where $C$ is the constant in~\Cref{eq:generic-assouad-sample-size}.
We first verify that the resulting hypercube satisfies the assumptions of \Cref{lem:large-R-assouad-validity,lem:large-R-assouad-hamming}. By \Cref{eq:generic-assouad-sample-size}, we have
\[
    n \ge \frac{C}{\mathcal{B}_{m,\lambda,q}} \frac{p}{q^2}
    \quad\Longrightarrow\quad
    \alpha^2 \le \frac{\widetilde{c}}{C}q \le q,
\]
since, by definition, $\widetilde{c} \le 3C/32 < C$. Similarly, since $\|J\|_\infty^2\le G_\lambda^m$ by \Cref{lem:two-point-validity},
\[
    n \ge \frac{C}{\mathcal{B}_{m,\lambda,q}} \frac{p}{q} \bigl(G_\lambda\bigr)^m
    \quad\Longrightarrow\quad
    \alpha^2\|J\|_\infty^2 \le \frac{\widetilde{c}}{C} \le \frac{1}{4},
\]
since, by definition, $\widetilde{c} \le 3C/32 < C/4$. Thus, $\alpha\|J\|_\infty\le\nicefrac{1}{2}$. Finally,
\[
    n \ge \frac{C}{\mathcal{B}_{m,\lambda,q}} \frac{p^2}{qR^2}
    \quad\Longrightarrow\quad
    p\alpha^2 \le \frac{\widetilde{c}}{C}R^2 \le \frac{3R^2}{32}.
\] 
Therefore, all conditions of \Cref{lem:large-R-assouad-validity,lem:large-R-assouad-hamming} hold. 
Together with \Cref{eq:generic-assouad-deterministic-conditions}, \Cref{lem:large-R-assouad-validity} therefore ensures that every $P_\omega$ belongs to the required model class.
Hence, for every $k\in[p]$ and every neighboring pair $\bigl(\omega,\omega^{(k)}\bigr)$, \Cref{lem:large-R-assouad-hamming} and \Cref{eq:generic-assouad-alpha-choice} together yield
\[
    \chi^2\left(
        Q_{\omega^{(k)}}^{\otimes n},
        Q_\omega^{\otimes n}
    \right)
    \le
    \exp\left\{
        \frac{Cnq\alpha^2}{p}\,
        \mathcal{B}_{m,\lambda,q}
    \right\}-1
    \overset{\1}{=}
    \exp(\widetilde{c}\,C)-1
    \overset{\2}{\le}
    1,
\]
where step $\1$ follows directly from the definition of $\alpha^2$ in \Cref{eq:generic-assouad-alpha-choice}, and step $\2$ exploits the fact that $\widetilde{c} \le \frac{1}{C} \log(2)$ which implies $\exp(\widetilde{c}\,C) \le 2$. Lastly, the inequality $\TV(P,Q) \le \frac{1}{2}\sqrt{\chi^2(P,Q)}$ yields
\begin{equation}
\label{eq:generic-assouad-TV}
    \TV\left(
        Q_{\omega^{(k)}}^{\otimes n},
        Q_\omega^{\otimes n}
    \right)
    \le
    \frac{1}{2}
\end{equation}
uniformly over all neighboring vertices. 

To complete the proof, we employ a standard Assouad reduction (see, e.g., \citet[Lemma 2.12]{Tsybakov2009}, \citet[Lemma 8.7]{SamworthShah2026modern}). From \Cref{lem:large-R-assouad-validity}(c), we have separation $8\alpha^2$ in $\ell^2_2$ between two neighboring vertices $\omega$ and $\omega^{(k)}$, and from \Cref{eq:generic-assouad-TV}, we have that their corresponding observed-data laws are within $1/2$ in total variation distance. Thus, by a calculation similar to the proof of our baseline Assouad minimax risk lower bound in \Cref{subsec:baseline-assouad}, we obtain
\[
    \mathfrak{R}_n(R,\mu,q)
    \ge
    \frac{p\alpha^2}{2}
    =
    \frac{\widetilde{c}}{2}
    \frac{p^2}
    {nq\,\mathcal{B}_{m,\lambda,q}},
\]
which proves the lemma after renaming $c := \widetilde{c}/2$.
\end{proof}

In \Cref{subsubsec:large-R-regimes}, we choose $p,m,\lambda$, and $\alpha$ so that these bounds yield the desired rates in the three regimes.

\addtocontents{toc}{\protect\setcounter{tocdepth}{3}}
\subsubsection{Choice of parameters in the three regimes}
\label{subsubsec:large-R-regimes}
\addtocontents{toc}{\protect\setcounter{tocdepth}{2}}

We now choose the parameters in the multidimensional construction to obtain the three regimes in \Cref{prop:regime-dependent-lower}, by specializing \Cref{lem:generic-assouad-reduction} accordingly. 
These choices are made in the following claims:

\begin{tcolorbox}
\begin{claim}
\label{claim:lower-bound-regime-I}
Suppose that $R\mu\le\sqrt{q}$. There exist universal constants $c,C_0,C_1,C_2>0$ such that, whenever
\[
    n \ge
    \max\biggl\{
        \frac{C_0}{q\mu^2R^2},
        \frac{C_1}{\mu^2} \exp\Bigl( C_2\frac{R}{\sqrt{q}} \log\left(e/q\right) \Bigr)
    \biggr\},
\]
we have
\[
    \mathfrak{R}_n(R,\mu,q)
    \ge
    \frac{c}{nq\mu^2}
    \exp\left\{
        c\frac{R}{\sqrt{q}}
    \right\}.
\]
\end{claim}
\end{tcolorbox}

\begin{tcolorbox}
\begin{claim}
\label{claim:lower-bound-regimes-II-III}
Suppose that $R\mu>\sqrt{q}$. There exist universal constants $c_0,c_1,C_0,C_1>0$ such that, whenever
\begin{equation}
\label{eq:regimes-II-III-sample-size}
    n \ge
    \frac{C_0}{\mu^2}
    \exp\biggl\{ \frac{C_1}{\mu} \log\biggl(\frac{e}{q}\biggr) \biggr\},
\end{equation}
we have
\begin{equation}
\label{eq:lower-bound-regimes-II-III}
    \mathfrak{R}_n(R,\mu,q)
    \ge
    \frac{c_0}{nq\mu^2}
    \exp\biggl\{
        \frac{c_1}{\mu} \log\biggl( 1+\frac{1-q}{q} \min\{R^2\mu^2,1\} \biggr)
    \biggr\}.
\end{equation}
\end{claim}
\end{tcolorbox}

Taking these claims as given---their proofs are presented at the end of this section---we now complete the proof of Proposition~\ref{prop:regime-dependent-lower}. We treat the three regimes on $R\mu$, viz. $R\mu \leq \sqrt{q}$, $R\mu \in (\sqrt{q},1]$, and $R\mu > 1$, separately.

\begin{enumerate}[label=(\Roman*)]
    \item Suppose first that $R\mu\le\sqrt q$. Then, $R/\sqrt{q} \le 1/\mu$. Thus,
    \[
        \min\left\{\frac{R}{\sqrt q},\frac{1}{\mu}\right\}
        =
        \frac{R}{\sqrt q}.
    \]
    Hence, after enlarging the universal constant $C>0$ if necessary, the sample-size condition in \Cref{prop:regime-dependent-lower} implies the assumptions of \Cref{claim:lower-bound-regime-I}. Therefore,
    \[
        \mathfrak R_n(R,\mu,q)
        \ge
        \frac{c}{nq\mu^2}
        \exp\left\{ c\frac{R}{\sqrt q} \right\}.
    \]

    \item Now suppose that $\sqrt q<R\mu\le1$. Then, $R/\sqrt q > 1/\mu$, so
    \[
        \min\left\{\frac{R}{\sqrt q},\frac{1}{\mu}\right\}
        =
        \frac{1}{\mu}.
    \]
    Thus, the sample-size condition in \Cref{prop:regime-dependent-lower} implies that of \Cref{claim:lower-bound-regimes-II-III}. Since $\min\{R^2\mu^2,1\}=R^2\mu^2$, we obtain
    \[
        \mathfrak R_n(R,\mu,q)
        \ge
        \frac{c}{nq\mu^2}
        \exp\left\{
            \frac{c}{\mu}
            \log\left(1+\frac{1-q}{q}R^2\mu^2\right)
        \right\}.
    \]

    \item Finally, suppose that $R\mu>1$. Again $R/\sqrt q > R > 1/\mu$, so the sample-size condition in \Cref{prop:regime-dependent-lower} implies that of \Cref{claim:lower-bound-regimes-II-III}. Now $\min\{R^2\mu^2,1\}=1$, and hence
    \[
        \log\left(1+\frac{1-q}{q}\right)
        =
        \log\left(\frac{1}{q}\right).
    \]
    Therefore,
    \[
        \mathfrak R_n(R,\mu,q)
        \ge
        \frac{c}{nq\mu^2}
        \exp\left\{
            \frac{c}{\mu} \log\left(\frac{1}{q}\right)
        \right\}.
    \]
\end{enumerate}

\noindent It remains to prove the claims.

\paragraph{Proofs of \Cref{claim:lower-bound-regime-I,claim:lower-bound-regimes-II-III}.} To begin, fix universal constants $c',c''>0$ such that
\begin{equation}
\label{eq:choice-a-b}
    c'\le\min\left\{\frac{\operatorname{sech}^2(2)}{2},\frac{1}{8\sqrt{2}}\right\} < \frac{1}{25},
    \quad\,\text{and}\quad\,
    c''\le\frac{1}{2}.
\end{equation}

We also fix a sufficiently small universal constant
\begin{equation}
\label{eq:mu-critical}
    \mu_{\mathrm{crit}}\le\min\left\{\frac{c'}{6},\frac{c''}{2}\right\} < \frac{1}{100},
\end{equation}
and write $L_q:=\log(e/q) = 1 + \log(1/q) \ge 1$.

\begin{proof}[Proof of \Cref{claim:lower-bound-regime-I}]
We prove \Cref{claim:lower-bound-regime-I} by separating the cases where $R/\sqrt{q}$ is bounded, and where $R/\sqrt{q}$ is unbounded. Set $u:=R/\sqrt{q}$, so that $u\le\mu^{-1}$ in the regime of the claim. Also, choose $U_0\ge\max\bigl\{6/c',2/c'',1\bigr\}$.

\medskip\noindent
\underline{\emph{Case 1: $u\le U_0$}}.
By \Cref{prop:baseline-assouad} and the fact that $e^{R/4}\ge1$,
\[
    \mathfrak{R}_n(R,\mu,q)\ge c \cdot \min\left\{R^2,\frac{1}{nq\mu^2}\right\}.
\]
Hence, whenever $n\ge C/(q\mu^2R^2)$, we obtain $\mathfrak{R}_n(R,\mu,q)\ge c/(nq\mu^2)$. Since $u\le U_0$, we have
\begin{equation}
\label{eq:regime-I-small-u}
    \mathfrak{R}_n(R,\mu,q)\ge\frac{c}{nq\mu^2}\exp(cu)
    =\frac{c}{nq\mu^2}\exp\left\{c\frac{R}{\sqrt{q}}\right\}.
\end{equation}

\noindent
\underline{\emph{Case 2: $u > U_0$}}.
Set
\begin{equation}
\label{eq:regime-I-parameters}
    \lambda:=2\sqrt{q},
    \quad\,\text{and}\quad\,
    p:=\left\lfloor\frac{c''}{\mu}\right\rfloor.
\end{equation}

Also, let $m$ be the largest positive odd integer satisfying $m\le c'u$. Then, thanks to the choice of $U_0$, we have $m\ge5$, and since consecutive positive odd integers differ by $2$,
\begin{equation}
\label{eq:regime-I-m}
    \frac{2c'u}{3}\le m\le c'u.
\end{equation}
By \Cref{eq:choice-a-b}, $m\lambda\le2c'R\le R/(4\sqrt{2})$. Moreover, since $u\le\mu^{-1}$ and $\lambda\le2$,
\[
    m\le\frac{c'}{\mu}\le\frac{\operatorname{sech}^2(2)}{2\mu}
    \le\frac{\operatorname{sech}^2(\lambda)}{2\mu}.
\]
Also, $u>U_0\ge2/c''$ and $u\le\mu^{-1}$ imply $c''/\mu>2$. Thus,
\begin{equation}
\label{eq:regime-I-p}
    \frac{c''}{2\mu}\le p\le\frac{c''}{\mu}.
\end{equation}
Since $c''\le1/2$, we also have $p\le (2\mu)^{-1}$. Therefore, the deterministic conditions \eqref{eq:generic-assouad-deterministic-conditions} in \Cref{lem:generic-assouad-reduction} are satisfied.

We next control $\mathcal{B}_{m,\lambda,q}$ (\Cref{eq:def-Bmlq}). For the present choice of $\lambda$, let $t := q/\sinh^2(2\sqrt{q})$.
Since $x\mapsto\sinh(x)/x$ is increasing on $(0,\infty)$, $1/\sinh^2(2)\le t\le1/4$. Recalling the function $\beta(\cdot)$ from \Cref{eq:beta-function}, define
\[
    \underline{\beta}:=\beta\left(\frac{1}{\sinh^2(2)}\right)>0,
    \qquad
    \overline{\beta}:=\beta\left(\frac{1}{4}\right)=\frac{1+\sqrt{17}}{8}<1.
\]

Since $\beta(\cdot)$ is a strictly increasing function on $[0,\infty)$, we have $\underline{\beta}\le\beta(t)\le\overline{\beta}$. Hence, using \Cref{eq:regime-I-m},
\begin{equation}
\label{eq:regime-I-B-bounds}
    \exp(\underline{c}u)\le\mathcal{B}_{m,\lambda,q}^{-1}\le\exp(\overline{c}u)
\end{equation}
for universal constants $\underline{c},\overline{c}>0$. 

It remains to verify the sample-size condition \eqref{eq:generic-assouad-sample-size} in \Cref{lem:generic-assouad-reduction}. Since $\sinh(2\sqrt{q})\ge2\sqrt{q}$, we have $G_\lambda\le2+(4q)^{-1}\le C/q$, and therefore $(G_\lambda)^m \le \exp(CuL_q)$. Combining this with \Cref{eq:regime-I-p,eq:regime-I-B-bounds}, and using $u\ge1$, $R^2=u^2q$, and $\mu\le1$, yields
\begin{equation}
\label{eq:regime-I-sample-size-control}
    \frac{1}{\mathcal{B}_{m,\lambda,q}}
    \max\left\{\frac{p}{q^2},\frac{p}{q}\,\bigl(G_\lambda\bigr)^m,\frac{p^2}{qR^2}\right\}
    \le\frac{C}{\mu^2}\exp(C'uL_q)
\end{equation}
for universal constants $C,C'>0$. Consequently, the sample-size condition \eqref{eq:generic-assouad-sample-size} of \Cref{lem:generic-assouad-reduction} holds whenever
\begin{equation}
\label{eq:regime-I-large-u-sample-size}
    n\ge\frac{C}{\mu^2}
    \exp\left\{C'\frac{R}{\sqrt{q}}\log\left(\frac{e}{q}\right)\right\}.
\end{equation}

Applying \Cref{lem:generic-assouad-reduction} and using \Cref{eq:regime-I-p,eq:regime-I-B-bounds}, we conclude that
\begin{equation}
\label{eq:regime-I-large-u}
    \mathfrak{R}_n(R,\mu,q)
    \ge\frac{c}{nq\mu^2}
    \exp\left\{c\frac{R}{\sqrt{q}}\right\}.
\end{equation}

Combining \Cref{eq:regime-I-small-u,eq:regime-I-large-u}, we obtain the desired lower bound throughout $R\mu\le\sqrt{q}$: provided
\[
    n\ge\frac{C_0}{q\mu^2R^2},
    \quad\,\text{and}\quad\,
    n\ge\frac{C_1}{\mu^2}
    \exp\left\{C_2\frac{R}{\sqrt{q}}\log\left(\frac{e}{q}\right)\right\},
\]
for sufficiently large universal constants $C_0,C_1,C_2>0$,
\begin{equation}
\label{eq:lower-bound-regime-I}
    \mathfrak{R}_n(R,\mu,q)
    \ge\frac{c}{nq\mu^2}
    \exp\left\{c\frac{R}{\sqrt{q}}\right\}.
\end{equation}
This proves \Cref{claim:lower-bound-regime-I}.
\end{proof}

\begin{proof}[Proof of \Cref{claim:lower-bound-regimes-II-III}]
Set $h:=\min\{R\mu,1\}$. Then, $\sqrt{q}\le h\le1$ and $h\le R\mu$. We first consider the range $0<\mu\le\mu_{\mathrm{crit}}$, where the Boolean construction with $m\ge5$ applies.

\medskip\noindent
\underline{\emph{Case 1: $0<\mu\le\mu_{\mathrm{crit}}$}}.
Let $m$ be the largest positive odd integer satisfying $m\le c'/\mu$, and set
\begin{equation}
\label{eq:regimes-II-III-parameters}
    \lambda := 2h = 2 \min\{R\mu,1\},
    \quad\,\text{and}\quad\,
    p:=\left\lfloor\frac{c''}{\mu}\right\rfloor.
\end{equation}
Since $\mu\le\mu_{\mathrm{crit}}\le\min\{c'/6,c''/2\}$, we have
\begin{equation}
\label{eq:regimes-II-III-mp}
    \frac{2c'}{3\mu}\le m\le\frac{c'}{\mu},
    \quad\,\text{and}\quad\,
    \frac{c''}{2\mu}\le p\le\frac{c''}{\mu}.
\end{equation}
Moreover, $m\lambda\le 2c'R\le R/(4\sqrt{2})$, while $\lambda\le2$ and \Cref{eq:choice-a-b} give
\[
    m\le\frac{c'}{\mu}\le\frac{\operatorname{sech}^2(2)}{2\mu}
    \le\frac{\operatorname{sech}^2(\lambda)}{2\mu}.
\]
Since $c''\le1/2$, we have $p\le1/(2\mu)$. Hence, the deterministic conditions of \Cref{lem:generic-assouad-reduction} are satisfied.

We next bound $\mathcal{B}_{m,\lambda,q}$. First, note that for any $t > 0$, $\beta(t)$ satisfies the following:
\begin{align}
\label{eq:beta-quadratic}
    \beta(t)^2
    &= 
    \frac{1}{4} (t^2 + 2t\sqrt{t^2 + 4t} + t^2 + 4t)
    =
    \frac{2t}{4} \Bigl[t + \sqrt{t^2 + 4t} + 2\Bigr]
    =
    \frac{t}{2} \Bigl[2 \, \beta(t) + 2\Bigr]
    =
    t\{1+\beta(t)\}.
\end{align}
Write $t:=q/\sinh^2(2h)$ and $z:=1/\beta(t)$. \Cref{eq:beta-quadratic} yields that
\[
    z(z+1)=\frac{1}{t}=\frac{\sinh^2(2h)}{q}.
\]
Since $h\ge\sqrt{q}$ and $\sinh(2h)>2h$, we have $z(z+1)>4$, and hence $z>4/3$. Furthermore,
\[
    \frac{1-q}{q}h^2\le\frac{h^2}{q}\le\frac{z(z+1)}{4}.
\]
For $z>4/3$, $1+z(z+1)/4\le z^2$, and therefore
\[
    \log(z)\ge\frac{1}{2}\log\Bigl(1+\frac{1-q}{q}h^2\Bigr).
\]
Since $\mathcal{B}_{m,\lambda,q}^{-1}=z^m$, \Cref{eq:regimes-II-III-mp} yields
\begin{equation}
\label{eq:regimes-II-III-B-lower}
    \mathcal{B}_{m,\lambda,q}^{-1}
    \ge
    \exp\biggl\{
        \frac{c'}{3\mu}
        \log\biggl(
            1+\frac{1-q}{q}\min\{R^2\mu^2,1\}
        \biggr)
    \biggr\}.
\end{equation}

It remains to verify the sample-size requirement. Since $2\sqrt{q}\le\lambda\le2$, we have $G_\lambda\le C/q$. Moreover, $z(z+1)=\sinh^2(2h)/q\le C/q$, so $z\le C/\sqrt{q}$, and hence
\[
    \mathcal{B}_{m,\lambda,q}^{-1}\le\exp(CmL_q),
    \quad\,\text{and}\quad\,
    \bigl(G_\lambda\bigr)^m\le\exp(CmL_q).
\]
Using $m\le c'/\mu$, $p\le c''/\mu$, and $R\mu>\sqrt{q}$, it follows that
\begin{equation}
\label{eq:regimes-II-III-sample-control}
    \frac{1}{\mathcal{B}_{m,\lambda,q}}
    \max\left\{
        \frac{p}{q^2}, \frac{p}{q}\bigl(G_\lambda\bigr)^m, \frac{p^2}{qR^2}
    \right\}
    \le
    \frac{C}{\mu^2}
    \exp\left\{\frac{C'}{\mu}L_q\right\}.
\end{equation}
Thus, for sufficiently large universal constants $C_0,C_1>0$, the sample-size condition
\begin{equation*}
    n\ge\frac{C_0}{\mu^2}
    \exp\left\{\frac{C_1}{\mu}\log\left(\frac{e}{q}\right)\right\}
\end{equation*}
allows us to invoke \Cref{lem:generic-assouad-reduction}. Combining that lemma with \Cref{eq:regimes-II-III-mp,eq:regimes-II-III-B-lower} yields
\begin{equation*}
    \mathfrak{R}_n(R,\mu,q)
    \ge
    \frac{c_0}{nq\mu^2}
    \exp\biggl\{
        \frac{c_1}{\mu}
        \log\biggl(
            1+\frac{1-q}{q}\min\{R^2\mu^2,1\}
        \biggr)
    \biggr\},
\end{equation*}
for universal constants $c_0,c_1>0$.

\noindent
\underline{\emph{Case 2: $\mu_{\mathrm{crit}}<\mu\le\nicefrac{1}{3}$}}.
In this range, the choice $m\asymp\mu^{-1}$ may no longer yield $m\ge5$, so we instead use a two-point construction with $m=1$. 
Since $\mu\ge\mu_{\mathrm{crit}}$, the quantity $1/\mu$ is bounded by a universal constant, and a fixed-dimensional construction is sufficient. 

Write $h:=\min\{R\mu,1\}$, and choose
\[
    \lambda:=\frac{h}{8},
    \quad~\text{and}\quad~
    \alpha^2:=c'''\frac{\sinh^2(\lambda)}{nq^2}
\]
for a sufficiently small universal constant $c'''$. In this case, we may assume a less stringent sample-size condition
\begin{equation}
\label{eq:moderate-mu-sample-size}
    n\ge C q^{-3}
\end{equation}
for a sufficiently large universal constant $C > 0$. We show the same lower bound on the minimax risk as in \Cref{eq:lower-bound-regimes-II-III} holds under the sample-size condition in \Cref{eq:moderate-mu-sample-size}. For brevity, we defer the formal construction and statement (see \Cref{lem:moderate-mu-lower-bound}) to \Cref{proof:lem:moderate-mu-lower-bound}.

Also, note that there exist sufficiently large universal constants $C_0,C_1>0$ such that the sample-size condition as in \Cref{eq:regimes-II-III-sample-size} implies the sample-size condition as in \Cref{eq:moderate-mu-sample-size}, since for $\mu_{\mathrm{crit}}<\mu\le\nicefrac{1}{3}$, choosing $C_1\ge1$ yields
\[
    \exp\left\{\frac{C_1}{\mu}\log\left(\frac{e}{q}\right)\right\}
    \ge q^{-C_1/\mu}
    \ge q^{-3}.
\]

Combining the two cases, we have
\begin{equation}
\label{eq:lower-bound-regimes-II-III-unified}
    \mathfrak{R}_n(R,\mu,q)
    \ge
    \frac{c_0}{nq\mu^2}
    \exp\biggl\{
        \frac{c_1}{\mu}
        \log\Bigl(
            1+\frac{1-q}{q}\min\{R^2\mu^2,1\}
        \Bigr)
    \biggr\}
\end{equation}
whenever $R\mu>\sqrt{q}$, and the sample-size condition in \Cref{eq:regimes-II-III-sample-size} holds. This proves \Cref{claim:lower-bound-regimes-II-III}.
\end{proof}


\section{Discussion} 
\label{sec:discussion}

We have studied parameter estimation in well-specified logistic regression with missing covariates in a setting where the covariate distribution is unknown. 
By formulating the problem as $Z$-estimation with a monotone operator, we constructed a distribution-free estimating equation and then solved for its root via a computationally efficient stochastic approximation procedure. Our results show that estimation is possible at a parametric rate without modeling or estimating the covariate distribution. Moreover, the matching lower bounds that we derive also show that the nonstandard dependence of our upper bounds on the missingness probability is intrinsic to the problem. The key takeaway is that missingness in covariates can impose a substantially more intricate statistical cost in logistic regression than in the extensively studied linear regression problem. We conclude by discussing a few natural open directions.

A first open direction is to weaken the distributional assumptions in our model. Although we only assume bounded covariates (viz., $\norm{X}_2\le1$ almost surely) and a very minimal curvature condition (viz., $\E[XX^\top] \succeq \mu I_d$), a natural question to ask would be whether similar guarantees can be obtained for unbounded covariates under suitable tail or moment assumptions (e.g., sub-Gaussian or sub-exponential tails). 
More broadly, we could allow $(X, Y)$ to be drawn from a general joint distribution, and seek to estimate the population minimizer of the logistic loss $\thetastar$. Is it possible to estimate this quantity at a parametric rate without distributional assumptions? Our current approach crucially exploits the fact that $P_{Y|X}$ is a logistic regression model to construct our population operator $\PopOp$. Constructing an analogous estimating equation in the misspecified setting remains open.

A second direction concerns logistic regression with missing covariates, but with a structured parameter $\thetastar$. In this work, we impose no structural assumption on $\thetastar$, but in many high-dimensional and nonparametric problems, the regression parameter may be sparse or possess some other low-dimensional structure~\citep[see, e.g.,][]{PlanVershynin-2013-robust1bit,AbramovichGrinshtein-2019-HDlogistic}. Is it possible to exploit such structure in the presence of missing covariates? More importantly, does the statistical price of missingness (in a minimax sense) depend on the sparsity level, or some more refined notion of effective dimension? It would also be interesting to design computationally efficient estimators that attain such rates. Another related direction is to go beyond minimax characterizations and design instance-optimal stochastic approximation schemes using our operator or its relatives. This would involve moving beyond the vanilla stochastic approximation scheme in \Cref{alg:Proj-SA} and incorporating, e.g., careful variance reduction techniques that have recently been explored in this context for stochastic convex optimization~\cite[see, e.g.,][]{LiweiAshwin2026visor}.

A third direction is to consider a broader class of nonlinear regression problems with MCAR covariates.  As noted in Section~\ref{sec:related}, our estimator falls in the large class of augmented inverse propensity weighted (AIPW) estimators~\citep{robins1994estimation}.  In general, it is difficult to use the optimal AIPW estimating equation.  Since our estimator is a member of the AIPW class and additionally enjoys monotonicity---thereby ensuring computationally efficient estimation---it is natural to ask whether it is possible to extend the ideas developed here to generic nonlinear regression problems.  In particular, it would be interesting to find monotone operators in the corresponding AIPW class that can be computed efficiently and enjoy optimality guarantees in a minimax sense.

A final set of open directions is to study our problem under broader settings of missingness.  Indeed, in this paper, we focused on an MCAR mechanism with known observation probabilities.  In practice, the MCAR assumption is often overly restrictive and real missingness patterns may be missing not at random (MNAR); parametric and nonparametric tests are discussed by \citet{little1988test} and \citet{berrett2023optimal} respectively.   While the general MNAR assumption may be too broad for our setting,~\citet{Kabir2026EstMCAR} proposed the \emph{realizable contamination model} which interpolates between an MCAR mechanism and an arbitrary MNAR mechanism.  We conclude by asking: What do optimal procedures look like in the realizable contamination model and what are the optimal tradeoffs between accuracy, robustness, and computational efficiency when we move beyond MCAR?


\section*{Acknowledgments and disclosure of AI usage}

JRC and AP were supported in part by the National Science Foundation through grant CCF-2107455, and a Google Research Scholar Award. JRC was also partially supported by the Ronald~J. and Carol T. Beerman President's Fellowship at Georgia Tech. KAV performed a portion of this work at Georgia Tech, where he was partially supported by the Gary C. Butler Family Postdoctoral Fellowship. RJS was supported by European Research Council Advanced Grant 101019498.

The methodology, results and proofs were developed by the authors. All writing was done by the authors. \textsc{OpenAI}'s GPT 5.6 Sol\footnote{\url{https://openai.com/index/previewing-gpt-5-6-sol/}} and some earlier models were used to aid in the proofs of some of the technical results (Lemmas~11, 12, 17, 20, and 22), to assist in writing code for the simulations, and for proofreading the final draft of the paper. \textsc{Undermind}\footnote{\url{https://www.undermind.ai/}} and \textsc{Anthropic}'s Claude Opus 5.5\footnote{\url{https://www.anthropic.com/claude-opus-5-5}} were used during a preliminary literature survey.

\bibliographystyle{abbrvnat-clean}
\bibliography{refs}

\bigskip\bigskip

\newpage
\appendix

\begin{center}
\LARGE\bfseries\mbox{Appendix}
\end{center}

\section{Proofs of results associated with the upper bound}
\label{appsec:proof-upper}

In this appendix, we prove the preliminary lemmas that were used in the proof of Theorem~\ref{thm:general-sgd-rate}.

\subsection{Proof of \Cref{lem:A-properties}} \label{sec:lem-A-proof}

We prove the two parts of the lemma separately.

\noindent \emph{(a)} 
Let $\eta^\star := \ip{X}{\thetastar}$.  Then
\[
    \bbP(Y=1\mid X)
    =
    \frac{e^{\eta^\star/2}}{e^{\eta^\star/2}+e^{-\eta^\star/2}},
    \qquad\text{and}\qquad
    \bbP(Y=-1\mid X)
    =
    \frac{e^{-\eta^\star/2}}{e^{\eta^\star/2}+e^{-\eta^\star/2}}.
\]
Therefore,
\begin{align*}
    \E\bigl[
        Y e^{-Y\eta^\star/2}
        \mid X
    \bigr]
    &=
    \bbP(Y=1\mid X) \cdot e^{-\eta^\star/2}
    +
    (-1) \cdot \bbP(Y=-1\mid X) \cdot e^{\eta^\star/2}
    \\ 
    &=
    \frac{e^{\eta^\star/2} e^{-\eta^\star/2}}
    {e^{\eta^\star/2}+e^{-\eta^\star/2}}
    -
    \frac{e^{-\eta^\star/2}e^{\eta^\star/2}}
    {e^{\eta^\star/2}+e^{-\eta^\star/2}}
    =0.
\end{align*}
Multiplying by $-X$ and taking expectations yields $\PopOp(\thetastar)=0$, as desired.

\noindent
\emph{(b)}
Since $X$ is bounded almost surely and $\theta \in \Theta_R$, we may differentiate under the expectation to obtain
\[
    \mathcal{D} \PopOp(\theta)
    =
    \frac{1}{2}
    \E\Bigl[
        XX^\top
        \exp\bigl\{-\tfrac{1}{2}Y\ip{X}{\theta}\bigr\}
    \Bigr],
\]
where $\mathcal{D}$ denotes the Jacobian operator.
Fix $\theta\in\Theta_R$ and write $\eta:=\ip{X}{\theta}$. For $y\in\{-1,1\}$,
\begin{align*}
    \bbP(Y=y\mid X)
    =
    \sigma(y\eta^\star)
    &= \frac{\exp(y\eta^\star/2)}
    {\exp(y\eta^\star/2)+\exp(-y\eta^\star/2)}
    =
    \frac{\exp(y\eta^\star/2)}
    {2\cosh(\eta^\star/2)}.
\end{align*}
Therefore,
\begin{align*}
    \E\bigl[
        \exp\bigl(-\tfrac{1}{2}Y\eta\bigr)
        \bigm| X
    \bigr]
    &=
    \frac{\exp\{(\eta^\star-\eta)/2\} + \exp\{-(\eta^\star-\eta)/2\}}{2\cosh(\eta^\star/2)}
    =
    \frac{\cosh\{(\eta^\star-\eta)/2\}}{\cosh(\eta^\star/2)}.
\end{align*}
Since $\cosh(u)\geq1$ for every $u\in\R$, and
$
    |\eta^\star|
    =
    |\langle X,\thetastar\rangle|
    \le
    \norm{X}\norm{\thetastar}
    \le
    R,
$
we obtain
\[
    \E\bigl[
        \exp\bigl(-\tfrac{1}{2}Y\eta\bigr) \!\bigm|\! X
    \bigr] 
    \ge
    \frac{1}{\cosh(R/2)}.
\]
Consequently,
\begin{align}
\label{eq:PopOp-hessian}
    \mathcal{D} \PopOp(\theta)
    =
    \frac{1}{2}
    \E\Bigl[
        XX^\top \exp\Bigl(-\mfrac{1}{2} Y\eta\Bigr)
    \Bigr]
    \succeq
    \frac{1}{2\cosh(R/2)}
    \E[XX^\top]
    \succeq
    \frac{\mu}{2\cosh(R/2)}I_d 
    =
    \alpha_R I_d.
\end{align}
Applying the mean value theorem to the map $t \mapsto \big\langle \PopOp\bigl(\theta' + t (\theta - \theta')\bigr) - \PopOp(\theta'), \theta - \theta'\big\rangle$, we then have  
\[
    \ip{\PopOp(\theta)-\PopOp(\theta')}{\theta-\theta'}
    \ge
    \alpha_R\|\theta-\theta'\|^2
    \qquad\text{for all}\,~
    \theta, \theta' \in \Theta_R,
\]
as claimed.
\qed

\subsection{Proof of \Cref{lem:A-hat-properties}} \label{sec:lem-Ahat-proof}

\noindent \emph{(a)}
Fix $\theta\in\Theta_R$. Conditional on $(X,Y)$, the masks $M_1,\ldots,M_d$ are independent with $M_j \sim \BER(q_j)$. 
For each $j \in [d]$, we have
\begin{align}
    \E\left[
        \widehat f_j(\theta) \,\middle|\, Y,X
    \right]
    &=
    1+
    \E\biggl[
        \frac{M_j}{q_j}
        \Bigl\{
            \exp\bigl(-\tfrac{1}{2}Y\theta_j Z_j\bigr)-1
        \Bigr\} \,\bigg|\, Y,X
    \biggr] \nonumber\\ 
    &=
    1+
    \Bigl\{
        \exp\bigl(-\tfrac{1}{2}Y\theta_j X_j\bigr)-1
    \Bigr\} 
    = 
    \exp\Bigl\{-\tfrac{1}{2}Y\theta_j X_j\Bigr\}.
\label{eq:fj-hat-exp}
\end{align}
Hence
\begin{align*}
    \E\Bigl[
        [\StochOp(\theta)]_j \,\Big|\, Y,X
    \Bigr]
    &=
    \E\biggl[
        -Y\frac{Z_j}{q_j}
        \exp\Bigl\{-\mfrac{1}{2}Y\theta_j Z_j\Bigr\}
        \prod_{k\ne j}\widehat{f}_k(\theta)\,\bigg|\, Y,X
    \biggr]\\
    &\overset{\1}{=}
    -Y X_j
    \exp\Bigl\{-\mfrac{1}{2}Y\theta_j X_j\Bigr\}
    \prod_{k\ne j} \E\Bigl[\widehat{f}_k(\theta) \,\Big|\, Y,X\Bigr]
    \\
    &\overset{\2}{=}
    -Y X_j
    \exp\Bigl\{-\mfrac{1}{2}Y\theta_j X_j\Bigr\}
    \prod_{k\ne j} \exp\Bigl\{-\mfrac{1}{2}Y\theta_k X_k\Bigr\}
    \\
    &=
    -Y X_j
    \exp\Bigl\{-\mfrac{1}{2}Y\ip{X}{\theta}\Bigr\},
\end{align*}
where step $\1$ uses $\E\Big[-Y\frac{Z_j}{q_j}
        \exp\Bigl\{-\mfrac{1}{2}Y\theta_j Z_j\Bigr\} \,|\, X_j, Y\Big] = -Y X_j
    \exp\Bigl\{-\mfrac{1}{2}Y\theta_j X_j\Bigr\}$, and step $\2$ uses \Cref{eq:fj-hat-exp}. Taking an expectation over $(X,Y)$ therefore completes the proof.

\medskip\noindent 
\emph{(b)}
Fix $\theta\in\Theta_R$ and $j \in [d]$, and define
\[
    T_j:=-\mfrac{1}{2}Y\theta_jX_j,
    \qquad\text{for}\,~ j\in[d].
\]
Since $Y^2=1$, we have
\begin{align}
\E\Bigl[ \bigl[\StochOp(\theta)\bigr]_j^2 \,\Big|\, Y,X \Bigr] 
&= \E\biggl[
    \frac{M^2_j X_j^2}{q_j^2}
    e^{2T_j} \prod_{k\ne j} \widehat f_k(\theta)^2 \,\bigg|\, Y,X
\biggr] 
=
    \frac{X_j^2}{q_j} e^{2T_j}
    \prod_{k\ne j}
    \E\Bigl[ \widehat f_k(\theta)^2 \Bigm| Y,X \Bigr].
\label{eq:StochOp-second-moment-exact}
\end{align}
Now
\begin{align}
    \E\Bigl[
        \widehat f_k(\theta)^2
        \,\Big|\, Y,X
    \Bigr] &= \E\biggl[\Bigl(1 + \frac{M_k}{q_k}(e^{T_k}-1)\Bigr)^2\Bigm| Y,X]
    =
    1+2\bigl(e^{T_k}-1\bigr)+\mfrac{1}{q_k} \bigl(e^{T_k}-1\bigr)^2 \nonumber\\ 
    &=
    e^{2T_k} + \Bigl(\mfrac{1}{q_k}-1\Bigr) \bigl(e^{T_k}-1\bigr)^2 \nonumber\\ 
    &=
    e^{2T_k}
    \Bigl[
        1 + \Bigl(\mfrac{1}{q_k}-1\Bigr) \bigl(1-e^{-T_k}\bigr)^2
    \Bigr]
    \,=:\,
    e^{2T_k} \Bigl[ 1 + \gamma_{q_k} (1-e^{-T_k})^2 \Bigr],
\label{eq:f-hat-sq-exact-1}
\end{align}
where we define $\gamma_q := \frac{1}{q} - 1 = \frac{1-q}{q}$ for convenience.  But for $t \in \mathbb{R}$ and $\gamma \geq 0$, we have
\begin{align} 
\label{eq:numeric-ineqs}
e^{2t} \bigl[ 1 + \gamma (1-e^{-t})^2 \bigr] = e^{2t} + \gamma (e^{t} - 1)^2 \overset{\1}{\leq} e^{2|t|} \bigl[ 1 + \gamma (1-e^{-|t|})^2 \bigr] \leq e^{2|t|} \bigl[ 1 + \gamma \min(|t|, 1)^2 \bigr],
\end{align}
where inequality $\1$ follows because $e^{2t} \leq e^{2|t|}$ and $|e^t - 1| \leq e^{|t|} - 1$.

Combining~\Cref{eq:StochOp-second-moment-exact}, \Cref{eq:f-hat-sq-exact-1} and \Cref{eq:numeric-ineqs} yields
\begin{align}
\mathbb{E}\Bigl[\bigl[ \StochOp(\theta)\bigr]_j^2 \,\Big|\, Y,X \Bigr]
&\leq 
\frac{X_j^2}{q_j} 
 \, e^{2T_j} \cdot
\exp\biggl\{
2\sum_{k \neq j} |T_k| + \sum_{k \neq j} \log \bigl(1 + \gamma_{q_k} \min(|T_k|, 1)^2 \bigr)
\biggr\} \nonumber\\
&\leq  
\frac{X_j^2}{q_j} 
 \cdot
\exp\biggl\{
2\sum_{k=1}^d |T_k| + \sum_{k \neq j} \log \bigl(1 + \gamma_{q_k} \min(|T_k|, 1)^2 \bigr)
\biggr\}. \label{eq:StochOp-second-moment-bound-1.5}
\end{align}

Now recall the function $\phi_q$ from~\Cref{eq:phi-function}, and set $S_j := \min(|T_j|, 1)$.  By Cauchy--Schwarz,
\[
\sum_{j = 1}^d S_j \leq \sum_{j = 1}^d |T_j| =\frac{1}{2}\sum_{j=1}^d |\theta_jX_j|\le\frac{1}{2}\norm{\theta}\norm{X}\le\frac{R}{2}, 
\]
so by~\Cref{eq:StochOp-second-moment-bound-1.5},
\begin{align}
\mathbb{E}\Bigl[\bigl[ \StochOp(\theta)\bigr]_j^2 \,\Big|\, Y,X \Bigr]
&\le
\frac{e^R \cdot X_j^2}{q_j}\,
\exp\biggl\{
\sum_{k \neq j} \phi_{q_k} \bigl(S_k\bigr)
\biggr\}
\leq e^R X_j^2 \cdot b_j(\boldsymbol{q}, R).
\label{eq:StochOp-second-moment-bound-2}
\end{align}
Summing~\Cref{eq:StochOp-second-moment-bound-2} over $j \in [d]$ and taking a further expectation over $Y,X$ completes the proof.
\qed

\subsection{Proof of Corollary~\ref{cor:sgd-rate}}
\label{proof:cor:sgd-rate}

The corollary follows directly from Theorem~\ref{thm:general-sgd-rate} when $d=1$, so we assume that $d \geq 2$ for the rest of this proof.
Recall the least concave majorant $\overline\phi_q$ of $\phi_q$~\eqref{eq:phi-function} on $[0,\infty)$. By our convention, the function $\overline\phi_q$ is non-decreasing and concave on its entire domain $[0,\infty)$. By definition, $\phi_q(s)\le\overline\phi_q(s)$ and $\overline\phi_q$ is concave. Jensen's inequality implies that for any fixed $j \in [d]$,
\begin{align}
\label{eq:phi-jensen}
    \sum_{k \neq j} \phi_q(s_k)
    &\le\sum_{k \neq j} \overline\phi_q(s_k)
    \overset{\1}{\le} (d-1)\,\overline\phi_q\biggl(\frac{1}{d-1}\sum_{k \neq j} s_k\biggr) \overset{\2}{\leq} (d-1)\,\overline\phi_q\Bigl(\mfrac{R}{2(d-1)}\Bigr),
\end{align}
where in step $\1$, we have used Jensen's inequality and in step $\2$ we have the constraint that $\sum_{k \neq j} s_k \leq R/2$ along with the fact that $\overline{\phi}_q$ is non-decreasing on $[0,\infty)$.

Next, observe that when $q_1 = \cdots = q_d = q \in (0,1]$, for any $j \in [d]$,
\begin{equation*}
b_j(\boldsymbol{q}, R)
= 
\frac{1}{q}
\sup_{\substack{\boldsymbol{s} \in [0,1]^d \\ \sum_{k = 1}^d s_k \leq R/2}}  
\exp\biggl\{\,\sum_{k \neq j} \phi_{q}(s_k) \biggr\}
\overset{\1}{=}
\frac{1}{q}
\sup_{\substack{(s_k)_{k \neq j} \in [0,1]^{d-1} \\[1mm] \sum_{k \neq j} s_k \leq R/2}}  
\exp\biggl\{\,\sum_{k \neq j} \phi_{q}(s_k) \biggr\},
\end{equation*}
where step $\1$ follows since the objective function does not depend on $s_j$, so we can set $s_j = 0$. Consequently, the resulting optimization problem is identical for each $j \in [d]$, up to relabeling of the coordinates; i.e.,
\[
b_1(\boldsymbol{q}, R) = \cdots = b_d(\boldsymbol{q}, R) =: b_\ast(\boldsymbol{q},R). 
\]

Thus, a direct application of Theorem~\ref{thm:general-sgd-rate} yields
\begin{align*}
\E\,\big\|\widehat\theta^{(n)} - \thetastar \big\|^2
&\le 
\frac{(e^R+1)^2}{n\mu^2} \cdot \biggl\{\,b_{\ast}(\boldsymbol{q}, R) \cdot \sum_{j = 1}^d \mathbb{E}[X_j^2] \cdot  \biggr\} 
\overset{\1}{\leq}
\frac{(e^R+1)^2}{n\mu^2} \cdot b_{\ast}(\boldsymbol{q}, R),
\end{align*}
where step $\1$ follows because $\| X \|_2 \leq 1$. Combining with~\Cref{eq:phi-jensen} yields
\[
\E\,\big\|\widehat\theta^{(n)} - \thetastar \big\|^2
\le \frac{(e^R+1)^2}{nq\mu^2} \cdot \exp\Big\{ (d-1)\,\overline\phi_q\Bigl(\mfrac{R}{2(d-1)}\Bigr) \Big\},
\]
as claimed.
\qed

For completeness, we formally state and prove that $\overline\phi_q$ defined in \Cref{eq:phiq-LCM-definition} is indeed the least concave majorant of $\phi_q$ on $[0,1]$.

\begin{claim}
\label{claim:phiq-LCM}
The least concave majorant $\overline\phi_q$ of the function $\phi_q : [0,1] \to [0,\log(1/q)]$ has the closed-form expression given by \Cref{eq:phiq-LCM-definition}, with the convention that $\overline{\phi}_q(s) = \log(1/q)$ for all $s > 1$.
\end{claim}

\begin{proof}[Proof of \Cref{claim:phiq-LCM}]
For $q<1$, write
\[
    r_q(s):=\frac{\phi_q(s)}{s}=\frac{\log\bigl(1+\gamma_qs^2\bigr)}{s},\qquad s\in(0,1].
\]
Differentiating $r_q(s)$, we obtain
\[
    r_q'(s)
    =
    \frac{1}{s^2}\left\{\frac{2\gamma_qs^2}{1+\gamma_qs^2}-\log\bigl(1+\gamma_qs^2\bigr)\right\}
    =
    \frac{h(\gamma_qs^2)}{s^2},
\]
where we recall from \Cref{eq:def-h-func} that
\[
    h(x) =\frac{2x}{1+x}-\log(1+x),\qquad\text{for}\,~ x\ge0.
\]
Note that
\[
    h'(x)=\frac{1-x}{(1+x)^2}.
\]
Thus, $h$ is strictly increasing on $[0,1)$ and strictly decreasing on $(1,\infty)$. Since $h(0)=0$, $h(1)=1-\log(2)>0$, and $h(x)\to-\infty$ as $x\to\infty$, there exists a unique $x_\star>1$ such that
$h(x_\star)=0$.

\medskip\noindent
\underline{Case 1 ($q_\star \le q \le 1$)}.
First, consider the case that $q_\star\le q<1$. Then $\gamma_q\le x_\star$, so $h(\gamma_qs^2)\ge0$ for every $s\in[0,1]$. Hence $r_q$ is non-decreasing on $(0,1]$, and therefore
\[
    \phi_q(s)\le s\,\phi_q(1)=s\log\frac{1}{q},\qquad s\in[0,1].
\]
The function $s\mapsto s\log(1/q)$ is a concave majorant of $\phi_q$. Moreover, any concave majorant $g$ of $\phi_q$ satisfies
\[
    g(s) \ge (1-s) \cdot g(0)+s \cdot g(1) \ge s\,\phi_q(1)=s \,\log\frac{1}{q}.
\]
Consequently,
\begin{equation}
\label{eq:phi-concave-majorant-large-q}
    \overline\phi_q(s)=s\log\frac{1}{q},\qquad q_\star\le q<1,\quad s\in[0,1].
\end{equation}
For $q=1$, we have $\gamma_q=0$, so $\phi_q\equiv0$, and \Cref{eq:phi-concave-majorant-large-q} remains valid.

\medskip\noindent
\underline{Case 2 ($0 < q < q_\star$)}.
Now consider the case that $0<q<q_\star$. Then, $\gamma_q>x_\star$. Define
\[
    \tau_q:=\sqrt{\frac{x_\star}{\gamma_q}}=\sqrt{\frac{x_\star q}{1-q}}\in(0,1).
\]

Note that since $r_q'(s) = h(\gamma_qs^2)/s^2$, it has the same sign as $h(\gamma_qs^2)$. Thus, the function $r_q$ is increasing on $(0,\tau_q]$ and decreasing on $[\tau_q,1]$. Hence,
\[
    \phi_q(s)\le\frac{\phi_q(\tau_q)}{\tau_q}s,
    \qquad\text{for}\,~ s\in[0,\tau_q].
\]
Furthermore,
\begin{align*}
    \frac{\phi_q(\tau_q)}{\tau_q}
    &=\sqrt{\frac{\gamma_q}{x_\star}}\log(1+x_\star)
    =\sqrt{\frac{\gamma_q}{x_\star}} \frac{2 x_\star}{1+x_\star}
    =2\kappa_\star\sqrt{\gamma_q}.
\end{align*}

Also, note that
\[
    \phi_q''(s)=\frac{2\gamma_q(1-\gamma_qs^2)}{(1+\gamma_qs^2)^2}.
\]
Since $x_\star>1$, we have $\tau_q>1/\sqrt{\gamma_q}$, and therefore $\phi_q$ is concave on $[\tau_q,1]$. Moreover,
\[
    \phi_q'(\tau_q)=\frac{2\gamma_q\tau_q}{1+\gamma_q\tau_q^2}=\frac{2\sqrt{\gamma_q x_\star}}{1+x_\star}=2\kappa_\star\sqrt{\gamma_q}=\frac{\phi_q(\tau_q)}{\tau_q}.
\]
Thus, the linear function on $[0,\tau_q]$ joins $\phi_q$ tangentially at $\tau_q$. It follows that
\begin{equation}
\label{eq:phi-concave-majorant-small-q}
    \overline\phi_q(s)=
    \begin{cases}
        \displaystyle 2\kappa_\star\sqrt{\tfrac{1-q}{q}}\,s, & 0\le s\le\tau_q,\\ 
        \displaystyle \log\left(1+\tfrac{1-q}{q}s^2\right), & \tau_q\le s\le1.
    \end{cases}
\end{equation}
This proves that $\overline\phi_q$ is a concave majorant of $\phi_q$ on $[0,1]$. To prove minimality, let $g$ be any concave majorant of $\phi_q$. For $s\in[0,\tau_q]$, concavity of $g$ yields
\[
    g(s)
    \ge
    \left(1-\frac{s}{\tau_q}\right)g(0)+\frac{s}{\tau_q}g(\tau_q)
    \ge
    \frac{s}{\tau_q}\phi_q(\tau_q)
    =
    \overline\phi_q(s),
\]
while for $s\in[\tau_q,1]$, we necessarily have $g(s)\ge\phi_q(s)=\overline\phi_q(s)$. Thus, $\overline\phi_q$ is indeed the least concave majorant of $\phi_q$ on $[0,1]$, and given by: 
\begin{equation*}
\overline\phi_q(s)
\,=\,
\begin{cases} 
  \, s\log\frac{1}{q}, &\text{if}\,~ q \in [q_\star,1) ~\text{and}~ s\in[0,1], \\ 
  \, 2\kappa_\star\sqrt{\frac{1-q}{q}}\,s, &\text{if}\,~ q \in (0,q_\star] ~\text{and}~ s\in[0,\tau_q],\\ 
  \, \log\left(1+\frac{1-q}{q}s^2\right), &\text{if}\,~ q \in (0,q_\star] ~\text{and}~ s\in[\tau_q,1].
\end{cases}
\end{equation*}
which proves the claim since by convention, we have $\overline{\phi}_q(s) = \log(1/q)$ for $s \geq 1$.
\end{proof}

\subsection{Proof of \Cref{cor:sgd-rate-simplified}}
\label{proof:cor:sgd-rate-simplified}

For $D>0$, denote the RHS of \Cref{eq:phi-jensen} as a function in $D$ by
\begin{equation}
\label{eq:xi-D}
    \xi_q(D):=D\,\overline\phi_q\Bigl(\mfrac{R}{2D}\Bigr).
\end{equation}

\begin{claim}
\label{claim:xi-D-inc}
The function $\xi_q : (0,\infty) \to \R$ defined in \Cref{eq:xi-D} is non-decreasing.
\end{claim}

We take \Cref{claim:xi-D-inc} as given for the moment, postponing its proof to the end of the section. Since $d\le1/\mu$ from \Cref{ineq:dimension-mu-relation}, combining \Cref{claim:xi-D-inc} with \Cref{eq:phi-jensen} implies that
\begin{equation}
\label{eq:phi-jensen-mu}
   (d-1)\,\overline\phi_q\Bigl(\mfrac{R}{2(d-1)}\Bigr) =
    \xi_q(d-1)
    \le
    \xi_q(d)
    \le
    \xi_q\left(\frac{1}{\mu}\right)
    =
    \frac{1}{\mu} \, \overline\phi_q\Bigl(\mfrac{R\mu}{2}\Bigr).
\end{equation}

Substituting the expression for $\overline{\phi}_q$ from \Cref{claim:phiq-LCM}, we have the following two cases:

\begin{enumerate}[label=(\alph*)]
\item 
If $q \in [q_\star,1]$, then 
\begin{align} \label{eq:simplify-phibar-1}
    \frac{1}{\mu}\overline\phi_q\Bigl(\mfrac{R\mu}{2}\Bigr)
    = \min\left\{\frac{R}{2},\frac{1}{\mu}\right\}\log\frac{1}{q}
    =
    \begin{cases}
    \,\mfrac{R}{2}\log\mfrac{1}{q}, & ~ \text{if}\,~ R\mu\le2,\\ 
    \,\mfrac{1}{\mu}\log\mfrac{1}{q}, & ~ \text{if}\,~ R\mu>2.
    \end{cases}
\end{align}

\item 
If $q \in (0,q_\star)$, then
\begin{align} 
\label{eq:simplify-phibar-2}
    \frac{1}{\mu}\overline\phi_q\Bigl(\mfrac{R\mu}{2}\Bigr)
    =
    \begin{cases}
        \,\kappa_\star R\sqrt{\mfrac{1-q}{q}}, & \text{if}\,~ R\mu\le2\tau_q \,,\\ 
        \,\mfrac{1}{\mu}\log\left(1+\mfrac{1-q}{q}\mfrac{R^2\mu^2}{4}\right), & \text{if}\,~ 2\tau_q < R\mu\le2.
    \end{cases}
\end{align}
\end{enumerate}
We now simplify the above expression case-by-case, writing $\frac{1}{\mu}\overline\phi_q\Bigl(\mfrac{R\mu}{2}\Bigr) =: \mathcal{K}(q,R,\mu)$ as shorthand.
Write $x:=R\mu$ and recall that $\gamma_q=(1-q)/q$. We bound $\mathcal{K}(q,R,\mu)$ separately in the three regimes.

\paragraph{\underline{Regime $R\mu\le\sqrt{q}$} :}
If $q<q_\star$, then $x\le\sqrt{q}<2\sqrt{x_\star q/(1-q)}$ and $\log(2-q) > \log(2-q_\star)$, so from the first case in~\Cref{eq:simplify-phibar-2} we have
\[
    \mathcal{K}(q,R,\mu)=\kappa_\star R\sqrt{\gamma_q}\le\kappa_\star\frac{R}{\sqrt{q}}\le\frac{\kappa_\star}{\log(2-q_\star)}\frac{R}{\sqrt{q}}\log(2-q).
\]
Next, using the fact that $\log(t) \le t-1$ with $t = \frac{1}{q}$ and $t = (2-q)^{-1}$, respectively, we obtain that
\begin{equation}
\label{eq:upper-cor-regime-I-ineqs}
    \log\frac{1}{q} \le \frac{1}{q} - 1 = \frac{1-q}{q}
    \quad\text{and}\quad
    \log(2-q)\ge 1 - \frac{1}{2-q} = \frac{1-q}{2-q}\ge\frac{1-q}{2}.
\end{equation}
If $q\ge q_\star$, then the first case in \Cref{eq:simplify-phibar-1} applies. Thus, \Cref{eq:upper-cor-regime-I-ineqs} yields
\[
    \mathcal{K}(q,R,\mu)
    = \frac{R}{2}\log\frac{1}{q}
    \le \frac{R}{q} \frac{1-q}{2}
    \le \frac{R}{q_\star}\log(2-q)
    \le \frac{1}{q_\star}\frac{R}{\sqrt{q}}\log(2-q).
\]
Therefore,
\[
    \mathcal{K}(q,R,\mu)\le C_{\mathrm I}\frac{R}{\sqrt{q}}\log(2-q),
    \qquad\text{where}\,~
    C_{\mathrm I}:=\max\left\{\frac{\kappa_\star}{\log(2-q_\star)},\frac{1}{q_\star}\right\}.
\]

\paragraph{\underline{Regime $\sqrt{q} < R\mu \le 1$} :}
Set $y:=x\sqrt{\gamma_q}=R\mu\sqrt{\gamma_q}$.

Suppose first that $q<q_\star$.
On the one hand, if the first case in \Cref{eq:simplify-phibar-2} applies, then
\[
    \sqrt{1-q_\star}<\sqrt{1-q}<y\le2\sqrt{x_\star}.
\]
Since $\log(1+y^2)\ge 1 - 1/(1+y^2) = y^2/(1+y^2)$, it follows that
\[
    y\le\left(y+\frac{1}{y}\right)\log(1+y^2)
    \le\left(2\sqrt{x_\star}+\frac{1}{\sqrt{1-q_\star}}\right)\log(1+y^2).
\]
Consequently, using $R = \frac{y}{\mu\sqrt{\gamma_q}}$, we obtain
\begin{equation}
\label{eq:cor-regime-II-i}
    \mathcal{K}(q,R,\mu)
    =\kappa_\star R \sqrt{\gamma_q}
    =\kappa_\star \frac{y}{\mu}
    \le\frac{c'}{\mu}\log(1+\gamma_qx^2),
    \quad~\text{where}\,~
    c':=\kappa_\star\left(2\sqrt{x_\star}+\frac{1}{\sqrt{1-q_\star}}\right).
\end{equation}
On the other hand, if the second case in \Cref{eq:simplify-phibar-2} applies, then directly
\begin{equation}
\label{eq:cor-regime-II-ii}
    \mathcal{K}(q,R,\mu)=\frac{1}{\mu}\log\left(1+\frac{\gamma_qx^2}{4}\right)\le\frac{1}{\mu}\log(1+\gamma_qx^2).
\end{equation}

Now suppose that $q\ge q_\star$. Only the first case in \Cref{eq:simplify-phibar-1} applies, and \Cref{eq:upper-cor-regime-I-ineqs} yields
\[
    \mathcal{K}(q,R,\mu)
    = \frac{R}{2} \log\frac{1}{q}
    \le \frac{x}{\mu} \frac{1}{q}\frac{1-q}{2}
    \le \frac{1}{q_\star}\frac{x}{\mu}\log(2-q).
\]
Moreover, $\sqrt{q} < x$ implies $1-q < \gamma_qx^2$, so $\log(2-q)\le\log(1+\gamma_qx^2)$. Since $x\le1$,
\begin{equation}
\label{eq:cor-regime-II-iii}
    \mathcal{K}(q,R,\mu) \le \frac{1}{q_\star} \frac{1}{\mu} \log(1+\gamma_qx^2).
\end{equation}
Thus, combining \Cref{eq:cor-regime-II-i}, \Cref{eq:cor-regime-II-ii}, and \Cref{eq:cor-regime-II-iii}, we obtain
\begin{equation*}
    \mathcal{K}(q,R,\mu)
    \le\frac{C_{\mathrm{II}}}{\mu}\log\left(1+\frac{1-q}{q}R^2\mu^2\right),
    \qquad\text{where}\,~
    C_{\mathrm{II}}:=\max\left\{c',1,\frac{1}{q_\star}\right\}.
\end{equation*}

\paragraph{\underline{Regime $R\mu > 1$} :}
If $x>2$, then the definition of $\overline\phi_q$ in \Cref{eq:phiq-LCM-definition} directly yields
\[
    \mathcal{K}(q,R,\mu)=\frac{1}{\mu}\log\frac{1}{q}.
\]
It therefore remains to consider $1<x\le2$. 

First suppose that $q\ge q_\star$. Then only the first case in \Cref{eq:simplify-phibar-1} applies, and yields
\[
    \mathcal{K}(q,R,\mu)=\frac{x}{2\mu}\log\frac{1}{q}\le\frac{1}{\mu}\log\frac{1}{q}.
\]

Next, suppose that $q<q_\star$. If the second case in \Cref{eq:simplify-phibar-2} applies, then
\[
    \mathcal{K}(q,R,\mu)=\frac{1}{\mu}\log\left(1+\frac{\gamma_qx^2}{4}\right)\le\frac{1}{\mu}\log(1+\gamma_q)=\frac{1}{\mu}\log\frac{1}{q}.
\]
If on the other hand the first case in \Cref{eq:simplify-phibar-2} applies, then $x\sqrt{\gamma_q}\le2\sqrt{x_\star}$. Using the definitions of $\kappa_\star$ and $x_\star$,
\[
    2\kappa_\star\sqrt{x_\star}
    =\frac{2x_\star}{1+x_\star}
    =\log(1+x_\star)
    =\log\frac{1}{q_\star}
    \le\log\frac{1}{q},
\]
we again obtain
\[
    \mathcal{K}(q,R,\mu)
    =\kappa_\star R \sqrt{\gamma_q}
    =\frac{\kappa_\star}{\mu}x\sqrt{\gamma_q}
    \le\frac{1}{\mu}\, 2\kappa_\star \sqrt{x_\star}
    \le\frac{1}{\mu}\log\frac{1}{q}.
\]

\paragraph{Putting together the pieces.} Choose $C:=\max\{C_{\mathrm I},C_{\mathrm{II}},1\}$. Combining the three upper bounds on $\mathcal{K}(q,R,\mu)$ obtained from regimes I, II, III yields
\[
    \frac{1}{\mu}\overline\phi_q\Bigl(\mfrac{R\mu}{2}\Bigr)
    =
    \mathcal{K}(q,R,\mu)
    \le
    \begin{cases}
        \,\displaystyle C\frac{R}{\sqrt{q} }\log(2-q), & \text{if}\,~ R\mu\le\sqrt{q},\\ 
        \,\displaystyle \frac{C}{\mu}\log\left(1+\frac{1-q}{q} \, R^2\mu^2\right), & \text{if}\,~ \sqrt{q}<R\mu\le1,\\ 
        \,\displaystyle \frac{C}{\mu}\log\frac{1}{q}, & \text{if}\,~ R\mu>1.
    \end{cases}
\]
Substituting this bound into \Cref{cor:sgd-rate}, and using the inequality $e^R + 1 \le 2e^R$ proves the desired result.
\qed

\noindent It remains to prove  Claim~\ref{claim:xi-D-inc}.

\paragraph{Proof of \Cref{claim:xi-D-inc}.}
Clearly, if $D\le R/2$, then $\xi_q(D)=D\,\overline\phi_q(1)$,
which is non-decreasing in $D$. 
If $D\ge R/2$, denoting $u := R/(2D)\in(0,1]$, we obtain
\[
    \xi_q(D)=\frac{R}{2}\frac{\overline\phi_q(u)}{u}.
\]
It thus suffices to show that the map $u\mapsto\overline\phi_q(u)/u$ is non-increasing on $(0,1]$. Take $0 < u < v \leq 1$ and note that by concavity of $\overline\phi_q$ on $[0,1]$ and the fact that $\overline\phi_q(0) = 0$, we have
\begin{align*}
    \overline\phi_q(u) 
    = \overline\phi_q\Bigl( \mfrac{u}{v} \cdot v + \bigl(1 - \mfrac{u}{v}\bigr) \cdot 0 \Bigr)
    \geq \mfrac{u}{v} \cdot \overline\phi_q(v) + \biggl(1 - \mfrac{u}{v}\biggr) \cdot \overline\phi_q(0) 
    = \frac{u}{v} \, \overline\phi_q(v),
\end{align*}
which proves that $\overline\phi_q(u)/u$ is non-increasing in $u$ on $(0,1]$. 
\qed

\section{Proofs of supporting lemmas associated with the lower bound}
\label{appsec:proof-lower}

In this section, we collect the deferred proofs of the lemmas stated in \Cref{sec:proof-lower}. \Cref{lem:baseline-assouad-divergence} is used in the proof of \Cref{prop:baseline-assouad}, while the rest are used in the proof of \Cref{prop:regime-dependent-lower}.

\subsection{Proof of \Cref{lem:baseline-assouad-divergence}}
\label{proof:lem:baseline-assouad-divergence}

Let $\widetilde Q_\omega$ denote the joint law of $(K,Y,M,Z)$ under $P_\omega$. Since $Q_\omega$ is the marginal law of $(Y,M,Z)$, the data-processing inequality yields
\begin{equation}
\label{eq:baseline-dpi-step}
    \chi^2\bigl(Q_{\omega^{(k)}},Q_\omega\bigr)
    \le
    \chi^2\Bigl(\widetilde Q_{\omega^{(k)}},\widetilde Q_\omega\Bigr).
\end{equation}

The distribution of $K$ is uniform on $[p]$ under both models. Conditional on $K=\ell\neq k$, we have $\omega_\ell = \omega^{(k)}_\ell$, so the conditional laws of $(Y,M,Z)$ coincide. Hence, the two models can differ only on the event $\{K=k\}$, which occurs with probability $1/p$.
Also denote $D:=M_{K+1}$ for convenience; this is the indicator that $B$ is observed.
Since $D\sim\BER(q)$ under both models and also $D \indep K$, we obtain
\begin{align}
&\chi^2\Bigl(\widetilde Q_{\omega^{(k)}},\widetilde Q_\omega\Bigr) \nonumber\\ 
&=
\sum_{m=0, 1} \sum_{\ell=1}^p
\mathbb{P}(K=\ell, D= m) \cdot \,
\chi^2\Bigl(
    \mathsf{Law}_{\omega^{(k)}}(Y,Z,M\mid K=\ell, D=m),
    \mathsf{Law}_\omega(Y,Z,M\mid K=\ell, D=m)
\Bigr) \nonumber\\ 
&\overset{\1}{=} 
\sum_{m = 0, 1} \sum_{\ell=1}^p
\mathbb{P}(K=\ell)\, \mathbb{P}(D= m) \cdot\,
\chi^2\Bigl(
    \mathsf{Law}_{\omega^{(k)}}(Y,Z,M\mid K=\ell, D=m),
    \mathsf{Law}_\omega(Y,Z,M\mid K=\ell, D=m)
\Bigr) \nonumber\\ 
&\overset{\2}{=}
\frac{q}{p}\cdot
\chi^2\Bigl(
    \mathsf{Law}_{\omega^{(k)}}(Y,Z,M\mid K=k, D=1),
    \mathsf{Law}_\omega(Y,Z,M\mid K=k, D=1)
\Bigr) \nonumber \\
&\qquad\qquad + \frac{1 - q}{p}\cdot
\chi^2\Bigl(
    \mathsf{Law}_{\omega^{(k)}}(Y,Z,M\mid K=k, D=0),
    \mathsf{Law}_\omega(Y,Z,M\mid K=k, D=0)
\Bigr),
\label{eq:baseline-chisq-UB-1}
\end{align}
where we use the independence of $D$ and $K$ in step $\1$, and use $\P(K=k) = 1/p$ and $\P(D=1) = q = 1 - \P(D=0)$ in step $\2$.

It remains to compute the $\chi^2$--divergence terms in \Cref{eq:baseline-chisq-UB-1}. First note that the covariate distribution is the same under the two laws, so it suffices to analyze the response distribution. We begin by conditioning on $K=k$. By \Cref{eq:baseline-X,eq:baseline-theta-omega}
\begin{align}
\label{eq:baseline-linear-predictor}
    \frac{1}{2}X^\top\theta_\omega
    &=\frac{1}{2}\left\{
        \frac{1}{\sqrt{2}}\frac{R}{\sqrt{2}}
        +
        \frac{B}{\sqrt{2}}\,2\sqrt{2}\,\alpha\omega_k
    \right\}
    =\frac{R}{4}+\alpha B\omega_k.
\end{align}
Write $R':=R/4$ and recall that $D=M_{k+1}$ on the event $\{K = k\}$. By \Cref{eq:baseline-linear-predictor},
\[
    \frac{1}{2}X^\top\theta_\omega=R'+\alpha B\omega_k,
    \qquad\text{and}\qquad
    \frac{1}{2}X^\top\theta_{\omega^{(k)}}=R'-\alpha B\omega_k.
\]

If $D=0$, then $B$ is unobserved. Since $B$ is a uniform Rademacher random variable under both models, marginalizing over $B$ gives the same distribution of $Y$. Thus, the conditional divergence is zero when $D=0$.

If $D=1$, then $B$ is observed through $Z_{k+1}=B/\sqrt{2}$. Recall that, since $Y\in\{-1,1\}$, the logistic model can be written as
\[
    \mathbb{P}_\theta(Y=y\mid X)
    =
    \frac{\exp\left\{yX^\top\theta/2\right\}}
    {2\cosh\left(X^\top\theta/2\right)},
    \qquad\text{for}\,~ y\in\{-1,1\}.
\]

Now for $a\in\mathbb{R}$, define
\[
    \pi_a(y):=\frac{\exp(ay)}{2\cosh(a)},
    \qquad\text{for}\,~ y\in\{-1,1\},
\]
so that $\pi_a$ is the conditional law of $Y$ whenever $X^\top\theta/2=a$. Conditional on $K=k$, $D=1$, and $B$, the conditional laws of $Y$ under $P_{\omega^{(k)}}$ and $P_\omega$ are therefore $\pi_{R'-\alpha B\omega_k}$ and $\pi_{R'+\alpha B\omega_k}$, respectively.

We first compute the divergence between a generic pair of laws having probability mass functions $\pi_{a+s}$ and $\pi_{a-s}$. For every $a,s\in\mathbb{R}$ and $y\in\{-1,1\}$, direct calculation shows that
\begin{align*}
    \frac{\pi_{a+s}(y)^2}{\pi_{a-s}(y)}
    =
    \frac{\cosh(a-s)}{2\cosh^2(a+s)}
    \exp\left\{y(a+3s)\right\}.
\end{align*}
Therefore,
\begin{align}
    \chi^2\bigl(\pi_{a+s},\pi_{a-s}\bigr)
    &=\sum_{y\in\{-1,1\}}\frac{\pi_{a+s}(y)^2}{\pi_{a-s}(y)} - 1 \nonumber\\
    &=\frac{\cosh(a-s)}{2\cosh^2(a+s)}
    \sum_{y\in\{-1,1\}}\exp\left\{y(a+3s)\right\} - 1 \nonumber\\
    &=\frac{\cosh(a-s)\cosh(a+3s)}{\cosh^2(a+s)} - 1 \nonumber\\
    &\overset{\1}{=} \frac{\cosh^2(a+s)+\sinh^2(2s)}{\cosh^2(a+s)} - 1
    =\frac{\sinh^2(2s)}{\cosh^2(a+s)},
\label{eq:baseline-chisq-pmf}
\end{align}
where in step $\1$, we have used the fact that  $\cosh(u-v)\cosh(u+v)=\cosh^2(u)+\sinh^2(v)$, with $u = a+s$ and $v = 2s$.

Since $B$ is observed when $D=1$ and is uniform on $\{-1,+1\}$ under both models, applying the identity~\eqref{eq:baseline-chisq-pmf} with $a=R'$ and $s=-\alpha B\omega_k$, and averaging over $B$, yields
\begin{align}
\label{eq:baseline-chisq-cond}
\begin{aligned}
    &\chi^2\Bigl(
        \mathsf{Law}_{\omega^{(k)}}(Y,Z,M\mid K=k,D=1),\,
        \mathsf{Law}_{\omega}(Y,Z,M\mid K=k,D=1)
    \Bigr)\\ 
    &\hspace{5cm}=
    \frac{\sinh^2(2\alpha)}{2}
    \left\{
        \frac{1}{\cosh^2(R'+\alpha)}
        +
        \frac{1}{\cosh^2(R'-\alpha)}
    \right\}.
\end{aligned}
\end{align}

Continuing from \Cref{eq:baseline-chisq-UB-1}, and using \Cref{eq:baseline-chisq-cond}, we obtain
\[
\chi^2\Bigl(\widetilde Q_{\omega^{(k)}},\widetilde Q_\omega\Bigr)
=
\frac{q}{p}\,
\frac{\sinh^2(2\alpha)}{2}
\left\{
    \frac{1}{\cosh^2(R'+\alpha)} + \frac{1}{\cosh^2(R'-\alpha)}
\right\}.
\]

Next, under the assumption $\alpha\le R/24=R'/6$, we have $R'-\alpha \ge \nicefrac{5R'}{6} > 0$. Since $x \mapsto \cosh^2(x)$ is strictly increasing for $x > 0$, we have $\cosh^2(R'+\alpha) \ge \cosh^2(R'-\alpha)$.
Putting together the pieces, we obtain
\begin{align*}
\chi^2\bigl(Q_{\omega^{(k)}}, Q_\omega\bigr)
&\le 
\chi^2\Bigl(\widetilde Q_{\omega^{(k)}},\widetilde Q_\omega\Bigr) \\ 
&\le
\frac{q}{p}\,
\frac{\sinh^2(2\alpha)}{\cosh^2(R'-\alpha)} \\ 
&\overset{\1}{\le}
\frac{16q\alpha^2}{p}\,e^{-2R'+6\alpha}
\overset{\2}{\le}
\frac{16q\alpha^2}{p}\,e^{-R'}
=
\frac{16q\alpha^2}{p}\,\exp(-R/4),
\end{align*}
where in step $\1$ we have used $\sinh(2\alpha)\le2\alpha e^{2\alpha}$ and $\cosh(R'-\alpha)\ge e^{R'-\alpha}/2$, and in step $\2$ we have used the fact that $6\alpha\le R'$.
This proves \Cref{eq:baseline-edge-chi2-bound}.
Finally, by tensorization of the $\chi^2$-divergence,
\[
    1+\chi^2\Bigl(Q_{\omega^{(k)}}^{\otimes n},Q_\omega^{\otimes n}\Bigr)
    =
    \bigl\{
        1+\chi^2\left(Q_{\omega^{(k)}},Q_\omega\right)
    \bigr\}^n
    \le
    \exp\left\{
        \frac{16nq\alpha^2}{p}\exp(-R/4)
    \right\}.
\]
Subtracting $1$ from both sides yields \Cref{eq:baseline-edge-chi2-n}, as claimed.
\qed

\subsection{Proof of \Cref{lem:two-point-validity}}
\label{proof:lem:two-point-validity}

We prove the lemma part-by-part.

\subsubsection{Proof of \Cref{lem:two-point-validity}\emph{(a)}}
We begin by showing that $J$ is an odd function.  To this end, note that since the product function is taken over indices with odd $|U|$, we have
\[
J(-t) 
= \sum_{\substack{U \subseteq [m]\\|U|\ \mathrm{odd}}} g_U\rho^{-|U|} \chi_U(-t)
= \sum_{\substack{U \subseteq [m]\\|U|\ \mathrm{odd}}} g_U\rho^{-|U|} \chi_U(t) \cdot (-1)^{|U|}
= -\sum_{\substack{U \subseteq [m]\\|U|\ \mathrm{odd}}} g_U\rho^{-|U|} \chi_U(t)
= -J(t).
\]

We next verify the expectation of $J(T)$ under the null distribution.  The independence of $T_1,\ldots,T_m$ and $\mathbb{E}_0[T_j]=\rho$ imply $\mathbb{E}_0\bigl[\chi_U(T)\bigr] = \rho^{|U|}$.
Therefore,
\[
    \mathbb{E}_0\bigl[J(T)\bigr]
    =
    \sum_{\substack{U \subseteq [m]\\|U|\ \mathrm{odd}}}
    g_U
    =
    f(\boldsymbol{1}_m)
    =
    1.
\]

We next consider the conditional expectation $\mathbb{E}[J(T) \mid T_S]$.  Fix $S\subseteq[m]$ with $|S|\leq r$. For every $U \subseteq [m]$, observe that
\begin{equation*}
    \mathbb{E}_0\bigl[\chi_U(T) \,\big|\, T_S\bigr]
    =
    \chi_{U\cap S}(T_S) \cdot \rho^{|U\setminus S|}.
\end{equation*}
Therefore, 
\begin{align*}
    \mathbb{E}_0\bigl[J(T) \,\big|\, T_S\bigr]
    &=
    \sum_{\substack{U \subseteq [m]\\|U|\ \mathrm{odd}}}
    g_U\rho^{-|U|}
    \chi_{U\cap S}(T_S) \cdot \rho^{|U\setminus S|} \\ 
    &\overset{\1}{=}
    \sum_{C\subseteq S}
    \rho^{-|C|}
    \chi_C(T_S)
    \sum_{\substack{
        U \subseteq [m]\\
        |U|\ \mathrm{odd}\\
        U\cap S=C
    }} g_U \\ 
    &\overset{\2}{=} 
    \sum_{C\subseteq S}
    \rho^{-|C|}
    \chi_C(T_S) \cdot \mathbbm{1}\{C=\varnothing\} 
    ~=~ \rho^{0} \chi_\varnothing(T_S) + 0 
    ~=~ 1,
\end{align*}
where step $\1$ uses a change of summation indices from $U \subseteq [m]$ to $C := U \cap S \subseteq S$ and step $\2$ follows from~\Cref{lem:boolean-cancellation}.
This proves our claimed result that $\mathbb{E}_0\bigl[J(T) \,\big|\, T_S\bigr] = 1$.

Finally, we prove the upper bound on $\|J\|_{\infty}$.  Note that for every $t\in\{-1,1\}^m$, the product function $\chi_U(\cdot)$ satisfies $|\chi_U(t)| \le 1$. We apply this fact in conjunction with the triangle inequality to obtain
\begin{align*}
    |J(t)|
    = \left| \sum_{\substack{U \subseteq [m]\\|U|\ \mathrm{odd}}}
    g_U\rho^{-|U|}\chi_U(t) \right| 
    &\leq \sum_{\substack{U \subseteq [m]\\|U|\ \mathrm{odd}}}
    \left|g_U\right| \, \rho^{-|U|}.
\end{align*}
Applying the Cauchy--Schwarz inequality then yields
\begin{align*}
    |J(t)|
    &\le
    \Biggl\{
        \sum_{\substack{U \subseteq [m]\\|U|\ \mathrm{odd}}} g_U^2
    \Biggr\}^{1/2}
    \Biggl\{
        \sum_{\substack{U \subseteq [m]\\|U|\ \mathrm{odd}}} \rho^{-2|U|}
    \Biggr\}^{1/2}\le
    \Biggl\{
        \sum_{U \subseteq [m]}g_U^2
    \Biggr\}^{1/2}
    \Biggl\{
        \sum_{U \subseteq [m]}\rho^{-2|U|}
    \Biggr\}^{1/2},
    \\
    &\overset{\1}{=} 
    1 \cdot \left[ \sum_{n=0}^m \binom{m}{n} \left(\rho^{-2}\right)^n \right]^{1/2} 
    \!\!= \left(1+\rho^{-2}\right)^{m/2}
    \!\!,
\end{align*}
where the penultimate  inequality follows since $\sum_{U \subseteq [m]} g_U^2 = 1$ (see Lemma~\ref{lem:majority-coefficients}) and the final inequality follows from the binomial formula.  Continuing, we note that since by definition, $\rho = \tanh(\lambda)$ (see \Cref{eq:def-rho}),
\begin{align*}
    1+\rho^{-2}
    =
    1 + \frac{\cosh^2(\lambda)}{\sinh^2{\lambda}}
    =
    \frac{\sinh^2(\lambda) + \cosh^2(\lambda)}{\sinh^2(\lambda)}
    =
    \frac{\sinh^2(\lambda) + 1 + \sinh^2(\lambda)}{\sinh^2(\lambda)}
    =
    2 + \frac{1}{\sinh^2(\lambda)}.
\end{align*}
Combining the previous displays then yields the bound
\[
\lvert J(t) \rvert \leq \left(1+\rho^{-2}\right)^{m/2}
    =
    \biggl( 2+\frac{1}{\sinh^2(\lambda)} \biggr)^{m/2}
\]
Since this holds for all $t\in\{-1,1\}^m$, the same upper bound holds for $\|J\|_{\infty}$.
\qed

\subsubsection{Proof of \Cref{lem:two-point-validity}\emph{(b)}}  We prove each part in turn, beginning with the validity of the covariate distribution $P_{1, X}$.  

\paragraph{Validity of the covariate distribution $P_{1,X}$.} We show that \Cref{eq:two-point-P1X} defines a valid probability distribution $P_{1,X}$ under the alternative. Write
\[
    F_\alpha(a,b) := \cosh(\alpha)-b\sinh(\alpha) \cdot J(a).
\]
Since $\alpha\|J\|_\infty\le\nicefrac12$ and $\tanh(\alpha)\le\alpha$, we deduce that
\[
    F_\alpha(a,b)
    =
    \cosh(\alpha) \bigl\{1-b\tanh(\alpha) \cdot J(a)\bigr\}
    \ge
    \cosh(\alpha) \bigl\{1-\tanh(\alpha) \cdot \|J\|_\infty\bigr\}
    \ge
    \frac{1}{2} \cosh(\alpha)>0.
\]
The ratio of hyperbolic cosines in \Cref{eq:two-point-P1X} is also strictly positive, so the proposed Radon--Nikodym derivative is positive.

It remains to verify that $\mathrm{d} P_{1, X}/\mathrm{d} P_{0, X}$ integrates to one.  Expanding yields
\begin{align*}
    \E_{0,X}\left[
        \frac{\mathrm{d}P_{1,X}}{\mathrm{d}P_{0,X}}(X)
    \right]
    &=
    \E_0\Bigl[F_\alpha(A,B)e^{\alpha YB}\Bigr] = \mathbb{E}_0\Bigl[F_\alpha(A,B) \mathbb{E}\Bigl[e^{\alpha YB} \mid A, B\Bigr]\Bigr].
\end{align*}
From \Cref{eq:label-distribution-null}, we compute
\[
    \E_0\bigl[e^{\alpha YB}\mid A=a,B=b\bigr]
    =
    \sum_{y\in\{-1,1\}}
    \frac{e^{ys(a)}}{2\cosh\{s(a)\}}e^{\alpha yb}
    =
    \frac{\cosh\{s(a)+\alpha b\}}{\cosh\{s(a)\}}.
\]
Combining the previous two displays thus yields
\[
    \E_{0,X}\left[
        \frac{\mathrm{d}P_{1,X}}{\mathrm{d}P_{0,X}}(X)
    \right]
    =
    \E_0\Bigl[F_\alpha(A,B)e^{\alpha YB}\Bigr].
\]
Let $V:=YB$. Under $P_0$, $V$ is a uniform Rademacher random variable independent of $T$. Moreover, since $A=YT$ and $J$ is odd by part (a),
\[
    J(A)=J(YT)=Y \cdot J(T),
    \qquad\text{so that}\qquad
    B \cdot J(A) = V \cdot J(T).
\]
Hence,
\[
    F_\alpha(A,B)
    =
    \cosh(\alpha)-V\sinh(\alpha) \cdot J(T).
\]
Using the independence of $V$ and $T$, together with $\E_0[J(T)]=1$ from part (a), we obtain
\begin{align*}
    \E_0\bigl[F_\alpha(A,B) \cdot e^{\alpha YB}\bigr]
    &=
    \cosh(\alpha)\E_0[e^{\alpha V}]
    -
    \sinh(\alpha)\E_0[Ve^{\alpha V}]\,\E_0[J(T)]\\
    &=
    \cosh^2(\alpha)-\sinh^2(\alpha)
    =
    1.
\end{align*}
Thus, \Cref{eq:two-point-P1X} indeed defines a valid probability distribution $P_{1,X}$.

\paragraph{Well-specified logistic model under the null.}
 Fix $a=(a_1,\ldots,a_m)\in\{-1,1\}^m$, $b\in\{-1,1\}$, and $y\in\{-1,1\}$. Since $A_j=YT_j$, the event $\{Y=y,A=a\}$ is equivalent to $\{Y=y,T_j=ya_j\text{ for every }j\in[m]\}$. By the construction of $P_0$,
\[
    \bbP_0(Y=y,A=a,B=b)
    =
    \frac{1}{4}
    \prod_{j=1}^m
    \frac{\exp(\lambda ya_j)}{2\cosh(\lambda)}
    =
    \frac{1}{4} \frac{\exp\{ys(a)\}}{\{2\cosh(\lambda)\}^m}.
\]
Summing over $y\in\{-1,1\}$ gives
\[
    \bbP_0(A=a,B=b)
    =
    \frac{\cosh\{s(a)\}}
    {2\{2\cosh(\lambda)\}^m}.
\]
Consequently,
\[
    \bbP_0(Y=y\mid X=x(a,b))
    =
    \bbP_0(Y=y\mid A=a,B=b)
    =
    \frac{\exp\{ys(a)\}}{2\cosh\{s(a)\}}.
\]
Since $\frac{1}{2} x(a,b)^\top\theta_0=s(a)$ by \Cref{eq:s-a-cond}, this is precisely the logistic model with parameter $\theta_0$.

\paragraph{Well-specified logistic model under the alternative.}
We next verify the asserted logistic model under the alternative and the parameter identities. By \Cref{eq:two-point-theta1},
\[
    \frac{1}{2} x(a,b)^\top\theta_1
    =
    s(a)+\alpha b.
\]
Thus, the conditional distribution prescribed in \Cref{eq:two-point-P1-conditional} is exactly
\[
    \bbP_1(Y=y\mid X=x)
    =
    \frac{\exp\{\frac{1}{2} yx^\top\theta_1\}}{2\cosh\{\frac{1}{2} x^\top\theta_1\}},
\]
so $P_1$ is a well-specified logistic regression model with parameter
$\theta_1$.

\paragraph{Complete-data likelihood ratio.}
By \Cref{eq:two-point-P1-conditional,eq:two-point-P1X,eq:generic-two-point-P1X},
\begin{align*}
    \frac{\mathrm{d}P_1}{\mathrm{d}P_0}(a,b,y)
    &=
    \frac{\mathrm{d}P_{1,X}}{\mathrm{d}P_{0,X}}
    \bigl(x(a,b)\bigr) \cdot 
    \frac{
        \mathbb{P}_1(Y=y\mid A=a,B=b)
    }{
        \mathbb{P}_0(Y=y\mid A=a,B=b)
    }
    \\
    &=
    F_\alpha(a,b) \cdot 
    \frac{\cosh\{s(a)+\alpha b\}}{\cosh\{s(a)\}}
    \cdot 
    \frac{\exp\{y(s(a)+\alpha b)\}}{2\cosh\{s(a)+\alpha b\}}
    \,
    \frac{2\cosh\{s(a)\}}{\exp\{ys(a)\}}
    \\
    &=
    F_\alpha(a,b) \cdot \exp(\alpha yb).
\end{align*}

Equivalently, using $V=YB$ and $B \cdot J(A) = V \cdot J(T)$,
\[
    \frac{\mathrm{d}P_1}{\mathrm{d}P_0}
    =
    e^{\alpha V} \bigl\{\cosh(\alpha)-V\sinh(\alpha) \cdot J(T)\bigr\}.
\]
This proves \Cref{eq:complete-data-likelihood-ratio}. 
\qed

\subsubsection{Proof of \Cref{lem:two-point-validity}\emph{(c)}}
First, under both $P_0$ and $P_1$, the covariate vector $X$ is supported on vectors of the form
\[
    X
    =
    \left(
        \frac{A_1}{\sqrt{2m}},\ldots, \frac{A_m}{\sqrt{2m}}, \frac{B}{\sqrt{2}}
    \right)^\top,
    \qquad\text{where}\,~ A_1,\ldots,A_m,B\in\{-1,1\}.
\]
Hence, we always have
\[
    \|X\|^2 = \frac{1}{2m}\sum_{j=1}^m A_j^2+\frac{B^2}{2} = 1,
\]
almost surely under both distributions.

Next, we prove that the second moment matrix is the same under both $P_0$ and $P_1$. Recall that $A_j=YT_j$ and $V=YB$. Since $Y^2 = 1$, we have the identities
\[
    A_j A_k=Y^2 T_j T_k = T_j T_k,
    \qquad\text{and}\qquad
    A_j B = Y A_j \cdot YB = T_j V.
\]
Under $P_0$, $V$ is a uniform Rademacher random variable independent of $T$, and by \Cref{lem:two-point-validity}(b),
\[
    L_\alpha
    :=
    \frac{\mathrm{d}P_1}{\mathrm{d}P_0}
    =
    e^{\alpha V}
    \left\{
        \cosh(\alpha) - V\sinh(\alpha) \cdot J(T)
    \right\}.
\]
Since $V$ is a uniform Rademacher, note that
\begin{align}
\begin{aligned}
\label{eq:cosh-sinh}
    \mathbb{E}[e^{\alpha V}] &= P(V=1) \cdot e^{\alpha(1)} + P(V=-1) \cdot e^{\alpha(-1)} = \frac{e^\alpha + e^{-\alpha}}{2} = \cosh(\alpha) =: c_\alpha, \\
    \mathbb{E}[V e^{\alpha V}] &= P(V=1) \cdot e^{\alpha(1)} - P(V=-1) \cdot e^{\alpha(-1)} = \frac{e^\alpha - e^{-\alpha}}{2} = \sinh(\alpha) =: s_\alpha.
\end{aligned}
\end{align}
Consider $j,k\in[m]$ with $j\ne k$. Using the independence of $V$ and $T$,
\begin{align*}
    \mathbb{E}_1[A_jA_k]
    &=
    \mathbb{E}_0[L_\alpha T_jT_k]
    =
    c_\alpha^2
    \mathbb{E}_0[T_jT_k]
    -
    s_\alpha^2
    \mathbb{E}_0[T_jT_kJ(T)].
\end{align*}
Since $m=2r+1\ge5$, we have $r\ge2$. Thus,
\Cref{eq:required-properties-J} applied with $S=\{j,k\}$ yields
\[
    \mathbb{E}_0[J(T)\mid T_j,T_k]=1.
\]
Consequently,
\[
    \mathbb{E}_0[T_jT_kJ(T)]
    =
    \mathbb{E}_0[T_jT_k],
\]
and hence for all $j,k \in [m]$,
\begin{equation}
\label{eq:second-moment-AiAj}
    \mathbb{E}_1[A_jA_k]
    =
    \left(c_\alpha^2-s_\alpha^2\right)
    \mathbb{E}_0[T_jT_k]
    =
    \mathbb{E}_0[T_jT_k]
    =
    \mathbb{E}_0[A_jA_k].
\end{equation}

Next, for each $j\in[m]$,
\begin{equation}
\label{eq:second-moment-AiB-alt-1}
    \mathbb{E}_1[A_jB]
    =
    \mathbb{E}_0[L_\alpha T_jV]
    =
    c_\alpha s_\alpha
    \mathbb{E}_0[T_j]
    -
    c_\alpha s_\alpha
    \mathbb{E}_0[T_jJ(T)].
\end{equation}
Since $r\ge1$, \Cref{eq:required-properties-J} applied with $S=\{j\}$ yields that $\mathbb{E}_0[J(T)\mid T_j]=1$, and therefore $\mathbb{E}_0[T_jJ(T)] = \mathbb{E}_0[T_j]$. It follows from \Cref{eq:second-moment-AiB-alt-1} that
\begin{equation}
\label{eq:second-moment-AiB-alt-2}
    \mathbb{E}_1[A_jB]
    =
    c_\alpha s_\alpha \mathbb{E}_0[T_j] - c_\alpha s_\alpha \mathbb{E}_0[T_jJ(T)]
    =
    c_\alpha s_\alpha \mathbb{E}_0[T_j] - c_\alpha s_\alpha \mathbb{E}_0[T_j]
    = 0.
\end{equation}
Under $P_0$, $T_j$ and $V$ are independent and $\mathbb{E}_0[V]=0$. Thus, using \Cref{eq:second-moment-AiB-alt-2},
\begin{equation}
\label{eq:second-moment-AiB-null}
    \mathbb{E}_0[A_jB] = \mathbb{E}_0[T_jV] = 0 = \mathbb{E}_1[A_jB].
\end{equation}
Also, $A_j^2=B^2=1$ almost surely under both distributions. 
Combining this with the equalities obtained in \Cref{eq:second-moment-AiAj,eq:second-moment-AiB-null}, we obtain that under $P_0$ and $P_1$, every entry of the second-moment matrix $\E[XX^\top]$ agrees, proving
\[
    \mathbb{E}_1[XX^\top] = \mathbb{E}_0[XX^\top].
\]

It remains to evaluate this common second-moment matrix. For simplicity, we operate under $P_0$. Note that for every $j,k\in[m]$ with $j\neq k$, the independence of $T_j$ and $T_k$ gives
\[
    \mathbb{E}_0[T_jT_k]
    =
    \mathbb{E}_0[T_j] \cdot \mathbb{E}_0[T_k]
    =
    \rho^2,
\]
whereas $\mathbb{E}_0[T_j^2]=\mathbb{E}_0[A_j^2]=1$. Thus,
\begin{equation}
\label{eq:second-moment-AAtop}
    \mathbb{E}_0[AA^\top]
    =
    (1-\rho^2) I_m + \rho^2\boldsymbol{1}_m\boldsymbol{1}_m^\top.
\end{equation}
Recalling that $X = \Bigl(\frac{A^\top}{\sqrt{2m}}, \frac{B}{\sqrt{2}}\Bigr)^\top$, we therefore obtain using \Cref{eq:second-moment-AAtop,eq:second-moment-AiB-null} that
\[
    \mathbb{E}_0[XX^\top]
    =
    \mathbb{E}_1[XX^\top]
    =
    \frac{1}{2}
    \begin{bmatrix}
        \mfrac{1}{m} \mathbb{E}_0[AA^\top] & \mfrac{1}{\sqrt{m}} \mathbb{E}_0[BA]\\
        \mfrac{1}{\sqrt{m}} \mathbb{E}_0[BA^\top] & \mathbb{E}_0[B^2]
    \end{bmatrix}
    =
    \frac{1}{2}
    \begin{bmatrix}
        \frac{1}{m}
        \bigl\{
        (1-\rho^2)I_m + \rho^2\boldsymbol{1}_m\boldsymbol{1}_m^\top \bigr\} & 0\\
        0 & 1
    \end{bmatrix}.
\]
Lastly, since $(1-\rho^2)I_m + \rho^2\boldsymbol{1}_m\boldsymbol{1}_m^\top \succeq (1-\rho^2)I_m$, and $1-\rho^2=\operatorname{sech}^2(\lambda)$, it follows that
\[
    \mathbb{E}_0[XX^\top]
    =
    \mathbb{E}_1[XX^\top]
    \succeq
    \frac{\operatorname{sech}^2(\lambda)}{2m}I_{m+1},
\]
which completes the proof of part (c). 
\qed

\subsection{Proof of \Cref{lem:observed-likelihood-ratio-exact}} 
\label{proof:lem:observed-likelihood-ratio-exact}

Since $Y$ is always observed and $T_j=YA_j$, observing $A_j$ is equivalent to observing $T_j$. Similarly, when $D=1$ (which we remind the reader is shorthand for the indicator $M_{m + 1}$), observing $B$ is equivalent to observing $V=YB$. Consequently, the $\sigma$-field generated by the observations may be written as
\begin{equation}
\label{eq:observed-sigma-field}
    \mathcal{O} := \sigma(Y,M,Z) = \sigma(Y,M,T_S,DV).
\end{equation}
Before establishing the desired property, we make the following claim, which ensures that the Radon--Nikodym derivative $\mathrm{d}Q_1/\mathrm{d}Q_0$ is well-defined, and is equal to the conditional expectation of the complete-data likelihood ratio $L_\alpha$ given $\mcO$.
\begin{claim}
\label{lem:observed-likelihood-ratio-general}
The distribution $Q_1$ is absolutely continuous with respect to $Q_0$, and
\[
    \frac{\mathrm{d}Q_1}{\mathrm{d}Q_0}(Y,M,Z)
    =
    \mathbb{E}_0 \big[L_\alpha\mid\mathcal{O}\big],
    \qquad
    Q_0\text{-almost surely}.
\]
\end{claim}

Deferring the proof of this claim to the end of the section, we apply it to write
\begin{equation}
\label{eq:observed-LR}
\widetilde{L}_{\alpha} = \mathbb{E}_0[L_{\alpha} \mid \mathcal{O}] = \mathbb{E}_0 \biggl[ \frac{\mathrm{d}P_1}{\mathrm{d} P_0} \; \Big \vert \; \mathcal{O}\biggr] = \mathbb{E}_0\bigl[\exp(\alpha V) \bigl\{\cosh(\alpha) - V\sinh(\alpha) \cdot J(T)\bigr\} \; \big \vert \; \mathcal{O}\bigr],
\end{equation}
where in the final step, we applied \Cref{lem:two-point-validity}(b) to expand the likelihood-ratio $\mathrm{d}P_1/\mathrm{d}P_0$.

Note that under $P_0$, the vector $T$ is conditionally independent of $V$ given $\mathcal{O}$, so that
\[
\mathbb{E}_0\bigl[\exp(\alpha V)\,  V\sinh(\alpha) \cdot J(T) \; \big \vert \; \mathcal{O}\bigr] = \sinh(\alpha) \cdot \mathbb{E}_0\bigl[V\exp(\alpha V)  \; \big \vert \; \mathcal{O}\bigr] \cdot \mathbb{E}_0 \bigl[J(T) \; \big \vert \; \mathcal{O}\bigr].
\]
We next simplify each of the terms in the product above in turn, beginning with the last term.  To this end, note that under $P_0$, the vector $T$ is independent of the tuple $(Y, V, M)$.  It follows that
\[
    \mathbb{E}_0[J(T)\mid\mathcal{O}]
    =
    \mathbb{E}_0[J(T)\mid S,T_S]
    =
    J_S(T_S).
\]

To compute $\mathbb{E}_0\bigl[V\exp(\alpha V)  \; \big \vert \; \mathcal{O}\bigr]$ we consider the case in which $D = 1$ and $D = 0$ separately.  If $D = 1$, then $V$ is $\mathcal{O}$-measurable and consequently \Cref{eq:observed-LR} yields
\begin{align*}
    \widetilde{L}_\alpha
    &=
    \exp(\alpha V)
    \bigl\{\cosh(\alpha) - V\sinh(\alpha) \cdot \mathbb{E}_0[J(T)\mid\mathcal{O}]\bigr\}
    \\
    &=
    \exp(\alpha V)
    \bigl\{\cosh(\alpha) - V\sinh(\alpha) \cdot J_S(T_S)\bigr\}.
\end{align*}

On the other hand, if $D=0$, $V$ is not observed; thus, $\mathcal{O} = \sigma(Y,M,T_S)$ as $DV = 0$, so $V \indep \mathcal{O}$ as  $\mathcal{O}$ carries no information about $V$. Therefore, since $V$ is a uniform Rademacher random variable independent of $(Y,T,M)$, it is independent of $\mathcal{O}$ on the event $\{D=0\}$. Hence,
\[
    \mathbb{E}_0[\exp(\alpha V)\mid\mathcal{O}]
    =
    \mathbb{E}_0[\exp(\alpha V)]
    =
    \cosh(\alpha)
\]
and
\[
    \mathbb{E}_0[V\exp(\alpha V) \mid\mathcal{O}]
    =
    \mathbb{E}_0[V\exp(\alpha V)  \mid T_S]
    =
    \sinh(\alpha),
\]
by the same calculation as in \Cref{eq:cosh-sinh}.  Combining the pieces, it follows that
\begin{align*}
    \widetilde{L}_\alpha
    &=
    \cosh^2(\alpha) - \sinh^2(\alpha) \cdot J_S(T_S)
    =
    1+\sinh^2(\alpha)
    \left\{
        1-J_S(T_S)
    \right\}.
\end{align*}

Finally, if $|S|\leq r$, then \Cref{eq:required-properties-J} gives $J_S(T_S)=1$. For $D=0$, the above display immediately yields $\widetilde{L}_\alpha=1$. For $D=1$,
\[
    \widetilde{L}_\alpha
    =
    \exp(\alpha V)
    \left\{
        \cosh(\alpha)-V\sinh(\alpha)
    \right\}
    =
    \exp(\alpha V)\exp(-\alpha V)
    =
    1,
\]
as desired. \hfill $\qed$

It remains to establish Claim~\ref{lem:observed-likelihood-ratio-general}.

\paragraph{Proof of \Cref{lem:observed-likelihood-ratio-general}.}
Let $g=g(Y,M,Z)$ be any bounded and measurable function of the observed data. Since $M$ is independent of $(X,Y)$ and has the same distribution (say, $P_M$) under both hypotheses, the likelihood ratio of the joint law of $(X,Y,M)$ under the alternative with respect to the null is $L_\alpha(X,Y)$. Hence
\begin{align*}
    \int g \, \mathrm{d}Q_1 
    = 
    \mathbb{E}_{Q_1}[g]
    &=
    \int g(y,m,m\odot x)\, P_1(\mathrm dx,\mathrm dy)\, P_M(\mathrm dm)
    \\
    &=
    \int g(y,m,m\odot x)\, L_\alpha(x,y)\, P_0(\mathrm dx,\mathrm dy)\, P_M(\mathrm dm)
    = \mathbb{E}_0\!\bigl[L_\alpha \cdot g(Y,M,Z)\bigr].
\end{align*}
Since $g(Y,M,Z)$ is $\mathcal{O}$-measurable, the tower property yields
\begin{align*}
    \int g \, \mathrm{d}Q_1 
    &=
    \mathbb{E}_0\!\Bigl[
        \mathbb{E}_0\bigl[L_\alpha \cdot g(Y,M,Z) \,\big|\, \mathcal{O}\bigr]
    \Bigr] 
    =
    \mathbb{E}_0 \!\Bigl[
        \mathbb{E}_0\bigl[L_\alpha \big| \mathcal{O}\bigr] \cdot g
    \Bigr] 
    =
    \int \mathbb{E}_0\bigl[L_\alpha \big| \mathcal{O}\bigr] \cdot g \, \mathrm{d}Q_0.
\end{align*}
As this identity holds for every bounded and measurable $g$, it follows that $Q_1\ll Q_0$ and
\[
    \frac{\mathrm{d}Q_1}{\mathrm{d}Q_0}(Y,M,Z)
    =
    \mathbb{E}_0 \bigl[L_\alpha\mid\mathcal{O}\bigr],
\]
$Q_0$-almost surely, as desired. \hfill \qed

\subsection{Proof of \Cref{lem:two-point-divergence}}
\label{proof:lem:two-point-divergence}

We note that the likelihood-ratio calculations in the proofs of \Cref{lem:two-point-validity}(b) and \Cref{lem:observed-likelihood-ratio-exact} do not use the restriction \(m\geq 5\), and remain valid verbatim for every odd \(m\geq 1\) under the condition \(\alpha\|J\|_\infty\leq 1/2\).

Our calculation hinges on the following lemma which establishes independence of the random variable $V$ and the random vector $T$.  We defer its proof to Section~\ref{sec:proof-lem-V-independence}.
\begin{lemma}
\label{lem:V-independence}
Recall the definition of $V$ from \Cref{eq:def-V}. Under $P_0$, $V$ is a uniform Rademacher random variable independent of ~$(Y,T_1,\ldots,T_m)$. 
\end{lemma}
We invoke the tower property of expectation in conjunction with the missing completely at random nature of the missingness to rewrite
\begin{align}
\label{eq:chi-square-expansion-two-point-divergence}
    \chi^2(Q_1, Q_0) = \mathbb{E}_0\bigl[(\widetilde{L}_{\alpha} -1)^2\bigr] = q \,\mathbb{E}_0\bigl[(\widetilde{L}_{\alpha} -1)^2 \mid D = 1\bigr] + (1 - q) \mathbb{E}_0\bigl[(\widetilde{L}_{\alpha} -1)^2 \mid D = 0\bigr].
\end{align}
Applying Lemma~\ref{lem:observed-likelihood-ratio-exact} then yields the pair
\begin{align*}
    \widetilde{L}_{\alpha} &= \exp(\alpha V)
        \left\{
            \cosh(\alpha) - V \sinh(\alpha) \cdot J_S(T_S)
        \right\} &\text{ if }\,~ D = 1,\\
    \widetilde{L}_{\alpha} &= \cosh^2(\alpha) - \sinh^2(\alpha) \cdot J_S(T_S) &\text{ if }\,~ D = 0.
\end{align*}
Using the shorthand
\[
    \Delta_U := J_U(T_U)-1 = \mathbb{E}_0\bigl[J(T)\mid T_U\bigr] - 1,
\]
we thus deduce that
\begin{align*}
    \mathbb{E}_0\bigl[(\widetilde{L}_{\alpha} -1)^2 \mid D = 1\bigr] &= \mathbb{E}_0\bigl[\bigl(\exp(\alpha V)
        \left\{
            \cosh(\alpha) - V \sinh(\alpha) \cdot J_S(T_S)
        \right\} -1\bigr)^2 \mid D = 1\bigr]\\
        &= \mathbb{E}_0\bigl[\bigl(e^{\alpha V}
    \left\{
        \cosh(\alpha)
        -
        V\sinh(\alpha)(1+\Delta_S)
    \right\} -1\bigr)^2 \mid D = 1\bigr]\\
    &\overset{\1}{=} \mathbb{E}_0\bigl[\bigl(V e^{\alpha V}\sinh(\alpha) \cdot \Delta_S\bigr)^2 \mid D = 1\bigr]\\
    &\overset{\2}{=} \sinh^2(\alpha) \cdot \mathbb{E}_0[e^{2\alpha V}] \cdot \mathbb{E}_0[\Delta_S^2],
\end{align*}
where step $\1$ follows from the identity $
    e^{\alpha V}
    \left\{
        \cosh(\alpha)-V\sinh(\alpha)
    \right\}
    =
    1
$, and step $\2$ follows since by \Cref{lem:V-independence}, $V$ is independent of the pair $(S, T_S)$.

On the other hand,
\[
    \widetilde{L}_\alpha-1
    =
    -\sinh^2(\alpha) \cdot \Delta_S \quad \text{ if } \quad D = 0.
\]
Substituting these expressions into \Cref{eq:chi-square-expansion-two-point-divergence} yields
\begin{align}
    \chi^2(Q_1,Q_0)
    =
    \mathbb{E}_0
    \Bigl[
        \bigl(\widetilde{L}_\alpha-1\bigr)^2
    \Bigr]
    &=
    q\sinh^2(\alpha) \cdot \mathbb{E}_0[e^{2\alpha V}] \cdot \mathbb{E}_0[\Delta_S^2]
    +
    (1-q)\sinh^4(\alpha) \cdot \mathbb{E}_0[\Delta_S^2]\nonumber\\
    &= \left\{ q\sinh^2(\alpha)\cosh(2\alpha) + (1-q)\sinh^4(\alpha) \right\}
    \cdot \mathbb{E}_0[\Delta_S^2] \label{eq:exact-observed-chi-square}\\
    &\leq \left\{ q\sinh^2(\alpha)\cosh(2\alpha) + (1-q)\sinh^4(\alpha) \right\}
    \cdot \mathcal{B}_{m, \lambda, q}, \label{ineq:Chi-square-penultimate-calculation}
\end{align}
where in the penultimate step we applied the fact that $V$ is a uniform Rademacher random variable to compute $\mathbb{E}_0[e^{2\alpha V}] = \cosh(2\alpha)$ and the final step follows from the inequality
\begin{align}
\label{ineq:Delta-B}
    \mathbb{E}_0[\Delta_S^2] \leq \mathcal{B}_{m, \lambda, q},
\end{align}
whose proof we defer to the end of the section.

To conclude, note that for $0<\alpha\le1$, there exists a universal constant $C_1 \in (0,6)$ such that the polynomial upper bounds
\[
    \sinh^2(\alpha)\cosh(2\alpha)
    \le
    C_1\alpha^2,
    \qquad\text{and}\qquad
    \sinh^4(\alpha)
    \le
    C_1\alpha^4,
\]
hold.  Moreover, if $0 < \alpha^2 \leq q$, then $\alpha^4 \leq q\alpha^2$.  Substituting these bounds into Ineq.~\eqref{ineq:Chi-square-penultimate-calculation} yields
\begin{align*}
    \chi^2(Q_1, Q_0)
    &\leq \left\{ q\sinh^2(\alpha)\cosh(2\alpha) + (1-q)\sinh^4(\alpha) \right\}
    \cdot \mathcal{B}_{m, \lambda, q}
    \leq 2C_1 q \alpha^2 \mathcal{B}_{m, \lambda, q},
\end{align*}
which proves \Cref{eq:two-point-observed-chi2} upon taking $C = 2C_1$.  In turn, \Cref{eq:two-point-observed-chi2-n} follows by applying the tensorization property of $\chi^2$-divergence in conjunction with the numeric inequality $1 + x \leq e^x$ to obtain
\[
    1+\chi^2\left(Q_1^{\otimes n},Q_0^{\otimes n}\right)
    \le
    \left\{1+\chi^2(Q_1,Q_0)\right\}^n
    \le
    \exp\left\{
        Cnq\alpha^2\,\mathcal{B}_{m,\lambda,q}
    \right\},
\]
as desired. \hfill $\qed$

It remains to prove Ineq.~\eqref{ineq:Delta-B}.
\paragraph{Proof of Ineq.~\eqref{ineq:Delta-B}.} 
We define the shorthand
\begin{align}
    \label{eq:def-Wmq}
    \mathcal{W}_{m, q} := \mathbb{E}_0\bigl[\Delta_S^2] = \mathbb{E}_0[\{J_S(T_S) - 1\}^2\bigr]
\end{align}
\Cref{lem:conditional-parseval} provides the following explicit representation:
\begin{equation}
\label{eq:Wmq-closed}
    \mathcal{W}_{m,q}
    =
    4\sum_{k=r+1}^m
    \binom{m}{k}
    \binom{k-1}{r}^{\!2}
    \left(\frac{t}{4}\right)^k,
    \qquad\text{where}\,~ t:=\frac{q(1-\rho^2)}{\rho^2}=\frac{q}{\sinh^2(\lambda)}.
\end{equation}
Fix $s>0$. Note that for all natural numbers $k \ge r+1$,
\begin{equation}
\label{ineq:binom-expansion-ineq}
    \binom{k-1}{r}s^r
    \le
    \sum_{j=0}^{k-1}\binom{k-1}{j}s^j
    =
    (1+s)^{k-1}.
\end{equation}
Therefore, using \Cref{ineq:binom-expansion-ineq} in \Cref{eq:Wmq-closed}, we obtain 
\begin{align}
    \mathcal{W}_{m,q}
    &\le
    4
    \sum_{k=r+1}^m
    \binom{m}{k}
    \frac{(1+s)^{2k-2}}{s^{2r}}
    \left(\frac{t}{4}\right)^k \nonumber\\
    &\le
    \frac{4}{(1+s)^2s^{2r}}
    \sum_{k=0}^m
    \binom{m}{k}
    \left\{
        \frac{t(1+s)^2}{4}
    \right\}^k \nonumber \\
    &=
    \frac{4}{(1+s)^2s^{m-1}}
    \left\{
        1+\frac{t(1+s)^2}{4}
    \right\}^m
    =
    \frac{4s}{(1+s)^2}
    \left[
        \frac{
            1+\frac{t(1+s)^2}{4}
        }{s}
    \right]^m
    \le
    \Biggl[
        \frac{1+\frac{t(1+s)^2}{4}}{s}
    \Biggr]^m
\label{eq:Wmq-bound-1}
\end{align}
where we used $2r=m-1$, and the fact that $4s \le (1+s)^2$ for all $s \ge 0$. 

\smallskip
It remains to minimize
\[
    v(s) 
    := \frac{1+\frac{t(1+s)^2}{4}}{s}
    =
    \frac{1}{s}+\frac{t}{4}\left(s+2+\frac{1}{s}\right).
\]
Differentiation with respect to $s$ gives
\[
    v'(s) 
    =
    -\frac{1}{s^2} + \frac{t}{4} \left(1-\frac{1}{s^2}\right)
    =
    -s^{-2} \left(1 + \frac{t}{4}\right) + \frac{t}{4},
\]
which vanishes at $s_\star = \sqrt{4/t + 1}$. Also, $v''(s) = 2s^{-3}(1+t/4) > 0$ for all $s > 0$. Thus, $s_\star$ is the unique minimizer of $v(\cdot)$ on $(0,\infty)$, and
\[
    v(s_\star)
    = \inf_{s>0} v(s)
    = \inf_{s>0} \frac{1+\frac{t(1+s)^2}{4}}{s}
    = \frac{t+\sqrt{t^2+4t}}{2}
    = \beta(t).
\]
Hence, \Cref{eq:Wmq-bound-1} implies
\[
    \mathcal{W}_{m,q} \le \big[\beta(t)\big]^m = \mathcal{B}_{m,\lambda,q}.
\]
\qed

\subsubsection{Proof of \Cref{lem:V-independence}}\label{sec:proof-lem-V-independence}
We first verify that \(V\) is a uniform Rademacher random variable. Note that for \(v_0\in\{-1,1\}\),
\begin{align*}
    \mathbb{P}_0(V=v_0)
    &=
    \sum_{y\in\{-1,1\}}
    \mathbb{P}_0(Y=y,YB=v_0)
    =
    \sum_{y\in\{-1,1\}}
    \mathbb{P}_0(Y=y,B=v_0y).
\end{align*}
Since \(B\) is a uniform Rademacher random variable independent of \(Y\),
\[
    \mathbb{P}_0(Y=y,B=v_0y)
    =
    \mathbb{P}_0(Y=y) \cdot \mathbb{P}_0(B=v_0y)
    =
    \frac{1}{2}\mathbb{P}_0(Y=y).
\]
Therefore,
\[
    \mathbb{P}_0(V=v_0)
    =
    \frac{1}{2}
    \sum_{y\in\{-1,1\}}
    \mathbb{P}_0(Y=y)
    =
    \frac{1}{2}.
\]
Thus, \(V\) is a uniform Rademacher random variable.

We next prove that \(V\) is independent of \((Y,T_1,\ldots,T_m)\). Fix \(v_0,y\in\{-1,1\}\) and \(t=(t_1,\ldots,t_m)\in\{-1,1\}^m\). Since \(V=YB\), on the event \(\{Y=y\}\), the event \(\{V=v_0\}\) is equivalent to \(\{B=v_0y\}\). Hence,
\begin{align*}
    \mathbb{P}_0(V=v_0,Y=y,T=t)
    &=
    \mathbb{P}_0(B=v_0y,Y=y,T=t).
\end{align*}
By construction, \(B\) is independent of \((Y,T)\), and it is uniform on \(\{-1,1\}\). Consequently,
\begin{align*}
    \mathbb{P}_0(V=v_0,Y=y,T=t)
    &=
    \mathbb{P}_0(B=v_0y) \cdot \mathbb{P}_0(Y=y,T=t)
    =
    \frac{1}{2}
    \mathbb{P}_0(Y=y,T=t).
\end{align*}
Since \(\mathbb{P}_0(V=v_0)=1/2\), we conclude that
\[
    \mathbb{P}_0(V=v_0,Y=y,T=t)
    =
    \mathbb{P}_0(V=v_0) \cdot \mathbb{P}_0(Y=y,T=t).
\]
Because this identity holds for every \(v_0,y\in\{-1,1\}\) and every \(t\in\{-1,1\}^m\), \(V\) is independent of \((Y,T_1,\ldots,T_m)\).

Finally, for every deterministic \(S\subseteq[m]\), the vector \(T_S\) is a measurable function of \(T=(T_1,\ldots,T_m)\). Independence of \(V\) and \((Y,T)\) therefore implies, in particular, that $V \indep T_S$.

\subsection{Proof of \Cref{lem:large-R-assouad-validity}}

\paragraph{Proof of part \emph{(a)}.}
Since $\alpha\|J\|_\infty \le 1/2$, the alternative member of the two-point construction is well-defined. Conditional on $K=k$, \Cref{eq:large-R-assouad-logit} implies that
\[
    \mathbb{P}_\omega(Y=y\mid A=a,B=b,K=k)
    =
    \frac{
        \exp\left\{
            y\bigl(s(a)+\alpha\omega_k b\bigr)
        \right\}
    }{
        2\cosh\left\{
            s(a)+\alpha\omega_k b
        \right\}
    }.
\]
Since $K$, $A$, and $B$ are determined by $X$, this is precisely
\[
    \mathbb{P}_\omega(Y=y\mid X=x)
    =
    \frac{
        \exp\left(\frac{1}{2} yx^\top\theta_\omega\right)
    }{
        2\cosh\left(\frac{1}{2} x^\top\theta_\omega\right)
    }.
\]
Thus, the model is well specified with $\theta_\omega$ being the true parameter since by the definition in~\Cref{eq:large-R-assouad-theta},
\[
\|\theta_\omega\|^2 = 8m^2\lambda^2 + 8\alpha^2 \sum_k \omega_k^2 \le \frac{R^2}{4} + \frac{3R^2}{4} = R^2.
\]

\paragraph{Proof of part \emph{(b)}.}  The almost sure bound on $\| X \|$ holds by construction, so we focus on lower bounding $\mathbb{E}_{\omega}[X X^{\top}]$.

Conditional on $K=k$, the distribution of $(A,B,Y)$ is either $P_0$ or $P_1$, depending on whether $\omega_k=0$ or $\omega_k=1$. Since $m\ge5$, \Cref{lem:two-point-validity}(c) shows that the second-moment matrix of $(A,B)$ is the same under these two distributions.
Consequently, conditional on $K=k$,
\[
    \mathbb{E}_\omega[A_j A_{\ell}\mid K=k]
    =
    \mathbb{E}_0[A_j A_{\ell}],
    \qquad
    \mathbb{E}_\omega[A_j B\mid K=k]=0,
    \qquad
    \mathbb{E}_\omega[B^2 \mid K=k]=1,
\]
independently of the value of $\omega_k$.

Averaging over the uniform random variable $K$, we therefore obtain
\[
    \mathbb{E}_\omega[XX^\top]
    =
    \frac{1}{2}
    \begin{bmatrix}
        \mfrac{1}{m}
        \left\{
            (1-\rho^2)I_m
            +
            \rho^2\boldsymbol 1_m\boldsymbol 1_m^\top
        \right\} & 0 \\ 
        0 & \mfrac{1}{p} I_p
    \end{bmatrix}.
\]
The smallest eigenvalues of the upper-left and lower-right blocks are $\frac{1-\rho^2}{2m}$ and $1/(2p)$, respectively. Hence, for all $\omega \in \{0,1\}^p$,
\[
    \mathbb{E}_\omega[XX^\top]
    \succeq
    \frac{1}{2}
    \min\left\{
        \frac{1-\rho^2}{m},
        \frac{1}{p}
    \right\}
    I_{m+p}.
\]

\paragraph{Proof of part \emph{(c)}.} Recall the definition of $\theta_\omega$ from \Cref{eq:large-R-assouad-theta}. If $\omega \in \{0,1\}^p$ satisfies $\omega_k = 0$, for the neighboring vertex $\omega^{(k)}$ with $\omega^{(k)}_k = 1$, we have
\begin{align*}
\label{eq:large-R-assouad-theta}
    \theta_{\omega^{(k)}}
    &:=
    2\sqrt{2m}\,\lambda \sum_{j=1}^m e_j + 2\sqrt{2}\,\alpha \sum_{r=1}^p\omega^{(k)}_r e_{m+r} \\
    &=
    2\sqrt{2m}\,\lambda \sum_{j=1}^m e_j + 2\sqrt{2}\,\alpha \sum_{r \in [p] \setminus \{k\}} \omega^{(k)}_r e_{m+r} + 2\sqrt{2}\,\alpha \, \omega^{(k)}_k e_{m+k} \\
    &\overset{\1}{=}
    2\sqrt{2m}\,\lambda \sum_{j=1}^m e_j + 2\sqrt{2}\,\alpha \sum_{r = 1}^p \omega_r e_{m+r} + 2\sqrt{2}\,\alpha \, \omega^{(k)}_k e_{m+k} \\
    &=
    \theta_\omega + 2\sqrt{2}\,\alpha \,e_{m+k},
\end{align*}
where step $\1$ uses the facts that $\omega_r = \omega^{(k)}_r$ for all $r \in [p] \setminus \{k\}$ and $\omega_k = 0$.

Since $\norm{e_{m+k}} = 1$, the claimed result follows.

\subsection{Proof of \Cref{lem:large-R-assouad-hamming}}

Let $\widetilde Q_\omega$ denote the joint law of $(K,Y,M,Z)$ under $P_\omega$. Since $Q_\omega$ is the marginal law of $(Y,M,Z)$, by the data-processing inequality,
\begin{equation}
\label{eq:chi2-dpi-step}
    \chi^2\left( Q_{\omega^{(k)}},Q_\omega \right)
    \le
    \chi^2\big(
        \widetilde Q_{\omega^{(k)}},
        \widetilde Q_\omega
    \big).
\end{equation}

For $\ell\in[p]$ and $i\in\{0,1\}$, let $Q_{i,\ell}$ denote the conditional law of $(Y,M,Z)$ given $K=\ell$ when the selected bit equals $i$. Since $K$ is uniform on $[p]$,
\begin{align}
\label{eq:assouad-joint-neighbor}
    \chi^2\left( \widetilde Q_{\omega^{(k)}}, \widetilde Q_\omega \right)
    &=
    \frac{1}{p} \sum_{\ell=1}^p \chi^2\left( Q_{\omega^{(k)}_\ell,\ell}, Q_{\omega_\ell,\ell} \right)
    =
    \frac{1}{p} \chi^2\left( Q_{\omega^{(k)}_k,k}, Q_{\omega_k,k} \right)
    =
    \frac{1}{p} \, \chi^2\left( Q_{1,k},Q_{0,k} \right),
\end{align}
where the last two equalities use $\omega^{(k)}_\ell=\omega_\ell$ for $\ell\neq k$, together with $\omega_k=0$ and $\omega^{(k)}_k=1$.

Conditional on $K=k$, all $p$ many appended coordinates other than $m+k$ are identically zero, and their mask variables $(M_{m+\ell}:\ell\neq k)$ are independent of the remaining observed data and have the same distribution under $Q_{0,k}$ and $Q_{1,k}$. They therefore do not affect the $\chi^2$-divergence.

After removing these common coordinates, the resulting pair of observed laws is exactly the two-point pair $(Q_0,Q_1)$ from the preceding subsubsection. Thus,
\begin{equation}
\label{eq:conditional-neighbor-is-two-point}
    \chi^2(Q_{1,k},Q_{0,k})
    =
    \chi^2(Q_1,Q_0).
\end{equation}
Combining \Cref{eq:chi2-dpi-step}, \Cref{eq:assouad-joint-neighbor}, and \Cref{eq:conditional-neighbor-is-two-point} yields
\[
    \chi^2\left( Q_{\omega^{(k)}},Q_\omega \right)
    \,\le\,
    \chi^2\big(
        \widetilde Q_{\omega^{(k)}},
        \widetilde Q_\omega
    \big)
    \,\le\,
    \frac{1}{p} \, \chi^2(Q_1,Q_0).
\]

By assumption, $0<\alpha^2\le q$. Thus, applying Lemma~\ref{lem:two-point-divergence} yields
\[
    \chi^2\left(Q_{\omega^{(k)}},Q_\omega\right)
    \le 
    \frac{Cq\alpha^2}{p} \, \mathcal{B}_{m,\lambda,q}.
\]
Finally, by tensorization,
\begin{align*}
    1+
    \chi^2\big( Q_{\omega^{(k)}}^{\otimes n}, Q_\omega^{\otimes n} \big)
    &=
    \bigl\{
        1+\chi^2\left( Q_{\omega^{(k)}},Q_\omega \right)
    \bigr\}^{n}
    \le
    \exp\left\{ \frac{Cnq\alpha^2}{p} \, \mathcal{B}_{m,\lambda,q} \right\},
\end{align*}
which proves the desired result.
\qed

\section{Additional details of simulation in Figure~\ref{fig:intro-simulation}}
\label{sec:non-Gaussian-simulation-details}

In this section, we provide the detailed method that was used to generate Figure~\ref{fig:intro-simulation}.

\paragraph{Data-generating mechanism.}
For an even dimension $d$, let $K:=d/2$ and partition the covariate vector $X\in\mathbb{R}^d$ into the $K$ two-dimensional blocks $X_{[b]}:=(X_{2b-1},X_{2b})$, for $b \in [K]$.
The blocks are mutually independent. For each $b\in[K]$, independently draw $\xi_b$ uniformly from $\{-1,1\}$ and set
\[
X_{[b]}=
\begin{cases}
(\xi_b,-\xi_b),&\text{with probability }1-\epsilon,\\
(\xi_b,L\xi_b),&\text{with probability }\epsilon.
\end{cases}
\]
We take $\epsilon=0.2$ and $L=6$. Thus, $\mathbb{E}[X]=0$, and the covariance matrix of each block is
\[
\Sigma_{\mathrm{blk}}
=
\begin{bmatrix}
1 & \epsilon L-(1-\epsilon)\\
\epsilon L-(1-\epsilon) & (1-\epsilon)+\epsilon L^2
\end{bmatrix}
=
\begin{bmatrix}
1 & 0.4\\
0.4 & 8
\end{bmatrix}.
\]
Since the blocks are independent, the full covariance matrix is $\Sigma_X=I_K\otimes\Sigma_{\mathrm{blk}}$. Note that the covariates are decidedly non-Gaussian.

We choose a signal $\thetastar$ having a single spike in coordinate $2$ and small values in other even coordinates. In particular, all odd coordinates of $\thetastar$ are chosen to be zero, while
\[
\theta_2^\star := \alpha\sqrt{1-\delta^2},
\quad\;\text{and}\quad\;
\theta_{2b}^\star := \frac{\alpha\delta}{\sqrt{K-1}},
\quad\text{for}\,~ b=2,\ldots,K,
\]
where we take $\alpha=1.25$ and $\delta=0.15$. Consequently, $\|\theta^\star\|_2=\alpha$, and the first block of $\thetastar$ carries a fraction $1-\delta^2$ of the squared signal strength.

For each $i\in[n]$, let $X^{(i)}$ be an independent copy of $X$.
Conditional on $X^{(i)}$, the response $Y^{(i)}\in\{-1,1\}$ follows the logistic model
\[
\mathbb{P}\bigl(Y^{(i)}=y\mid X^{(i)}\bigr)
=
\sigma\bigl(y(X^{(i)})^\top\thetastar\bigr),
\qquad
\sigma(u)=\frac{1}{1+e^{-u}}.
\]
Missingness is homogeneous and introduced independently across observations and coordinates. In particular,
\[
M_j^{(i)} \overset{\scriptscriptstyle\mathrm{i.i.d.}}{\sim} \BER(q),
\qquad
Z_j^{(i)}:=M_j^{(i)}X_j^{(i)},
\qquad i\in[n],\ j\in[d],
\]
where $M_j^{(i)}=1$ indicates that the $j$-th coordinate of the $i$-th sample is observed. We consider the two settings $(d,q)\in\{(10,0.40),(100,0.85)\}$, use $B=100$ Monte Carlo replications, and vary the sample size over
\[
n\in\{200,500,1000,2000,5000,10000,20000,50000,100000,250000,500000, 1000000\}.
\]

\paragraph{Procedures compared.}
We compare the following estimators as already mentioned in \Cref{sec:intro}.

Let $\Theta_R:=\{\theta\in\mathbb{R}^d:\|\theta\|_2\leq R\}$, where $R=1.75$, and write
\[
\ell(\theta;y,x)
:=
\log\bigl(1+\exp(-yx^\top\theta)\bigr)
\]
for the logistic loss. We compare the following six estimators.

\begin{enumerate}[label=(\Roman*)]
\item 
As a gold-standard benchmark, we compute the logistic MLE from the complete, unmasked covariates:
\[
    \widehat\theta_{\mathrm{oracle}}^{(n)}
    \in
    \argmin_{\theta\in\Theta_R} \,
    \frac{1}{n}\sum_{i=1}^n 
    \ell\bigl(\theta;Y^{(i)},X^{(i)}\bigr).
\]

\item \emph{Complete-case MLE.}
Let $\mathcal{C}_n:=\{i \in [n] : M^{(i)}_j=1 \text{ for every }j\}$ denote the set of fully observed samples. When $\mathcal{C}_n \neq \varnothing$, we compute
\[
    \widehat\theta_{\mathrm{CC}}^{(n)}
    \in
    \argmin_{\theta\in\Theta_R} \,
    \frac{1}{|\mathcal{C}_n|}\sum_{i\in\mathcal{C}_n} 
    \ell\bigl(\theta;Y^{(i)},X^{(i)}\bigr).
\]
If no complete case is available, we set $\widehat\theta_{\mathrm{CC}}^{(n)}=0$. Notice that a sample is complete with probability $q^d$; thus the expected number of complete-case samples is $nq^d$.

\item \emph{Zero imputation.}
We replace every missing coordinate by zero, resulting in the imputed covariate $Z^{(i)} := M^{(i)} \odot X^{(i)}$, and compute
\[
    \widehat\theta_{\mathrm{ZI}}^{(n)}
    \in
    \argmin_{\theta\in\Theta_R} \,
    \frac{1}{n}\sum_{i=1}^n
    \ell\bigl(\theta;Y^{(i)},Z^{(i)}\bigr).
\]

\item \emph{Gaussian conditional-mean imputation.}
We introduce the moment-matched Gaussian model $G:=\mathcal{N}(0,\Sigma_X)$. We use the true covariance matrix $\Sigma_X$, rather than estimating it from the data. We impute the missing coordinates using
\[
\widetilde X_{\scriptscriptstyle\mathrm{GMI}}^{(i)}
:=
\mathbb{E}_G \Bigl[X^{(i)}\bigm|Z^{(i)},M^{(i)}\Bigr].
\]
The resulting estimator is
\[
    \widehat\theta_{\scriptscriptstyle\mathrm{GMI}}^{(n)}
    \in
    \argmin_{\theta\in\Theta_R} \,
    \frac{1}{n}\sum_{i=1}^n
    \ell\Bigl(\theta; Y^{(i)},\widetilde X_{\scriptscriptstyle\mathrm{GMI}}^{(i)}\Bigr).
\]
Because $\Sigma_X$ is block diagonal, the conditional means are computed separately within each two-dimensional block. A block with neither coordinate observed is imputed by its unconditional mean, which is zero.

\item \emph{Gaussian observed-data MLE.}
This procedure also approximates the distribution of $X$ by $G=\mathcal{N}(0,\Sigma_X)$ and uses the true covariance matrix $\Sigma_X$. Rather than imputing the missing coordinates, it integrates them out of the likelihood. In particular, define
\[
p_G^{(i)}(\theta)
:=
\mathbb{E}_G \Bigl[\sigma\bigl(Y^{(i)} \bigl\langle X^{(i)},\theta\bigr\rangle\bigr) \Bigm| Z^{(i)},M^{(i)}\Bigr].
\]
The Gaussian observed-data MLE is
\[
    \widehat\theta_{\scriptscriptstyle\mathrm{GMLE}}^{(n)}
    \in
    \argmin_{\theta\in\Theta_R}
    \left\{
        -\frac{1}{n}\sum_{i=1}^n
        \log p_G^{(i)}(\theta)
    \right\}.
\]
Conditional on $(Z^{(i)},M^{(i)})$ under $G$, note that $\bigl\langle X^{(i)},\theta \bigr\rangle$ is univariate Gaussian. 
We therefore approximate $p_G^{(i)}(\theta)$ using a $50$-point Gauss--Hermite quadrature. Both Gaussian procedures are supplied with the exact covariance matrix $\Sigma_X$, so that their performance is not affected by estimation error in the first two moments.

\item \emph{Proposed stochastic approximation.}
Recall the loss function $\Phi$ from \Cref{eq:Phi-theta}, and the fact that 
\[
\thetastar
=
\argmin_{\theta\in\Theta_R}\,
\underbrace{2 \, \mathbb{E}\left[
    \exp\left\{-\mfrac{1}{2}YX^\top\theta\right\}
\right]}_{\Phi(\theta)}.
\]
Also, recall that its gradient $\PopOp$, which is the population operator of interest, is not computable from the incomplete data. Thus, our proposed \Cref{alg:Proj-SA} uses as the descent direction the unbiased stochastic operator $\StochOp$ from \Cref{eq:A-hat} in place of $\PopOp$.
We initialize at $\widehat\theta^{(0)}=0$, set $\mu=\lambda_{\min}(\Sigma_X)$. After $n$ iterations, the output of \Cref{alg:Proj-SA} is $\widehat\theta^{(n)}$.

\end{enumerate}

\paragraph{Numerical optimization scheme for competing methods.}

For every competing method $\mathrm{M} \in\{\mathrm{oracle},\mathrm{CC},\mathrm{ZI},\mathrm{GMI},\mathrm{GMLE}\}$ and a sample size $n$, let $F_{\mathrm{M}}^{(n)}$ denote its corresponding objective function. Denote by $\Pi_{\Theta_R}$ the $\ell_2$--projection operator onto $\Theta_R$.

We use $\vartheta^{(\ell)}$ to denote the $\ell$-th iterate of the numerical optimizer. Starting from $\vartheta^{(0)}\in\Theta_R$, projected gradient descent performs the updates
\[
\vartheta^{(\ell+1)}
=
\Pi_{\Theta_R}\Bigl\{
\vartheta^{(\ell)}
-\gamma_\ell\nabla F_{\mathrm{M}}^{(n)}\bigl(\vartheta^{(\ell)}\bigr)
\Bigr\},
\]
until convergence, or after $300$ steps. The step size $\gamma_\ell$ is initialized using a Barzilai--Borwein rule~\citep{BarzilaiBorwein-1988-twopointGD} and then repeatedly halved via Armijo backtracking. We denote the final iterate obtained for procedure $\mathrm{M}$ at sample size $n$ by $\widehat\theta_{\mathrm{M}}^{(n)}$.

At the smallest sample size, i.e., $n = 200$, we take $\vartheta^{(0)} := 0$. For each subsequent sample size, we initialize $\vartheta^{(0)}$ at the estimate obtained for the same procedure at the preceding sample size.

\section{Auxiliary technical lemmas}
\label{appsec:aux-results}

In this appendix, we collect auxiliary technical lemmas that were used in the proofs.

\subsection{Technical lemmas related to Boolean Fourier analysis}

Our first section collects lemmas on Boolean Fourier analysis for the majority function. See~\citet{odonnell2021analysisbooleanfunctions} for a more general treatment.

\subsubsection{Fourier coefficients of the majority function}
In this section we provide several useful properties of the majority function
\begin{equation*}
    f(t):=\operatorname{Maj}_m(t)
    :=\operatorname{sign}\Biggl\{\,\sum_{j=1}^m t_j\Biggr\},
\end{equation*}
whose Fourier coefficients are denoted by $\{g_U:U \subseteq [m]\}$ as in \Cref{eq:def-gU}.

\begin{lemma}
\label{lem:majority-coefficients}
Let $r \in \mathbb{N} \cup \{0\}$ and suppose that $m = 2r + 1$.  The Fourier coefficients $\{g_U:U \subseteq [m]\}$ admit the representation
\[
g_U = 
\begin{cases}
~0\,, & \text{if~}\, |U|~ \text{is even};\\ 
~\displaystyle\frac{1}{2^{m-1}}
        \sum_{\substack{C\subseteq[m],\\ |C|\leq r}}
        (-1)^{|C\cap U|}\,, & \text{if~}\, |U|~ \text{is odd}.
\end{cases}
\]
Moreover, $\sum_{U \subseteq [m]} \, g_U^2 = 1$.
\end{lemma}

\begin{proof}
We first consider the case in which $\lvert U \rvert$ is even.  Recall that
\[
g_U = \frac{1}{2^m} \sum_{t\in\{-1,1\}^m} f(t) \, \chi_U(t).
\]
Since the majority function is odd and the Fourier characters satisfy $\chi_U(-x)=(-1)^{|U|}\chi_U(x)$, it follows that $g_U = (-1)^{|U|+1}g_U$.  Since $\lvert U \rvert$ is even, we deduce that $g_U = 0$. 

Next, consider the case in which $\lvert U \rvert$ is  odd.  For $x\in\{-1,1\}^m$, define the set $C(x) \subseteq [m]$ to denote the indices of the negative components of $x$:
\[
    C(x):=\{j\in[m]:x_j=-1\},
\]
and use this to write the Fourier characters as $\chi_U(x)=(-1)^{|C(x)\cap U|}$.  Further, note that since $m = 2r + 1$, it holds that $f(x) = 1$ if and only if $\lvert C(x) \rvert \leq r$.  Combining these two facts yields the equivalent representation 
\[
g_U 
= \frac{1}{2^m} \sum_{t\in\{-1,1\}^m} f(t) \, \chi_U(t) 
= \frac{1}{2^m} \sum_{C \subseteq [m]} f_C \cdot (-1)^{|C\cap U|},
\]
where $f_C=1$ if $|C|\leq r$, and $f_C=-1$ if $|C|\geq r+1$.

Continuing, pair every set $C\subseteq[m]$ with its complement $C^c$. Since $m=2r+1$, exactly one member of each pair has cardinality at most $r$. Moreover, $f_{C^c}=-f_C$, and
\[
    (-1)^{|C^c\cap U|}
    =
    (-1)^{|U|-|C\cap U|}
    =
    -(-1)^{|C\cap U|},
\]
where the last equality uses the fact that $|U|$ is odd. Hence the two terms in each pair are equal. It thus follows that
\[
    g_U
    =
    \frac{1}{2^{m-1}}
    \sum_{\substack{C\subseteq[m]\\|C|\leq r}}
    (-1)^{|C\cap U|},
\]
as desired.

Finally, we note that the relation $\sum_{U \subseteq [m]} g_U^2 = 1$ follows from Parseval's identity.
\end{proof}

\begin{lemma}
\label{lem:boolean-cancellation}
Let $r \in \mathbb{N} \cup \{0\}$ and let $m = 2r + 1$.  The Fourier coefficients $\{g_U\}_{U \subseteq [m]}$ of the majority function satisfy the following.  For every $S\subseteq[m]$ satisfying $|S|\leq r$, and every $C\subseteq S$,
\[
    \sum_{\substack{
        U \subseteq [m],\\
        |U|\ \mathrm{odd},\\
        U\cap S=C
    }} g_U
    \,=\,
    \mathbbm{1}\{C=\varnothing\}.
\]
\end{lemma}

\begin{proof}
Fix $S\subseteq[m]$ with $|S|\leq r$. For every $x_S\in\{-1,1\}^S$, consider the vector $(x_S,\boldsymbol{1}_{S^c})\in\{-1,1\}^m$. Regardless of the values of $x_S$, this vector has at least $m-|S|$ positive coordinates and at most $|S|$ negative coordinates. Since
\[
    m-|S| \geq 2r+1-r = r+1 > |S|,
\]
the positive coordinates form a strict majority. Therefore,
\[
    f(x_S,\boldsymbol{1}_{S^c})=1
    \qquad
    \text{for every}\,~x_S\in\{-1,1\}^S.
\]
Using the Fourier expansion of $f$, we thus have
\begin{align}
    1
    =
    f(x_S,\boldsymbol{1}_{S^c})
    =
    \sum_{U\subseteq[m]} g_U \chi_U(x_S,\boldsymbol{1}_{S^c}) \nonumber
    &=
    \sum_{U\subseteq[m]} g_U \prod_{j\in U\cap S} x_j \nonumber \\
    &=
    \sum_{C\subseteq S} \Bigg(\sum_{\substack{U\subseteq[m]\\U\cap S=C}} g_U\Bigg) \prod_{j\in C} x_j \nonumber\\
    &=
    \Bigg(\sum_{\substack{U\subseteq[m]\\U\cap S=\varnothing}} g_U\Bigg) +
    \sum_{\varnothing \neq C\subseteq S} \Bigg(\sum_{\substack{U\subseteq[m]\\U\cap S=C}} g_U\Bigg) \prod_{j\in C} x_j.
\label{eq:extended-maj-fourier-1}
\end{align}

On the other hand, note that the constant function $1$ has Fourier coefficient $1$ at $C=\varnothing$ and $0$ otherwise, since
\begin{align}
\label{eq:extended-maj-fourier-2}
    1
    =
    1 + \sum_{\varnothing \neq C\subseteq S} 0 \cdot \prod_{j\in C} x_j
\end{align}

Since the functions $\{\chi_C:C\subseteq S\}$ form an orthonormal basis for the space of real-valued functions on $\{-1,1\}^S$~\citep[Theorem~1.5]{odonnell2021analysisbooleanfunctions}, the Fourier expansion is unique. Therefore, comparing the Fourier coefficients of $f(x_S,\boldsymbol{1}_{S^c}) \equiv 1$ from \Cref{eq:extended-maj-fourier-1,eq:extended-maj-fourier-2}, we obtain
\[
    \sum_{\substack{U\subseteq[m]\\U\cap S=C}} g_U
    =
    \mathbbm{1}\{C=\varnothing\},
    \qquad\text{for every}\,~ C\subseteq S.
\]

Finally, we use the result from \Cref{lem:majority-coefficients}~\citep[or,][Theorem~5.19]{odonnell2021analysisbooleanfunctions} that $g_U=0$ for every even-cardinality $U$, so the inner sum may equivalently be restricted to odd-cardinality $U$. Thus,
\[
    \sum_{\substack{
        U \subseteq [m]\\
        |U|\ \mathrm{odd}\\
        U\cap S=C
    }}
    g_U
    \,=\,
    \mathbbm{1}\{C = \varnothing\},
\]
which completes the proof.
\end{proof}

For the next pair of lemmas, it is useful to define the coefficients 
\begin{equation}
    \label{eq:def-hU}
    h_U := \sum_{W\supseteq U} g_W
    \qquad
    \text{for all}\,~ U \subseteq [m].
\end{equation}

\begin{tcolorbox}
\begin{lemma}
\label{lem:exact-hU}
For each $U \subseteq [m]$, the quantity $h_U$ as in \Cref{eq:def-hU} depends only on $k:=|U|$, and has the following closed-form expression
\[
    h_U
    =
    \begin{cases}
        \, 1, & \text{if}\;~ k=0, \\ 
        \, 0, & \text{if}\;~ 1\le k\le r, \\ 
        \, \displaystyle \dfrac{1}{2^{k-1}} \, (-1)^r \, \binom{k-1}{r}, & \text{if}\;~ r+1\le k\le m.
    \end{cases}
\]
\end{lemma}
\end{tcolorbox}

\begin{proof}
Fix $U \subseteq [m]$, and write $k:=|U|$. Define the restriction
\[
    f_U(z) := f(z,\boldsymbol 1_{U^c}),
    \qquad
    \text{for}\;~z\in\{-1,1\}^U.
\]
Restricting the Fourier expansion of $f$ gives
\[
    f_U(z) = \sum_{W\subseteq[m]}g_W \cdot \chi_{W\cap U}(z).
\]
Therefore, the Fourier coefficient of the full monomial $\chi_U$ in $f_U$ is
\[
    \widehat{f}_U(U) = \sum_{W\supseteq U}g_W = h_U.
\]
Hence,
\begin{equation}
\label{eq:hU-restricted-Fourier}
    h_U
    =
    \frac{1}{2^{k}} \sum_{z\in\{-1,1\}^U} f_U(z) \cdot \chi_U(z).
\end{equation}

Suppose that $z$ has exactly $j$ coordinates equal to $-1$. Then, $\chi_U(z)=(-1)^j$, and the sum of all $m$ coordinates of $(z,\boldsymbol 1_{U^c})$ is
\[
    (m-k)+(k-2j) = m-2j = 2(r-j)+1.
\]
Since $m=2r+1$,
\[
    f_U(z)
    =
    \begin{cases}
        +1, & \text{if}\,~ j\le r,\\
        -1, & \text{if}\,~ j\ge r+1.
    \end{cases}
\]
Thus,
\begin{equation}
\label{eq:hU-binomial-sum}
    h_U
    = 2^{-k}
    \left\{
        \sum_{j=0}^{\min\{r,k\}}
        (-1)^j\binom{k}{j}
        -
        \sum_{j=r+1}^{k}
        (-1)^j\binom{k}{j}
    \right\},
\end{equation}
where the second sum is $0$ when $k\le r$.

\begin{itemize}
\item 
If $k=0$, then $h_\varnothing=f(\boldsymbol 1_m)=1$. Suppose henceforth that $k\ge1$. Since
\[
    \sum_{j=0}^k(-1)^j\binom{k}{j}=0,
\]
\Cref{eq:hU-binomial-sum} becomes
\[
    h_U
    =
    2^{1-k} \sum_{j=0}^{\min\{r,k\}} (-1)^j\binom{k}{j}.
\]
\item 
If $1\le k\le r$, the sum on the right is the complete alternating binomial sum, and hence $h_U=0$.

\item
If $k\ge r+1$,  we have
\begin{align*}
    h_U = 2^{1-k} \sum_{j=0}^r(-1)^j\binom{k}{j}
    &=
    2^{1-k} \sum_{j=0}^r(-1)^j
    \left\{
        \binom{k-1}{j}
        +
        \binom{k-1}{j-1}
    \right\}
    \\
    &=
    2^{1-k} \left[\sum_{j=0}^r(-1)^j\binom{k-1}{j}
    -
    \sum_{j=0}^{r-1}(-1)^j\binom{k-1}{j}\right]
    \\
    &=
    \frac{1}{2^{k-1}} (-1)^r\binom{k-1}{r}.
\end{align*}
\end{itemize}
This proves the claimed formula.
\end{proof}

\subsubsection{Fourier coefficient representation of \texorpdfstring{$\mathcal{W}_{m,q}$}{Wmq}}

Recall the definition of $\mathcal{W}_{m,q}$ from \Cref{eq:def-Wmq}. We derive its exact closed-form expression as a weighted sum of the family of coefficients $\{h_U : U \subseteq [m]\}$ defined in \Cref{eq:def-hU}.

\begin{tcolorbox}
\begin{lemma}
\label{lem:conditional-parseval}
$\mathcal{W}_{m,q}$, defined in \Cref{eq:def-Wmq}, admits the representation:
\begin{equation*}
    \mathcal{W}_{m,q}
    ~=
    \sum_{\varnothing\neq U \subseteq [m]}
    \left\{
        \frac{q(1-\rho^2)}{\rho^2}
    \right\}^{|U|}
    h_U^2.
\end{equation*}
Consequently, $\mathcal{W}_{m,q}$ has the following explicit expression in terms of $q$, $\rho$, and $m$:
\begin{equation*}
    \mathcal{W}_{m,q}
    \,=\,
    4 \sum_{k=r+1}^m
    \binom{m}{k} \binom{k-1}{r}^{\!2}
    \left\{
        \frac{q(1-\rho^2)}{4\rho^2}
    \right\}^{\!k}.
\end{equation*}
\end{lemma}
\end{tcolorbox}

\begin{proof}
Let us denote $\sigma_\rho:=\sqrt{1-\rho^2}$, and for all $U \subseteq [m] \setminus \varnothing$, define
\begin{equation}
\label{eq:psi-T}
    \psi_U(T)
    :=
    \prod_{j\in U} \underbrace{\frac{T_j-\rho}{\sigma_\rho}}_{=:\,\beta_{j}(\rho)},
    \qquad
    \psi_\varnothing\equiv1.
\end{equation}

Since $T_1,\ldots,T_m$ are independent, with $\mathbb{E}_0[T_j]=\rho$ and
$\operatorname{Var}_0(T_j)=1-\rho^2$, the collection $\{\psi_U:U \subseteq [m]\}$ is orthonormal under the null distribution of $T$. To see this, note the following.
\begin{itemize}
    \item For $A,B \subseteq [m]$ such that $A \neq B$, we have
    \begin{align*}
    \mathbb{E}\left[ \psi_A(T) \cdot \psi_B(T) \right]
    &=
    \mathbb{E}\left[ \prod_{j\in A \setminus B} \beta_{j}(\rho) \right] \cdot 
    \mathbb{E}\left[ \prod_{j\in B \setminus A} \beta_{j}(\rho) \right] \cdot 
    \mathbb{E}\left[ \prod_{j\in A \cap B} \beta_{j}(\rho)^2 \right] \\
    &= \prod_{j\in A \Delta B} \mathbb{E}\left[ \beta_{j}(\rho) \right] \cdot \prod_{j\in A \cap B} \mathbb{E}\left[ \beta_{j}(\rho)^2 \right] \\
    &= 0 \cdot \prod_{j\in A \cap B} \mathbb{E}\left[ \beta_{j}(\rho)^2 \right] \\
    &= 0.
    \end{align*}

    \item For $A = B$, we have
    \begin{align*}
    \mathbb{E}\left[ \psi_A(T)^2 \right]
    =
    \mathbb{E}\left[ \prod_{j\in A} \beta_{j}(\rho)^2 \right] 
    &= \prod_{j\in A} \mathbb{E}\left[ \beta_{j}(\rho)^2 \right] \\
    &= \prod_{j\in A} \frac{1}{1-\rho^2}\var\left[ T_{j} \right] \\
    &= \prod_{j\in A} \frac{1}{1-\rho^2} \left[ \frac{\exp(\lambda) + \exp(-\lambda)}{2 \cdot \cosh(\lambda)} - \rho^2 \right] 
    = \prod_{j\in A} \frac{1-\rho^2}{1-\rho^2} = 1.
    \end{align*}
\end{itemize}

Recall the definition of $\beta_j(\rho)$ in \Cref{eq:psi-T}. For every $j\in[m]$, write
\[
    T_j
    =
    \rho + \sigma_\rho \frac{T_j-\rho}{\sigma_\rho}
    :=
    \rho + \sigma_\rho \, \beta_j(\rho).
\]
Thus, for every $W\subseteq[m]$, we have
\begin{align*}
    \chi_W(T) = \prod_{j \in W} (\rho + \sigma_\rho \beta_j(\rho)) 
    &= \prod_{j \in W} \left( \sum_{\mathcal{I}_j \in \{0, 1\}} \rho^{1 - \mathcal{I}_j} (\sigma_\rho \beta_j(\rho))^{\mathcal{I}_j} \right) \\ 
    &= \sum_{(\mathcal{I}_j)_{j \in W} \in \{0, 1\}^{|W|}} \prod_{j \in W} \rho^{1 - \mathcal{I}_j} (\sigma_\rho \beta_j(\rho))^{\mathcal{I}_j} \\ 
    &= \sum_{U \subseteq W} \left( \prod_{j \in W \setminus U} \rho \right) \left( \prod_{j \in U} \sigma_\rho \beta_j(\rho) \right) \quad \text{(letting } U = \{j \in W : \mathcal{I}_j = 1\}) \\ 
    &= \sum_{U \subseteq W} \rho^{|W| - |U|} \sigma_\rho^{|U|} \prod_{j \in U} \beta_j(\rho)\\ 
    &= \rho^{|W|} \sum_{U \subseteq W} \left(\frac{\sigma_\rho}{\rho}\right)^{|U|} \prod_{j \in U} \beta_j(\rho)
    ~=~ \rho^{|W|} \sum_{U \subseteq W} \left(\frac{\sigma_\rho}{\rho}\right)^{|U|} \psi_U(T)
\end{align*}
Using the definition of $J$ in \Cref{eq:def-J}, and interchanging the two finite sums, we obtain
\begin{align*}
    J(T)
    =
    \sum_{W\subseteq[m]}
    g_W\rho^{-|W|}\chi_W(T)
    &= 
    \sum_{W\subseteq[m]}
    g_W\rho^{-|W|} 
    \rho^{|W|} \sum_{U \subseteq W} \left(\frac{\sigma_\rho}{\rho}\right)^{|U|} \psi_U(T)
    \\
    &=
    \sum_{U \subseteq [m]}
    \left(
        \frac{\sigma_\rho}{\rho}
    \right)^{|U|}
    \left(
        \sum_{W\supseteq U}g_W
    \right)
    \psi_U(T)\\ 
    &=
    \sum_{U \subseteq [m]}
    \left(
        \frac{\sigma_\rho}{\rho}
    \right)^{|U|}
    h_U \, \psi_U(T).
\end{align*}
For notational convenience, let us  define
\begin{equation}
    \label{eq:def-aU}
    a_U
    :=
    \left(
        \frac{\sigma_\rho}{\rho}
    \right)^{|U|}
    h_U,
    \qquad
    a_\varnothing = 1.
\end{equation}

Fix any deterministic set $\Bar{S}\subseteq[m]$. If $U\subseteq \Bar{S}$, then $\psi_U(T)$ is measurable with respect to $T_{\Bar{S}}$. If $U\nsubseteq \Bar{S}$, choose $j\in U\setminus \Bar{S}$. The factor $\beta_j(\rho) = (T_j-\rho)/\sigma_\rho$ is independent of $T_{\Bar{S}}$ and has mean zero. Hence,
\[
    \mathbb{E}_0[\psi_U(T)\mid T_{\Bar{S}}]
    \,=\,
    \begin{cases}
        \psi_U(T), & \text{if}\,~ U \subseteq \Bar{S},\\ 
        0, & \text{if}\,~ U\nsubseteq \Bar{S}.
    \end{cases}
\]
Therefore, since $a_{\varnothing} = 1$,
\[
    J_{\Bar{S}}(T_{\Bar{S}}) - 1
    =
    \sum_{\varnothing\neq U\subseteq \Bar{S}} a_U \, \psi_U(T).
\]
By orthonormality of the collection $\{\psi_U:U \subseteq [m]\}$, Parseval's theorem yields
\[
    \mathbb{E}_0
    \left[
        \left\{
            J_{\Bar{S}}(T_{\Bar{S}})-1
        \right\}^2
    \right]
    =
    \sum_{\varnothing\neq U\subseteq \Bar{S}} a_U^2
    =
    \sum_{\varnothing\neq U\subseteq [m]} a_U^2 \cdot \mathbf{1} \left\{U\subseteq \Bar{S}\right\}.
\]
Averaging over the independent random set $S$ yields
\begin{align*}
    \mathcal{W}_{m,q}
    =
    \sum_{\varnothing\neq U \subseteq [m]}
    a_U^2 \cdot \mathbb{P}(U\subseteq S)
    &=
    \sum_{\varnothing\neq U \subseteq [m]}
    q^{|U|}a_U^2 \\
    &=
    \sum_{\varnothing\neq U \subseteq [m]}
    q^{|U|}
    \left(
        \frac{\sigma_\rho}{\rho}
    \right)^{2|U|}
    h_U^2
    ~=
    \sum_{\varnothing\neq U \subseteq [m]}
    \left\{
        \frac{q(1-\rho^2)}{\rho^2}
    \right\}^{|U|}
    h_U^2,
\end{align*}
which proves the first claim.

Next, substituting the expression for $h_U$ from \Cref{lem:exact-hU} in the preceding representation yields
\begin{align*}
    \mathcal{W}_{m,q}
    &=
    \sum_{\{U \subseteq [m] \,:\, |U| \geq r+1\}}
    \left\{
        \frac{q(1-\rho^2)}{\rho^2}
    \right\}^{|U|}
    h_U^2
    \\
    &=\sum_{k=r+1}^m
    \binom{m}{k} 
    \left\{
        \frac{q(1-\rho^2)}{\rho^2}
    \right\}^{\!k}
    \left[
        \frac{1}{2^{k-1}} \ (-1)^r \ \binom{k-1}{r}
    \right]^2
    \\
    &= 4 \sum_{k=r+1}^m
    \binom{m}{k} \binom{k-1}{r}^{\!2}
    \left\{
        \frac{q(1-\rho^2)}{4\rho^2}
    \right\}^{\!k},
\end{align*}
as in the second claim.
\end{proof}

\subsection{Supporting lower bound lemmas for moderate curvature \texorpdfstring{$\mu$}{mu}}
\label{proof:lem:moderate-mu-lower-bound}

This section presents the portion of the lower bound argument for the regime $R\mu > \sqrt{q}$, which was claimed in the proof of \Cref{claim:lower-bound-regimes-II-III}. Recall that we first handled the case $\mu \in (0,\mu_{\mathrm{crit}}]$ using a Boolean construction with an odd number $m\asymp\mu^{-1}$ of auxiliary covariates. The argument required $m\ge5$, which was guaranteed by our choice of $\mu_{\mathrm{crit}}$ in \Cref{eq:mu-critical}. When $\mu \in (\mu_{\mathrm{crit}},1/3]$, however, a choice of $m$ of order $\mu^{-1}$ need not satisfy $m\ge5$, so the preceding construction cannot be applied directly.

For this range, $\mu$ is bounded below by a universal constant, and hence $\mu^{-1}$ is of constant order. A fixed-dimensional construction is therefore sufficient. We specialize the two-point construction of \Cref{subsubsec:large-R-two-point} to $m=1$ and use it to establish the required lower bound for $\mu \in (\mu_{\mathrm{crit}},1/3]$.

\begin{tcolorbox}
\begin{lemma}
\label{lem:two-dimensional-patch}
Consider the two-point construction specialized to $m=1$, and let $\lambda,\alpha>0$. Suppose that
\begin{equation}
\label{eq:two-dimensional-patch-conditions}
    \alpha^2\le q,\qquad
    \alpha\le\frac{\tanh(\lambda)}{2},\qquad
    \frac{\sinh(2\alpha)}{\sinh(2\lambda)}\le1-2\mu,\qquad
    8(\lambda^2+\alpha^2)\le R^2.
\end{equation}
Then, for all $i\in\{0,1\}$, $P_i$ is a well-specified logistic model satisfying
\[
    \|X\|_2=1\quad\text{a.s.},\quad\;
    \E_i[XX^\top]\succeq\mu I_2,\quad\;
    \theta_i\in\Theta_R,\quad\;
    \text{and}\quad\;
    \|\theta_1-\theta_0\|_2^2 = 8\alpha^2.
\]
Moreover, for a universal constant $C>0$,
\[
    \chi^2(Q_1,Q_0)
    \le
    C\frac{q^2\alpha^2}{\sinh^2(\lambda)}.
\]
Consequently, there exist sufficiently small universal constants $\widetilde{c},c>0$ such that, if
$
    \alpha^2 = \widetilde{c}\, \frac{\sinh^2(\lambda)}{nq^2}
$
and the conditions in \Cref{eq:two-dimensional-patch-conditions} hold, then
\begin{equation}
\label{eq:two-dimensional-patch-lower}
    \mathfrak{R}_n(R,\mu,q)
    \ge
    c\frac{\sinh^2(\lambda)}{nq^2}.
\end{equation}
\end{lemma}
\end{tcolorbox}

\begin{proof}
When $m=1$, the majority function is $f(t)=t$, so $g_{\{1\}}=1$ and hence $J(T)=\frac{T}{\rho}$.

The null covariance follows from $X=(A,B)^\top/\sqrt2$, together with
$A^2=B^2=1$ and $\mathbb{E}_0[AB]=0$. Under the alternative,
$AB=TV$ and
\[
    L_\alpha
    =
    e^{\alpha V}
    \left\{
        \cosh(\alpha) - V\sinh(\alpha) \cdot \frac{T}{\rho}
    \right\}.
\]
Since $T\indep V$ under $P_0$, $\mathbb{E}_0[T]=\rho$, $T^2=1$, and
\[
    \mathbb{E}_0[e^{\alpha V}]=\cosh(\alpha),
    \qquad\text{and}\qquad
    \mathbb{E}_0[Ve^{\alpha V}]=\sinh(\alpha),
\]
we obtain
\begin{align*}
    \mathbb{E}_1[AB]
    =\mathbb{E}_0[L_\alpha TV]
    &=
    \cosh(\alpha) \cdot \mathbb{E}_0[T] \cdot \mathbb{E}_0[Ve^{\alpha V}]
    -
    \frac{\sinh(\alpha)}{\rho} \cdot \mathbb{E}_0[T^2] \cdot \mathbb{E}_0[e^{\alpha V}] \\
    &=
    \cosh(\alpha)\sinh(\alpha)
    \left(\rho-\frac{1}{\rho}\right) \\
    &=
    \frac{1}{2}\sinh(2\alpha) \cdot 
    \left(\tanh(\lambda)-\coth(\lambda)\right) \\ 
    &=
    -\frac{1}{2}\sinh(2\alpha) \cdot \frac{2}{\sinh(2\lambda)} 
    =
    -\frac{\sinh(2\alpha)}{\sinh(2\lambda)}.
\end{align*}
Since $A^2=B^2=1$, the diagonal entries remain equal to $1/2$, and hence
\[
    \mathbb{E}_1[XX^\top]
    =
    \frac{1}{2}
    \begin{bmatrix}
        1 & -\mfrac{\sinh(2\alpha)}{\sinh(2\lambda)} \\ 
        -\mfrac{\sinh(2\alpha)}{\sinh(2\lambda)} & 1
    \end{bmatrix}
    \succeq
    \frac{1}{2}\left\{1-\frac{\sinh(2\alpha)}{\sinh(2\lambda)}\right\} I_2.
\]

The conditions $\alpha\le\rho/2$ and $8(\lambda^2+\alpha^2)\le R^2$ ensure, respectively, that $\alpha\|J\|_\infty\le1/2$ and that both parameters belong to $\Theta_R$. The assumption
\[
    \frac{\sinh(2\alpha)}{\sinh(2\lambda)}\le1-2\mu
\]
yields $\mathbb{E}_1[XX^\top]\succeq\mu I_2$, while $\mathbb{E}_0[XX^\top]=I_2/2\succeq\mu I_2$.

Also, since \(J(T)=T/\rho\), a direct calculation gives
\[
\mathcal{W}_{1,q}
= q\,\mathbb{E}_0\left[\left(\frac{T}{\rho}-1\right)^2\right]
= \frac{q(1-\rho^2)}{\rho^2}
= \frac{q}{\sinh^2(\lambda)}.
\]

Finally, note that whenever $0 < \alpha^2 \le q$, \Cref{eq:exact-observed-chi-square} and \Cref{lem:two-point-divergence} jointly imply that there exists a sufficiently large universal constant $C' > 0$ such that
\[
    \chi^2(Q_1,Q_0)
    \le C q\alpha^2\mathcal{W}_{1,q}
    = C'\frac{q^2\alpha^2}{\sinh^2(\lambda)}.
\]
Thus, with $\alpha^2=\frac{c'}{C'} \frac{\sinh^2(\lambda)}{nq^2}$, tensorization yields
\[
    \chi^2\left(Q_1^{\otimes n},Q_0^{\otimes n}\right)\le\exp(c')-1.
\]
Choosing $c'>0$ sufficiently small makes the corresponding total variation distance at most $1/2$. Since $\|\theta_1-\theta_0\|_2^2=8\alpha^2$, a standard application of Le Cam's two-point inequality (see, e.g., \citet{Wainwright2019highdim,SamworthShah2026modern}) gives the claimed lower bound.
\end{proof}

\begin{tcolorbox}
\begin{lemma}
\label{lem:moderate-mu-lower-bound}
Suppose that
\[
    \mu_{\mathrm{crit}}<\mu\le\mfrac{1}{3},
    \qquad
    R\mu>\sqrt{q}.
\]
If the sample-size condition $n \ge C q^{-3}$ holds for a sufficiently large universal constant $C>0$, then there exist universal constants $c_0,c_1>0$ for which
\begin{equation}
\label{eq:moderate-mu-unified-lower}
    \mathfrak{R}_n(R,\mu,q)
    \ge
    \frac{c_0}{nq\mu^2}
    \exp\left\{
        \frac{c_1}{\mu}
        \log\left(
            1+\frac{1-q}{q}\min\{R^2\mu^2,1\}
        \right)
    \right\}.
\end{equation}
\end{lemma}
\end{tcolorbox}

\begin{proof}
Write $h:=\min\{R\mu,1\}$. Set
\[
    \lambda:=\frac{h}{8},
    \quad~\text{and}\quad~
    \alpha^2:=c'' \frac{\sinh^2(\lambda)}{nq^2}.
\]
Since $R\mu>\sqrt{q}$, we have $\sqrt{q}\le h\le1$ and $0<\lambda\le1/8$. Also, the sample-size condition $n \ge Cq^{-3}$ implies that $1/(nq^2)\le q/C\le1/C$. In particular,
\[
    \alpha^2\le\frac{c'' \sinh^2(1/8)}{C}\,q\le q,
    \quad\;\text{and}\quad\;
    \frac{\alpha}{\tanh(\lambda)}
    =\cosh(\lambda)\sqrt{\frac{c''}{nq^2}}
    \le\cosh(1/8)\sqrt{\frac{c''}{C}}\le\frac{1}{2},
\]
after increasing $C$ if necessary.

We next verify the curvature condition. Since $\alpha^2\le q\le1$, we have $\alpha\le1$. Therefore, using $\sinh(2\alpha)\le2\alpha\cosh(2\alpha)\le2\alpha\cosh(2)$ and $\sinh(2\lambda)\ge2\lambda$,
\[
    \frac{\sinh(2\alpha)}{\sinh(2\lambda)}
    \le
    \cosh(2) \cdot \frac{\alpha}{\lambda}
    \le
    C\sqrt{\frac{c''}{nq^2}}
    \le
    C\sqrt{\frac{c''}{C}}
    \le\frac{1}{3} \le 1-2\mu,
\]
for $C$ sufficiently large. Finally, $h\le R\mu$ and $\mu\le1/3$ give $8\lambda^2=h^2/8\le R^2/72$. Moreover,
\[
    \alpha^2
    \le
    \frac{c''\cosh^2(1/8)}{64C}\,h^2,
\]
so, after increasing $C$ once more, $8(\lambda^2+\alpha^2)\le R^2$. Thus all the conditions of \Cref{lem:two-dimensional-patch} are satisfied.
Since $\sinh(\lambda)\ge\lambda=h/8$, \Cref{eq:two-dimensional-patch-lower} yields
\begin{equation}
\label{eq:moderate-mu-two-dimensional-bound}
    \mathfrak{R}_n(R,\mu,q)
    \ge
    c\frac{h^2}{nq^2}.
\end{equation}

We also use the baseline lower bound from \Cref{prop:baseline-assouad}. Since $R\mu>\sqrt{q}$, we have $R^2>q/\mu^2$, and therefore
\[
    \frac{1}{q\mu^2R^2}<\frac{1}{q^2}.
\]
Thus, under the minimum sample size assumption $n\ge C q^{-3}$, we have $1/(nq\mu^2)\le R^2$. Since $\exp(R/4)\ge1$, \Cref{prop:baseline-assouad} consequently yields
\begin{equation}
\label{eq:moderate-mu-baseline-bound}
    \mathfrak{R}_n(R,\mu,q)\ge\frac{c}{nq\mu^2}.
\end{equation}

Finally, choose a sufficiently small universal constant $c_1\in(0,\mu_{\mathrm{crit}})$, and define $\eta:=c_1/\mu\in(0,1)$ and $y:=(1-q)h^2/q$. Since $s\mapsto s^\eta$ is concave on $[0,\infty)$,
\[
    \frac{1}{nq\mu^2}(1+y)^\eta
    \le
    \frac{1}{nq\mu^2}(1+y)
    \le
    \frac{1}{nq\mu^2}
    +
    \frac{h^2}{nq^2\mu_{\mathrm{crit}}^2}.
\]
Combining \Cref{eq:moderate-mu-two-dimensional-bound,eq:moderate-mu-baseline-bound}, and adjusting the universal constant gives
\begin{equation*}
    \mathfrak{R}_n(R,\mu,q)
    \ge
    \frac{c_0}{nq\mu^2}
    \exp\biggl\{
        \frac{c_1}{\mu}
        \log\biggl(
            1+\frac{1-q}{q}\min\{R^2\mu^2,1\}
        \biggr)
    \biggr\},
\end{equation*}
where we used $h^2=\min\{R^2\mu^2,1\}$. This completes the proof.
\end{proof}

\end{document}